\documentclass[twoside,11pt]{article}

\usepackage{jmlr2e}

\usepackage[utf8]{inputenc} 
\usepackage[T1]{fontenc}    
\usepackage{hyperref}       
\usepackage{url}            
\usepackage{booktabs}       
\usepackage{nicefrac}       
\usepackage{microtype}      
\usepackage{amsmath, amsfonts, bm, bbm}
\usepackage[dvipsnames]{xcolor} 
\usepackage{graphicx,float}
\graphicspath{./}
\usepackage{float}
\usepackage{caption}
\usepackage{subcaption}
\usepackage{multirow}
\usepackage{amssymb}
\usepackage{enumitem}
\usepackage[linesnumbered,ruled,vlined,longend]{algorithm2e}
\usepackage{afterpage}
\usepackage{adjustbox} 
\usepackage{array}
\usepackage{tabularx}
\usepackage{siunitx} 
\usepackage[colorinlistoftodos,prependcaption,textsize=scriptsize]{todonotes}

\usepackage{comment}

\newtheorem{thm}{Theorem}[section]

\newtheorem{lem}[theorem]{Lemma}
\newtheorem{assumption}{Assumption}[section]

\newtheorem{rmk}{Remark}[section]
\newcommand{\vect}[1]{\boldsymbol{\mathbf{#1}}}

\newcommand{\bbR}{\mathbb{R}}
\newcommand{\bbP}{\mathbb{P}}
\newcommand{\tr}{\mathrm{tr}}
\newcommand{\bbE}{\mathbb{E}}

\newcommand{\ind}[1]{\mathbbm{1}_{\{#1\}}}

\jmlrheading{}{2020}{}{}{}{}{}

\ShortHeadings{SGD for Gaussian Process}{}
\firstpageno{1}
\begin{document}

\title{Gaussian Process Inference Using Mini-batch Stochastic Gradient Descent: Convergence Guarantees and Empirical Benefits}

\author{\name Hao Chen\thanks{Equal contribution.} \email haochen@stat.wisc.edu\\
        \name Lili Zheng\footnotemark[1] \email lilizheng@stat.wisc.edu\\
        \addr Department of Statistics\\
       University of Wisconsin-Madison\\
       1300 University Avenue\\
       Madison, WI 53706, USA\\
        \name Raed Al Kontar \thanks{Corresponding author.} \email alkontar@umich.edu\\
        \addr Department of Industrial and Operations Engineering\\
        University of Michigan\\
        1891 IOE Building 1205, Beal Ave\\
        Ann Arbor, MI 48109, USA\\
        \name Garvesh Raskutti \email raskutti@stat.wisc.edu\\
        \addr Department of Statistics\\
       University of Wisconsin-Madison\\
       1300 University Avenue\\
       Madison, WI 53706, USA}


\maketitle

\begin{abstract}
Stochastic gradient descent (SGD) and its variants have established themselves as the go-to algorithms for large-scale machine learning problems with independent samples due to their generalization performance and intrinsic computational advantage. However, the fact that the stochastic gradient is a biased estimator of the full gradient with correlated samples has led to the lack of theoretical understanding of how SGD behaves under correlated settings and hindered its use in such cases. In this paper, we focus on hyperparameter estimation for the Gaussian process (GP) and take a step forward towards breaking the barrier by proving minibatch SGD converges to a critical point of the full log-likelihood loss function, and recovers model hyperparameters with rate $O(\frac{1}{K})$ for $K$ iterations, up to a statistical error term depending on the minibatch size. Our theoretical guarantees hold provided that the kernel functions exhibit exponential or polynomial eigendecay which is satisfied by a wide range of kernels commonly used in GPs. Numerical studies on both simulated and real datasets demonstrate that minibatch SGD has better generalization over state-of-the-art GP methods while reducing the computational burden and opening a new, previously unexplored, data size regime for GPs. 
\end{abstract}


\section{Introduction}

The Gaussian process (GP) has seen many success stories in various domains, be it in optimization \citep{yue2019non,snoek2012practical}, reinforcement learning \citep{srinivas2009gaussian,krause2011contextual}, time series analysis \citep{kontar2020minimizing,alvarez2011computationally}, control theory \citep{kocijan2004gaussian,mesbah2016stochastic} and simulation meta-modeling \citep{zhou2011simple,qian2008bayesian}. One can attribute such success to its natural Bayesian interpretation, uncertainty quantification capability and highly flexible model priors. Yet its main limitation is the $O(n^3)$ computation and $O(n^2)$ storage for $n$ training points \citep{rasmussen2003gaussian}. Indeed, as mentioned in \cite{hensman2013gaussian}, a traditional large dataset for a GP is one with a few thousand data points and even those often require approximation techniques.

As a result, in the past two decades, a large body of work on GPs tackled approximate inference procedures to reduce the computational demands and numerical instabilities (mainly due to the need for matrix inversions). This push towards scalability dates back to the seminal paper by \citet{quinonero2005unifying} in 2005 which unified previous approximation methods into a single probabilistic framework based on inducing points. Since then, many new methods have also been introduced. Most notable are: variational inference procedures that laid the theoretical foundation for the class of inducing point methods \citep{damianou2016variational,nguyen2014collaborative,zhao2016variational,alvarez2010efficient,wilson2016stochastic}, mixture of experts models \citep{deisenroth2015distributed,tresp2000bayesian}, covariance tapering \citep{furrer2006covariance,kaufman2008covariance} and kernel expansions \citep{le2013fastfood,rahimi2008random, yang2015carte}. On the other hand, there has been a recent push to utilize increasing computational power and GPU acceleration to solve exact GPs. This recent literature includes distributed Cholesky factorizations \citep{nguyen2019exact}, preconditioned conjugate gradients (PCG) to solve linear systems \citep{gardner2018gpytorch} and kernel matrix partitioning to perform all matrix-vector multiplications \citep{wang2019exact}. Interestingly, \cite{wang2019exact} was able to fit a bit more than 1 million data points using 8 GPUs in a few days.

One possible solution to extend GPs far beyond what is currently possible is through stochastic gradient decent (SGD) and its variants: \emph{drawing $m<<n$ samples at each iteration and updating model parameters following the gradient of the log-likelihood loss function on the $m$ subsamples}. Indeed, SGD, or more generally the capability of inference via minibatches (possibly also with second order information), has been a key propeller behind the success of deep learning~\citep{lecun2015deep} in its various forms and other objectives. \textit{The caveat in GPs, however, is that, unlike empirical loss minimization, there exists correlation across all samples where any finite collection of the samples have a joint Gaussian distribution with covariance characterized by an empirical kernel matrix}. Hence the log-likelihood loss function is no longer the sum of losses evaluated at each sample, translating to the stochastic gradient being a biased estimator of the full gradient when taking expectation with respect to the random sampling. The lack of theoretical backing and understanding of how SGD behaves has long stood in the way of using SGD to conduct inference in GPs \citep{hensman2013gaussian} and in most settings where correlation amongst samples is high.

In this paper, we establish convergence guarantees of SGD for GPs for both the full gradient and the model parameters. Interestingly, without convexity or even Liptchitz conditions on the loss function, the structure of GP leads to an optimization error term of $O(\frac{1}{K})$ for converging to a critical point and recovering the true noise variance up to a statistical error that vanishes as $m$ tends to $+\infty$, for both RBF kernels and Mat\'{e}rn kernels. 
Our proof involves two key steps: first we concentrate the stochastic gradient to its conditional expectation using a covering argument and then we show that the latter satisfies a property similar to strong convexity by exploiting eigenvalues of the empirical kernel matrix. The proof and key findings offer standalone value beyond GPs and we hope they encourage researchers to further investigate SGD in other correlated settings such as L\'evy, It\^{o} and Markov processes.

Most importantly, our results open up a new data size regime to explore GPs. We are able to train $n \approx 1.2\times 10^6$ data points using a single CPU core in around 30 minutes. Recall, it took the most recent advancements in exact GPs a couple of days using 8 GPUs to train when $n \approx 10^6$, and $n$ is limited to approximately $10^4$ without GPU. We find that GPs inferred using SGD offer remarkably better performance in various case studies with different dataset sizes, noise levels and input dimensions. 



\subsection{Main Contributions} \label{keyfindings} 
We establish convergence guarantees for the minibatch SGD algorithm when training a GP under sampling with or without replacement and conduct numerical experiments to validate and supplement our theoretical results. Our main contributions are summarized as follows:
\begin{itemize}[leftmargin=*]
    \item \textbf{Convergence guarantees:} For a large enough minibatch size $m$, minibatch SGD converges to a {\em critical point} of the full log-likelihood loss function, and {\em recovers the true noise variance} up to a statistical error depending on $m$, when the kernel function exhibits exponential (RBF kernels) or polynomial eigendecay (Mat\'{e}rn kernels). 
    To be specific, the full gradient and the estimation error of the noise variance evaluated at the $K$th iterate are bounded by an {\em optimization error} term $O(\frac{1}{K})$ and a {\em statistical error} term $O(m^{-\frac{1}{2}})$, see Theorems \ref{thm:large_m_param_converg} to \ref{thm:large_m_grad_converg_poly}.  
    \begin{itemize}
        \item \textbf{Proof techniques for statistical error:} Since the stochastic gradient is biased for estimating full gradients, we instead bound the difference from its conditional expectation given the covariates in the corresponding minibatch, {\em uniformly} over all possible parameter iterates. We use novel truncation and covering arguments to prove the uniform error bound, in order to avoid the dependence between past parameter iterates and the minibatch in the current iteration. This contributes to the statistical error term $O(m^{-\frac{1}{2}})$ in the convergence error bound.
        \item \textbf{Proof techniques for optimization error:} To guarantee the $O(\frac{1}{K})$ optimization error bound, no convexity or even Liptchitz condition on the loss function are assumed. Instead, we prove that the conditional expectation of the loss function given covariates $\vect{X}_n$ satisfies a relaxed property of strong convexity (see e.g., Lemma \ref{lem:g_k_conv_property_abbrv}), where the ``curvature'' parameter is lower bounded by a constant regardless of minibatch size $m$. This proof relies on careful analysis for bounding the eigenvalues of empirical kernel matrices.
    \end{itemize}
   \item \textbf{Numerical findings:} Through benchmarking with state-of-the-art methods on various datasets we show that SGD offers great value from both computational and statistical perspectives. Computationally, we scale to dataset sizes previously unexplored in GPs in a fraction of time needed for competing methods. Meanwhile statistically, we find that SGD improves generalization in GPs, specifically in large data settings.  
\end{itemize}
\subsection{Related Work}

As mentioned earlier, there are several methods trying to tackle the {\em computational complexity of GPs}. Those can be roughly split into the following three categories, though it is by no means an exhaustive list (see the survey in [1]). 
\begin{itemize}
    \item \textbf{Exact inference via matrix vector multiplications (MVM):} This recent class of literature has had the most success in scaling GPs. Initially such approaches depended on a structured kernel matrix where data lies in a regularly spaced grid \citep{saatcci2012scalable,wilson2015kernel}. Then with the help of GPU acceleration, conjugate gradient and distributed Cholesky factorization, MVMs were applied to more general settings \citep{wang2019exact,gardner2018gpytorch,ubaru2017fast}. Such approaches have training complexity of $O(n^2)$ ($O(n\, \mbox{log} \, n)$ possible on spaced grids), yet amenable to distributed computation and GPU acceleration.
    \item \textbf{Sparse approximate inference:} This class of methods is based on a low rank approximation of the empirical kernel matrix where $\vect{K}_n \approx \vect{K}_{nz} \vect{K}^{-1}_{zz} \vect{K}_{zn}$ and $z$ denotes a set of inducing points with $\mbox{cardinality}(z)=n_z << n$ \citep{kontar2018nonparametric, alvarez2009sparse,damianou2016variational, zhao2016variational,snelson2006sparse}. Their time complexity is mainly $O(n_z^2 n)$ which can be reduced to $O(n + cn_z)$ for structured and regularly spaced grids. Indeed, sparse GPs have gained increased attention since variational inference (VI) laid the theoretical foundation of this class of inducing points/kernel approximations (starting from the early work of \citet{titsias2009variational}). 
    \item \textbf{Stochastic variational inference (SVI):} Following the work of \cite{hoffman2013stochastic}, SVI was introduced to GPs in \cite{hensman2013gaussian}. The key idea is to introduce a variational distribution over the inducing points so that the VI framework is amenable to stochastic optimization. This leads to a complexity of $O(n_z^3)$ at each iteration \citep{hoang2015unifying,blei2017variational}. Unfortunately, recent results in \cite{burt2019rates} show the need for at least $O(\mbox{log}^D n)$ inducing points for Gaussian kernels, which implies a superlinear growth with the input dimension $D$. Although many of the aforementioned methods are proposed in the context of model inference (prediction), the idea can often carry over to the model selection task, and the complexity is multiplied by the number of iterations.
\end{itemize}


The {\em theoretical analysis for SGD} has also been extensively studied under various assumptions~\citep{nemirovski2009robust,rakhlin2011making,frostig2015competing,bottou2018optimization}. In particular, in the context of empirical risk minimization where the objective loss function is summed over $n$ data points and stochastic gradients are calculated for i.i.d. sampled data at all iterations, it is known that the expected squared error of SGD iterates at iteration $K$ compared to the true minimizer is $O(1/K)$, with diminishing step size for strongly convex objectives. 
Furthermore, some recent literature~\citep{hardt2016train,keskar2016large} suggests that SGD has good generalization power, which has encouraged more practitioners to apply SGD to various application scenarios.

However, there are much fewer results when the stochastic gradients are {\em biased estimators} for the full gradients, despite the fact that unbiased estimates for the full gradient can be expensive or unavailable in certain cases. Some examples include: learning graph neural networks \citep{chen2018fastgcn}, distributed parallel optimization where sparsified stochastic gradient is applied \citep{stich2019error} and performing model selection for GPs. \cite{homem2008rates,chen2018stochastic,ajalloeian2020analysis} study the stochastic gradient algorithms under non-i.i.d. sampling or when the stochastic gradients are biased, and provide error bounds involving the bias term, or convergence guarantees built on consistency assumptions. Our paper does not make assumptions on the consistency of stochastic gradients or convexity of the full loss function, but exploits the nature of GP loss function and kernel matrices instead.

\subsection{Organization}
The paper is organized as follows: The problem setup is described in Section \ref{sec:problem_setup}, and theoretical guarantees are provided in Section \ref{sec:thm}; Section \ref{sec:proofoverview} is devoted to a proof outline, key lemmas and the proof of some main steps; Section \ref{sec:practical} presents practical considerations for applying minibatch SGD on GPs, while Section \ref{sec:numeric} includes our numerical results. We point out some open problems in Section \ref{sec:open_problem} and conclude in Section \ref{sec:conclusion}.
\section{Problem Setup}\label{sec:problem_setup}

\paragraph{Notation}
Vectors and matrices are denoted by boldface letters, e.g., $\vect{K}_n$, $\boldsymbol{\mathbf{\theta}}$, except for the full gradient $\nabla \ell(\vect{\theta})$ and stochastic gradient $g(\vect{\theta})$. For any vector $\vect{u}\in \bbR^p$, $u_i$ denotes its $i$th entry, and $\|\vect{u}\|_2=\left(\sum_{i=1}^p u_i^2\right)^{\frac{1}{2}}$ denotes its $\ell_2$ norm. For any square matrix $\vect{A}$, $\lambda_i(\vect{A})$ denotes its $i$th largest eigenvalue. Also see a table of important notations in Appendix~\ref{append:notations}.
%

We consider the Gaussian process model 
\begin{equation}\label{model}
\begin{split}
    f\sim \mathcal{GP}(\mu(\cdot), k(\cdot,\cdot)),&\quad \mathbf{x}_1,\dots,\mathbf{x}_n\overset{\text{i.i.d.}}{\sim}\bbP,\\
    y_i=f(\mathbf{x}_i) + \epsilon_i,&\quad \epsilon_i\overset{\text{i.i.d.}}{\sim} \mathcal{N}(0, \sigma_{\epsilon}^2),\quad 1\leq i\leq n, 
\end{split}
\end{equation}
where $\mathbf{x}_i\in \mathcal{X}\subset \mathbb{R}^D$ is the input, $\mu(\cdot): \mathcal{X}\rightarrow \mathbb{R}$ is the prior mean function, $k(\cdot,\cdot): \mathcal{X}\times \mathcal{X}\rightarrow \mathbb{R}$ is the prior covariance function, and $\epsilon_i$ is the observational noise with variance $\sigma_{\epsilon}^2$. Without loss of generality, we consider constant 0 mean function. In addition, the prior covariance function $k(\cdot,\cdot)=\sigma_f^2k_0(\cdot,\cdot)$ involves a known kernel function $k_0(\cdot,\cdot)$ and a signal variance parameter $\sigma_f^2$. 
We observe data points $\{(\mathbf{x}_i, y_i)\}_{i=1}^n$ generated from (\ref{model}) and organize them into $(\mathbf{X}_n, \mathbf{y}_n)=((\mathbf{x}_1,\dotsc,\mathbf{x}_n)^\top,(y_1,\dotsc,y_n)^\top)$, from which we aim to learn the hyperparameters in order to predict outputs from new inputs based on the posterior process.

Denote by $\vect{\theta}^*=(\sigma_{f}^2,\sigma_{\epsilon}^2)^\top\in\bbR^{2}$ 
the underlying hyperparameters to be determined, and for notational convenience, we may also use $\theta_1^*$ to denote $\sigma_f^2$ and $\theta_2^*$ to denote $\sigma_{\epsilon}^2$ in the following. 
One direct approach to estimate $\vect{\theta}^*$ is by applying gradient descent to minimize the scaled negative log marginal likelihood function
\begin{equation}\label{eq:objective}
\begin{split}
    \ell(\vect{\theta};\mathbf{X}_n, \mathbf{y}_n) & = -\frac{1}{n}\log p(\mathbf{y}_n| \mathbf{X}_n,\vect{\theta}) \\
    & = \frac{1}{2n}[\mathbf{y}_n^\top\mathbf{K}_n^{-1}(\vect{\theta})\mathbf{y}_n+\log|\mathbf{K}_n(\vect{\theta})|+n\log (2\pi)]
\end{split}
\end{equation}
over $\vect{\theta}\in (0,\infty)^{2}$, 
where $\mathbf{K}_n(\vect{\theta})=\theta_1\mathbf{K}_{f,n}+\theta_{2}\mathbf{I}_n\in \bbR^{n\times n}$ 
is the marginal covariance matrix for noisy observations $\mathbf{y}_n$ given $\vect{X}_n$ and $\mathbf{K}_{f,n}\in\bbR^{n\times n}$ is the kernel matrix of $k_0$ evaluated at $\mathbf{X}_n$, i.e. $(\vect{K}_{f,n})_{i,j}=k_0(\vect{x}_i,\vect{x}_j)$. For notational convenience we will omit $\vect{K}_n(\vect{\theta})$ to $\vect{K}_n$ when $\vect{\theta}$ is clear from the context and denote $\vect{K}_n(\vect{\theta}^*)$ by $\vect{K}_n^*$. In this case, the derivative of $\ell(\vect{\theta})$ is of particular interest to us where each of its entries takes the form
\begin{equation}\label{eq:gradient}
\begin{split} 
  \left(\nabla\ell(\vect{\theta};\mathbf{X}_n, \mathbf{y}_n)\right)_l & = \frac{1}{2n}\left[-\mathbf{y}_n^\top\mathbf{K}_n^{-1}\frac{\partial \mathbf{K}_n}{\partial \theta_l}\mathbf{K}_n^{-1}\mathbf{y}_n+\text{tr}\left(\mathbf{K}_n^{-1}\frac{\partial \mathbf{K}_n}{\partial \theta_l}\right)\right] \\
   & = \frac{1}{2n}\text{tr}\left[(\mathbf{K}_n^{-1}(\mathbf{I}_n-\mathbf{y}_n\mathbf{y}_n^T\mathbf{K}_n^{-1})\frac{\partial \mathbf{K}_n}{\partial \theta_l}\right],
\end{split}   
\end{equation}
where $\theta_l$ is the $l$th element of $\vect{\theta}$ and $(\partial \mathbf{K}_n/\partial \theta_l)_{jk}=\partial (\mathbf{K}_n)_{jk}/\partial \theta_l$. For notational convenience we will suppress $\vect{X}_n, \vect{y}_n$ and use $\nabla \ell(\vect{\theta})$ instead. Notice that the computation in (\ref{eq:gradient}) is dominated by the calculation of $\mathbf{K}_n^{-1}$ which requires $O(n^3)$ time. In order to reduce the computational cost of training, we consider the minibatch stochastic gradient descent approach to optimize ($\ref{eq:objective}$).
\subsection{Minibatch SGD algorithm}
Let $\xi$ be a random subset of $\{i\}_{i=1}^n$ of size $m$, then $\{(\mathbf{x}_i,y_i)\}_{i\in \xi}$ is the corresponding subset of data points which we organize into $(\mathbf{X}_{\xi}, \mathbf{y}_{\xi})$, where  $\mathbf{X}_{\xi}$ is the submatrix formed by the rows of $\mathbf{X}_n$ and $\mathbf{y}_{\xi}$ is the subvector of $\mathbf{y}_n$, both indexed by $\xi$. Define $g(\vect{\theta};\mathbf{X}_{\xi}, \mathbf{y}_{\xi})\in \bbR^{2}$ as an approximation to $\nabla\ell(\theta;\mathbf{X}_n, \mathbf{y}_n)$ that can be calculated from this subset, i.e.
\begin{equation}\label{eq:stochastic-gradient}
    \left(g(\vect{\theta};\mathbf{X}_{\xi}, \mathbf{y}_{\xi})\right)_l =  \frac{1}{2s_l(m)}\text{tr}\left[(\mathbf{K}_{\xi}^{-1}(\mathbf{I}_m-\mathbf{y}_{\xi}\mathbf{y}_{\xi}^\top\mathbf{K}_{\xi}^{-1})\frac{\partial \mathbf{K}_{\xi}}{\partial \theta_l}\right],\quad 1\leq l\leq 2,
\end{equation}
where $\mathbf{K}_{\xi}$ is the covariance matrix for $\mathbf{y}_{\xi}$ while also the principle submatrix formed by the rows and columns of $\mathbf{K}_n$ indexed by $\xi$. In the following we will also let $\vect{K}_{f,\xi}$ denote the $m\times m$ block of $\vect{K}_{f,n}$ indexed by $\xi$. A natural choice for $s_l(m)$ is $m$, but we will see in Section \ref{sec:thm} that if kernels $k_0$ have exponential eigendecay, setting $s_1(m)\asymp \log m$ and $s_{2}(m)=m$ would lead to the convergence of both $\theta^{(k)}_1, \theta^{(k)}_{2}$ to the true hyperparameters. 
Algorithm \ref{alg:algorithm1} summarizes the steps of minibatch SGD, where we do not specify whether minibatches are sampled with or without replacement since our convergence guarantees will hold true under both cases, if minibatch size $m$ is large enough (details provided in Section \ref{sec:thm}). We consider \emph{diminishing step sizes}: the step size at the $k$ th iteration is $\alpha_k=\frac{\alpha_1}{k}$. It is noteworthy that the time complexity of Algorithm \ref{alg:algorithm1} is $O(Km^3)$ , compared to $O(Kn^3)$ for running gradient descent with $K$ iterations.
\subsection{Sampling Methods}

In Algorithm \ref{alg:algorithm1} 
we conduct uniform sampling for each minibatch, that is, any subset of indices $\xi\subset[n]$ of size $m$ has the same probability of being selected. An alternative to uniform sampling is to sample data points that are close to each other, which we call {\em nearby sampling}. One particular nearby sampling strategy is nearest neighbor search, where a minibatch consists of a uniformly sampled data point and its $m-1$ nearest neighbors within the data pool. We may construct a $k$-$d$ tree to conduct nearest neighbor search, which finds the $m-1$ nearest neighbors for every data point in a given dataset of size $n$ in $O(n\log n)$ time and $O(n)$ space. That is to say, the time complexity for minibatch SGD with this nearby sampling method (only line 3 in Algorithm \ref{alg:algorithm1} changes) is $O(Km^3+n\log n)$ for $K$ iterations.

Our main theoretical contribution is establishing convergence guarantees for uniform sampling SGD in Algorithm \ref{alg:algorithm1}, but in addition to that, we will also provide some theoretical insights and numerical experiments for understanding the effect of nearby sampling. 
\begin{algorithm}[t]
\SetAlgoLined
 Input: $\vect{\theta}^{(0)}\in \bbR^{2}$, initial step size $\alpha_1>0$.\\
 \For{$k=1,2,\dotsc,K$}{
  Randomly sample a subset of indices $\xi_k$ of size $m$\; 
  Compute the stochastic gradient $g(\vect{\theta}^{(k-1)};\mathbf{X}_{\xi_k}, \mathbf{y}_{\xi_k})$\;
  $\alpha_k\leftarrow \frac{\alpha_1}{k}$\;
  $\vect{\theta}^{(k)}\leftarrow \vect{\theta}^{(k-1)} - \alpha_kg(\vect{\theta}^{(k-1)};\mathbf{X}_{\xi_k}, \mathbf{y}_{\xi_k})$\;
 }
 \caption{Minibatch SGD with uniform sampling}
 \label{alg:algorithm1}
\end{algorithm}\vspace{-2mm}
\section{Theoretical Guarantees}\label{sec:thm}
In this section, we present convergence guarantees for Algorithm \ref{alg:algorithm1}, including error bounds for $\theta_l^{(k)}-\theta_l^*, l=1,2$ and $\nabla \ell(\vect{\theta}^{(k)})$. Two types of kernels are considered: those with exponential eigendecay (Section \ref{sec:theory_exponential}) and those with polynomial eigendecay (Section \ref{sec:theory_polynomial}). For both types of kernels, the convergence of $\theta^{(k)}_{2}$ to the true noise variance $\sigma_{\epsilon}^2$ and the full gradient $\nabla \ell(\vect{\theta}^{(k)})$ to $0$ is guaranteed. In particular, for the former type of kernel, $\theta^{(k)}_1$ is also guaranteed to converge to the truth $\sigma_{f}^2$ under appropriate choice of $s_1(m)$. 

Furthermore, to understand the faster convergence observed in numerical studies when nearby sampling is applied rather than uniform sampling (Algorithm \ref{alg:algorithm1}), we provide theoretical insights into how nearby sampling influences the curvature parameter in Section \ref{sec:nearby_sampling}. 

First we state the assumptions needed for our convergence guarantees.
\begin{assumption}[Bounded iterates]\label{assump:param_bnd}
Both $\vect{\theta}^*$ and $\vect{\theta}^{(k)}$ for $0\leq k\leq K$ lie in $[\theta_{\min},\theta_{\max}]^{2}$, where $0<\theta_{\min}<\theta_{\max}$. 
\end{assumption}
\begin{remark}[Justification for Assumption~\ref{assump:param_bnd}]
The boundedness of parameter iterates is usually assumed in the literature of theoretical analysis for SGD~\citep[e.g.,][]{nemirovski2009robust}. As will be revealed by Theorem~\ref{thm:large_m_param_converg} and Theorem~\ref{thm:large_m_grad_converg_poly}, $\vect{\theta}^{(k)}$ is guaranteed to be bounded within a region around $\vect{\theta}^*$ that gets smaller as $k$ increases w.h.p., whenever the previous $\vect{\theta}^{(0)},\dots,\vect{\theta}^{(k-1)}$ are bounded within some $[\theta_{\min},\theta_{\max}]$. 
Hence, we only need to be careful about the initial steps of Algorithm~\ref{alg:algorithm1}, 
ensuring that the parameters $\theta^{(k)}_j, j=1,2$ are always positive and bounded. 
Moreover, our numerical experiments suggest that the iterate $\vect{\theta}^{(k)}$ of Algorithm~\ref{alg:algorithm1} is always bounded as long as the initial step size is chosen appropriately (see Figure~\ref{fig:param_unif} in Section~\ref{sec:numeric}). 

\end{remark}
\begin{assumption}[Bounded stochastic gradient]\label{assump:sg_bnd}
For all $0\leq k<K$, \begin{equation*}
    \|g(\vect{\theta}^{(k)};\vect{X}_{\xi_{k+1}},\vect{y}_{\xi_{k+1}})\|_2\leq G
\end{equation*} for some $G>0$.
\end{assumption}

\begin{remark}[Justification for Assumption~\ref{assump:sg_bnd}]
The boundedness assumption for stochastic gradients is also commonly seen in the literature~\citep[e.g.,][]{hazan2011beyond}. Furthermore, \textbf{Assumption~\ref{assump:param_bnd} can imply Assumption~\ref{assump:sg_bnd} with high probability}, under similar conditions to those in Theorem~\ref{thm:large_m_param_converg} or Theorem~\ref{thm:large_m_param_converg_poly}. The key idea is the stochastic gradients will be shown to be close to their conditional expectation $\mathbb{E}(g(\vect{\theta}^{(k)};\vect{X}_{\xi_{k+1}},\vect{y}_{\xi_{k+1}}|\vect{X}_{\xi_{k+1}})$, and the latter is bounded with high probability given Assumption~\ref{assump:param_bnd} and the eigendecay assumptions for kernels in Theorem~\ref{thm:large_m_param_converg} or Theorem~\ref{thm:large_m_param_converg_poly}. We defer the detailed explanation on this in Appendix~\ref{append:sg_bnd}. 

\end{remark}

\subsection{Kernels with Exponential Eigendecay}\label{sec:theory_exponential}
The exponential eigendecay assumption is stated in detail as follows:
\begin{assumption}[Exponential eigendecay]\label{assump:eigen_decay_exp}
Consider the kernel operator $\mathcal{K}: L^2(\mathbb{P})\rightarrow L^2(\mathbb{P})$ that satisfies $\mathcal{K}\phi(\cdot)= \int \phi(\vect{x})k_0(\cdot,\vect{x})d\mathbb{P}(\vect{x})$, where $\mathbb{P}$ is the probability measure of the input as defined in~\eqref{model}. The eigenvalues of $\mathcal{K}$ are $\{Ce^{-bj}\}_{j=0}^{\infty}$, where $b>0$ and $C\leq 1$ are regarded as constants.
\end{assumption}
The exponential eigendecay assumption is satisfied by the Radial Basis Function (RBF) kernels $k(\vect{x},\vect{x}')=\exp\{-\|\vect{x}-\vect{x}'\|_2^2/(2l^2)\}$ when the probability distribution $\mathbb{P}$ of the input is Gaussian (see Section 4.3.1 of \cite{rasmussen2003gaussian}), which is widely seen in the GP literature. The specific decay rate $b$ depends on the lengthscale parameter $l$ of the corresponding kernel $k_0$. 
The requirement $C\leq 1$ is only for theoretical convenience, and it suffices to have bounded $C$.
The following theorem guarantees the convergence of the parameter iterates under the aforementioned assumptions.
\begin{thm}[Convergence of parameter iterates, exponential eigendecay]\label{thm:large_m_param_converg}
Consider the output $\vect{\theta}^{(K)}$ of Algorithm~\ref{alg:algorithm1}, the minibatch SGD algorithm with diminishing step sizes. Under Assumptions \ref{assump:param_bnd} to \ref{assump:eigen_decay_exp}, when $m>C$ for some constant $C>0$, we have the following results under two corresponding conditions on $s_l(m)$:
 \begin{enumerate}[leftmargin=*]\vspace{-2mm}
 \item If $s_2(m)=m$, initial step size $\alpha_1$ satisfies $\frac{3}{2\gamma}\leq \alpha_1\leq \frac{2}{\gamma}$ where $\gamma=\frac{1}{4\theta_{\max}^2}$, then for any $0<\varepsilon<C\frac{\log\log m}{\log m}$, with probability at least $1-CK\exp\{-cm^{2\varepsilon}\}$,
\begin{equation}\label{eq:noise_var_convg_bnd}
    (\theta^{(K)}_2-\theta^*_2)^2\leq \frac{8G^2}{\gamma^2(K+1)}+Cm^{-\frac{1}{2}+\varepsilon}.
\end{equation}
\item If in addition to $s_2(m)=m$, $s_1(m)$ is set as $\tau\log m$ where $\tau>C$, $\frac{3}{2\gamma}\leq \alpha_1\leq \frac{2}{\gamma}$ where $\gamma$ depends on $\tau$, 
    then for any $0<\varepsilon<\frac{1}{2}$, with probability at least $1-CK\exp\{-c(\log m)^{2\varepsilon}\}$,
\begin{equation}
    (\theta^{(K)}_1-\theta^*_1)^2+(\theta^{(K)}_{2}-\theta^*_{2})^2\leq \frac{8G^2}{\gamma^2(K+1)}+C(\log m)^{-\frac{1}{2}+\varepsilon}.
\end{equation}
 \end{enumerate}
 Here $c,C>0$ depend only on $\theta_{\min},\theta_{\max},b$.
\end{thm}

\begin{rmk}
Theorem \ref{thm:large_m_param_converg} suggests that the noise variance parameter $\theta^{(K)}_{2}$ is guaranteed to converge to the truth $\theta^*_{2}$, with the optimization error term $O(\frac{1}{K})$ and the statistical error term $O(m^{-\frac{1}{2}+\varepsilon})$ with high probability if $\varepsilon\log m$ is large, when the initial stepsize is appropriately chosen and $s_2(m)=m$. Furthermore, if we let $s_1(m)=\tau\log m$, then Algorithm \ref{alg:algorithm1} achieves convergence for both $\theta^{(K)}_1$ and $\theta^{(K)}_{2}$ with statistical error $O((\log m)^{-\frac{1}{2}+\varepsilon})$.
\end{rmk}

\begin{rmk}
The different rates of statistical errors for estimating $\theta^*_1$ and $\theta^*_2$ come from the different eigenvalue structures between $\vect{K}_{f,\xi}$ (the $m\times m$ block of $\vect{K}_{f,n}$ indexed by $\xi$) and $\vect{I}_m$. One may also note that the statistical errors depend on $m$ instead of $n$: this is due to the correlation among $\vect{y}_{\xi}$ from different minibatches $\xi$, conditioning on $\vect{X}$, which is different from the problems with independent samples.
\end{rmk}

\begin{rmk}
The choice $s_1(m)\asymp \log m$ is because
\begin{equation}
    s_1(m)g_1(\vect{\theta})=\tr\left[\vect{K}_{\xi}^{-1}(\vect{I}_m-\vect{y}_{\xi}\vect{y}_{\xi}^\top \vect{K}_{\xi}^{-1})\frac{\partial \vect{K}_{\xi}}{\partial \theta_1}\right]\asymp \log m,
\end{equation} 
and thus this choice of $s_1(m)$ ensures that $g_1(\vect{\theta})$ has the same scale as $g_2(\vect{\theta})$ (constant scale). 
\end{rmk}

\begin{rmk}
For the second case where $s_1(m)=\tau\log m$, we need $\tau>\frac{64\theta_{\max}^4}{b\theta_{\min}^4}$ and 
 \begin{equation*}
     \gamma=\min\left\{\frac{1}{32\tau b\theta_{\max}^2},\frac{1}{4\theta_{\max}^2}-\frac{2\theta_{\max}^2}{\tau b\theta_{\min}^4}\right\}.
 \end{equation*}
\end{rmk}
\begin{rmk}
The optimization error $O(\frac{1}{K})$ is credited to the structure of the GP loss function, which satisfies a relaxation of strong convexity (see Lemma \ref{lem:g_k_conv_property_abbrv} in Section \ref{sec:proofoverview}). $\gamma$ can be viewed a lower bound of an approximate ``curvature'' of the loss function, in the sense of a relaxed convexity. We will revisit this curvature term in Section \ref{sec:nearby_sampling} and illustrate the potential improvement nearby sampling brings to the curvature. 
\end{rmk}

Based on Theorem \ref{thm:large_m_param_converg}, we also derive the following convergence guarantee for the full gradient.
\begin{thm}[Convergence of full gradient, exponential eigendecay]\label{thm:large_m_grad_converg}
Consider the output $\vect{\theta}^{(K)}$ of Algorithm~\ref{alg:algorithm1}, the minibatch SGD algorithm with diminishing step sizes. Under Assumptions \ref{assump:param_bnd} to \ref{assump:eigen_decay_exp}, if $\frac{3}{2\gamma}\leq \alpha_1\leq \frac{2}{\gamma}$ for $\gamma=\frac{1}{4\theta_{\max}^2}$,  $m>C$, $s_2(m)=m$, then for any $0<\varepsilon<C\frac{\log\log m}{\log m}$, with probability at least $1-CK\exp\{-cm^{2\varepsilon}\}$,
\begin{equation}
    \|\nabla \ell(\vect{\theta}^{(K)})\|_2^2\leq C\left[\frac{G^2}{K+1}+m^{-\frac{1}{2}+\varepsilon}\right],
\end{equation}
holds, where $c,C>0$ depend only on $\theta_{\min},\theta_{\max},b$.
\end{thm}
Theorem \ref{thm:large_m_grad_converg} implies that running SGD for sufficiently many iterations with large minibatch size leads to the convergence to a critical point of $\ell(\vect{\theta})$. As we will show in the proof sketch, the error bound for the full gradient is dominated by $(\theta^{(K)}_2-\theta^{*}_2)^2$, the estimation error of the noise variance, thus it scales the same as \eqref{eq:noise_var_convg_bnd}.

\subsection{Kernels with Polynomial Eigendecay}\label{sec:theory_polynomial}
Now we consider the kernels with polynomial eigendecay, which captures much stronger correlation than the kernels with exponential eigendecay, and thus broadens the applications of GP to a wider class of data sets. Due to this reason, it is of both practical and theoretical interest to investigate how SGD performs for this type of kernels. The polynomial eigendecay assumption is stated in detail as follows:
\begin{assumption}[Polynomial eigendecay]\label{assump:eigen_decay_poly}
Consider the kernel operator $\mathcal{K}: L^2(\mathbb{P})\rightarrow L^2(\mathbb{P})$ that satisfies $\mathcal{K}\phi(\cdot)= \int \phi(\vect{x})k_0(\cdot,\vect{x})d\mathbb{P}(\vect{x})$, where $\mathbb{P}$ is the probability measure of the input as defined in~\eqref{model}. The eigenvalues of $\mathcal{K}$ are $\{Cj^{-2b}\}_{j=0}^{\infty}$, where $b>\frac{\sqrt{21}+3}{4}$, and $C\leq 1$ are regarded as constants.
\end{assumption}
This assumption is satisfied by the Mat\'{e}rn kernels (see section 2.3 in \cite{bach2017equivalence}, \cite{kanagawa2018gaussian}), another important kernel function class widely used in GP:
\begin{equation}\label{eq:matern}
    k_{\alpha,h}(\vect{x},\vect{x}')=\frac{1}{2^{\alpha-1}\Gamma(\alpha)}\left(\frac{\sqrt{2\alpha\|\vect{x}-\vect{x}\|_2^2}}{h}\right)^{\alpha}B_{\alpha}\left(\frac{\sqrt{2\alpha\|\vect{x}-\vect{x}\|_2^2}}{h}\right),
\end{equation} 
where $B_{\alpha}(\cdot)$ is the modified Bessel function of the second kind of order $\alpha$, and larger $\alpha$ leads to faster decay rate $b>0$.
   \begin{thm}[Convergence of parameter iterates, polynomial eigendecay ]\label{thm:large_m_param_converg_poly}
 Consider the output $\vect{\theta}^{(K)}$ of Algorithm~\ref{alg:algorithm1}, the minibatch SGD algorithm with diminishing step sizes. Under Assumptions \ref{assump:param_bnd} to \ref{assump:sg_bnd} and Assumption \ref{assump:eigen_decay_poly}, when $m>C$ for some constant $C>0$, $s_2(m)=m$, $\frac{3}{2\gamma}\leq \alpha_1\leq \frac{2}{\gamma}$ where $\gamma=\frac{1}{8\theta_{\max}^2}$, then for any $\varepsilon\in(\max\{0,f_1(b)\},\frac{1}{2})$, with probability at least $1-CKm^{-f_2(b)\left[\varepsilon-f_1(b)\right]}-CK\exp\{-cm^{2\varepsilon}\}$,
      \begin{equation}\label{eq:noise_var_converg_bnd_poly}
          (\theta^{(K)}_2-\theta^*_2)^2\leq \frac{8G^2}{\gamma^2(K+1)}+Cm^{-\frac{1}{2}+\varepsilon}.
      \end{equation}
  Here $c,C>0$ depend only on $\theta_{\min},\theta_{\max},b$, and $f_1(b)=-\frac{2b^2-5b-3}{2b(2b-1)}$, $f_2(b)=\frac{4b(2b-1)}{4b+3}$.
\end{thm}
\begin{rmk}[Comparison with Theorem~\ref{thm:large_m_param_converg}: error bounds]
Compared with Theorem~\ref{thm:large_m_param_converg}, Theorem~\ref{thm:large_m_param_converg_poly} reflects the influence of stronger correlations (slower eigendecay of kernels) on the convergence of SGD. More specifically, $\varepsilon>f_1(b)$ is required in addition to $0<\varepsilon<\frac{1}{2}$. That is to say, when $0<f_1(b)<\frac{1}{2}$ ($\frac{\sqrt{21}+3}{4}<b<3$), the statistical error scales at least as $m^{-(\frac{1}{2}-f_1(b))}>m^{-\frac{1}{2}}$, which decreases as $b$ increases on $(\frac{\sqrt{21}+3}{4},3)$; while if $f_1(b)\leq 0$ ($b\geq 3$) which means the correlation is not too strong, the statistical error still scales roughly as $m^{-\frac{1}{2}}$, the same as the exponential eigendecay case. Therefore, the slower eigendecay of kernels (stronger correlation structure) may lead to a slower convergence of SGD; while the good news is that, for moderately fast polynomial eigendecay ($b\geq 3$) we still have the same rate as the exponential eigendecay case.
\end{rmk}
\begin{rmk}[Comparison with Theorem~\ref{thm:large_m_param_converg}: probability terms]
Another difference between Theorem~\ref{thm:large_m_param_converg} and Theorem~\ref{thm:large_m_param_converg_poly} lies in the probability term. When $b<3$ and thus $f_1(b)>0$, $CKm^{-f_2(b)[\varepsilon-f_1(b)]}$ always dominates $CK\exp\{-cm^{2\varepsilon}\}$ for $\varepsilon>f_1(b)$, which is a lower probability for the error bound in \eqref{eq:noise_var_converg_bnd_poly} compared to \eqref{eq:noise_var_convg_bnd}. This is also the price we need to pay when considering kernels with slower eigendecay.
\end{rmk}

To extend the theoretical results for kernels with exponential eigendecay to polynomial eigendecay, we develop novel upper and lower bounds for $\sum_{j=1}^n \lambda_{j}^l(\theta_1^{(k)}\lambda_{j}+\theta^{(k)}_2)^{-2}, l=0,1,2$ where $\lambda_j$ is the $j$th largest eigenvalue of $K_{f,\xi}$, see Lemma~\ref{lem:eigen_ratio_bnds_poly}. The proof for Lemma~\ref{lem:eigen_ratio_bnds_poly} requires careful analysis and different arguments from the proof for Lemma~\ref{lem:eigen_ratio_bnds} that is established for kernels with exponential eigendecay, although they are both based on error bounds for eigenvalues of empirical kernel matrices in \cite{braun2006accurate}.

We briefly explain the reason behind the different scalings of statistical errors between Theorem \ref{thm:large_m_param_converg_poly} and the first part of Theorem \ref{thm:large_m_param_converg} in the following. In fact, the statistical error term for $(\theta^{(K)}_2-\theta^*_2)^2$ is composed of two parts: one is $m^{-\frac{1}{2}+\varepsilon}$ for both types of kernels, another is caused by the fact that $\theta^{(K)}_1$ may not be estimated well, and the error induced by this fact depends on the eigendecay of the kernel, which scales as $m^{-(\frac{1}{2}-f_1(b))}$ for kernels with polynomial eigendecay. While for kernels with exponential eigendecay, this error term scales as $\frac{\log m}{m}$ and is dominated by the first error term $m^{-\frac{1}{2}+\varepsilon}$. 
\begin{rmk}
For kernels with polynomial eigendecay, we don't have convergence guarantee for $\theta^{(K)}_1$ (signal variance of the kernel with the slowest eigendecay). This is due to that it is very hard to derive matched upper and lower bounds for the stochastic gradient $s_1(m)g_1(\vect{\theta})=\tr\left[(\vect{K}_{\xi}^{-1}(\vect{I}_m-\vect{y}_{\xi}\vect{y}_{\xi}^\top \vect{K}_{\xi}^{-1})\frac{\partial \vect{K}_{\xi}}{\partial \theta_1}\right]$ (which scales as $\log m$ for exponential eigendecay), and thus we cannot specify the choice for $s_1(m)$ in order to make $g_1(\vect{\theta})$ scales similarly from $g_2(\vect{\theta})$. 
\end{rmk}

\begin{thm}[Convergence of full gradient, polynomial eigendecay]\label{thm:large_m_grad_converg_poly}
Consider the output $\vect{\theta}^{(K)}$ of Algorithm~\ref{alg:algorithm1}, the minibatch SGD algorithm with diminishing step sizes. Under the same conditions as Theorem \ref{thm:large_m_grad_converg_poly}, for any $\varepsilon\in(\max\{0,f_1(b)\},\frac{1}{2})$, with probability at least $1-CK\left(m^{-f_2(b)\left[\varepsilon-f_1(b)\right]}+\exp\{-cm^{2\varepsilon}\}\right)$,
\begin{equation}\label{eq:full_grad_converg_bnd_poly}
    \|\nabla \ell(\vect{\theta}^{(K)})\|_2^2\leq C\left[\frac{G^2}{K+1}+m^{-\frac{1}{2}+\varepsilon}\right],
\end{equation}
holds, where $c,C>0$ depend only on $\theta_{\min},\theta_{\max},b$, $f_1(b)$ and $f_2(b)$ are defined as in Theorem \ref{thm:large_m_param_converg_poly}.
\end{thm}
As mentioned after Theorem \ref{thm:large_m_grad_converg}, the bound \eqref{eq:full_grad_converg_bnd_poly} for the full gradient scales the same as the bound \eqref{eq:noise_var_converg_bnd_poly} for $(\theta^{(K)}_{2}-\theta^*_{2})^2$.

\subsection{Effect of Nearby Sampling}\label{sec:nearby_sampling}
Although our theoretical guarantees are all derived for uniform sampling, some empirical evidence suggests that sampling nearby points for each minibatch can lead to lower errors for learning the noise variance $\theta_2^*=\sigma_{\epsilon}^2$ and sometimes improved prediction performance (some comparisons are provided in Section \ref{sec:case_studies}). 
In this section, we present some theoretical insights to understand why and how nearby sampling helps, in terms of learning $\theta_2^*$. Note that we focus on $\theta_2^*$ instead of $\theta_1^*$, since $\theta^{(k)}_{2}-\theta_{2}^*$ dominates the convergence of the upper bound for $\|\nabla \ell(\vect{\theta}^{(k)}\|_2$ at the $k$th iteration.

In the following, we investigate the following hypothesis: {\em the approximate ``curvature'' for $\theta_2^{(k)}$ is improved when the points within each minibatch are closer to each other, so that nearby sampling leads to larger curvature and hence faster convergence.}
More specifically, some calculation shows that at the $(k+1)$th iteration, the ``approximate'' curvature term w.r.t. the noise variance is
\begin{equation}\label{eq:appr_curv}
    \gamma(\vect{\theta}^{(k)}):=\frac{\partial \mathbb{E}(g_{2}(\vect{\theta}^{(k)};\vect{X}_{\xi_{k+1}},\vect{y}_{\xi_{k+1}})|\vect{X}_{\xi_{k+1}})}{\partial \theta_{2}}=\frac{1}{2m}\sum_{j=1}^m(\theta^{(k)}_1\lambda_j^{(k)}+\theta^{(k)}_2)^{-2},
\end{equation}
where $\lambda_{j}^{(k)}$ is the $j$th largest eigenvalue of $\vect{K}_{f,\xi_{k+1}}$. 
As revealed in our proof of Theorem \ref{thm:large_m_param_converg}, \ref{thm:large_m_grad_converg_poly}, our convergence rate results critically depends on a lower bound for~\eqref{eq:appr_curv}. 

\begin{figure}[!t]
    \centering
    \includegraphics[width=0.6\textwidth]{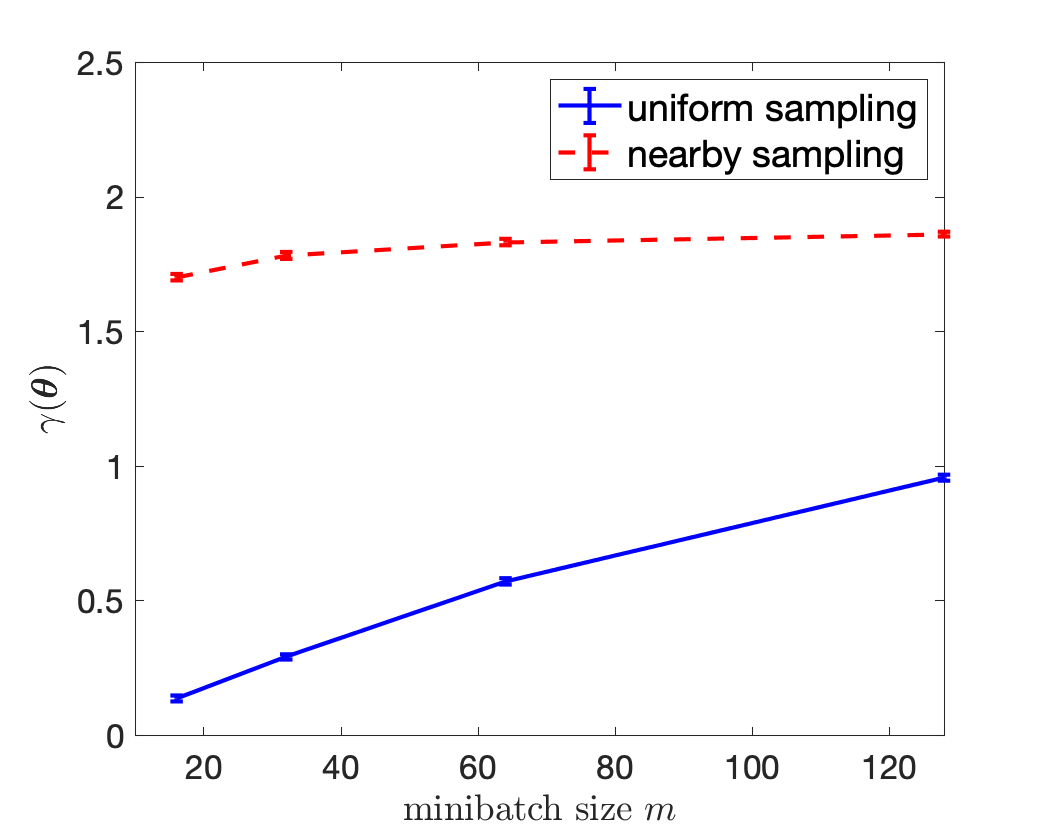}
    \caption{\small Curvature $\gamma(\vect{\theta})$, defined in~\eqref{eq:appr_curv} under different minibatch sizes $m$ and sampling schemes. For each minibatch size $m$ and sampling scheme, the mean value of $\gamma(\vect{\theta})$ over $50$ replicates is taken and the standard deviation is marked as the error bar.}
    \label{fig:curv_unif_nn}
\end{figure}
\paragraph{Numerical evidence:} Figure~\ref{fig:curv_unif_nn} demonstrates an example of the ``approximate'' curvature term \eqref{eq:appr_curv} with various minibatch sizes $m$ under the uniform sampling and nearby sampling schemes, where nearby sampling leads to larger curvatures. 

The detailed numerical experiments for generating Figure~\ref{fig:curv_unif_nn} is as follows. We randomly generate a full data pool including $n=2048$ data points: $x_1,\dots,x_n\overset{\mathrm{i.i.d.}}{\sim}\mathcal{N}(0,10^2)$. For each minibatch size $m$, we perform uniform sampling and nearby sampling for $50$ replicates; for uniform sampling the minibatch $\xi$ of size $m$ is sampled from $1,\dots,n$ uniformly at random; for nearby sampling, the first index $\xi_1$ in the minibatch $\xi$ is sampled uniformly at random from $1,\dots,n$, and the rest $m-1$ indices correspond to the nearest neighbors of $\vect{x}_{\xi_1}$. Then for each replicate, we calculate $\gamma(\vect{\theta})=\frac{1}{2m}\sum_{j=1}^m(\theta_1\lambda_j+\theta_2)^{-2}$, where $\vect{\theta}=(4,1)^\top$ and $\lambda_j$ is the $j$th eigenvalue of $\vect{K}_{f,\xi}$, the kernel matrix formed by $\vect{x}_i,i\in \xi$. The kernel function $k_0(\cdot)$ is set as the RBF kernel with lengthscale $l=0.5$. We then take the mean of $\gamma(\vect{\theta})$ over the $50$ replicates.

\paragraph{Theoretical insights:} Now we provide more theoretical insights into the influence of nearby sampling upon $\gamma(\vect{\theta})$. For simplicity of analysis, we fix our focus on the $D=1$ case, and let the kernel $k_0$ be the RBF kernel: $k_0(x,x')=\exp\{-(x-x')^2/(2l^2)\}$, $x_1,\dots,x_n\overset{\text{i.i.d.}}{\sim}\mathcal{N}(0,\sigma^2)$. First note that $\frac{\lambda_{j}^{(k)}}{m}$ converges to $\lambda_{j}^*$ (the $j$th largest eigenvalue of kernel $k_0(\cdot,\cdot)$) under uniform sampling, when minibatch size $m$ tends to $\infty$. Define $\widetilde{\gamma}$ by substituting $\lambda_{j}^{(k)}$ with $m\lambda_{j}^*$ ($\lambda_j^*$ is the $j$th eigenvalue of kernel $k$) in the definition \eqref{eq:appr_curv} of $\gamma^{(k)}$:
\begin{equation}
    \widetilde{\gamma}(\vect{\theta})=\frac{1}{m}\sum_{j=1}^m\left(\theta_1\lambda_{j}^{*}m+\theta_{2}\right)^{-2},
\end{equation}
then $\widetilde{\gamma}(\vect{\theta}^{(k)})$ is a reasonable approximation for $\gamma^{(k)}$ when $m$ is large enough under uniform sampling. 
\textit{If we increase the length scale $l$ under uniform sampling, then it is equivalent to decreasing the distance between points, and thus can serve as an approximation to nearby sampling.}

The following lemma illustrates how the length scale $l$ influences $\widetilde{\gamma}(\vect{\theta})$:
\begin{lemma}\label{lem:curv_lengthscale}
For any $l_0>0$, there exists $m_0>0$ depending on $\theta_1,\theta_2,\sigma,l_0$ such that as long as $m>m_0$, $\widetilde{\gamma}(\vect{\theta})$ is an increasing function of $l\geq l_0$.
\end{lemma}
 Lemma \ref{lem:curv_lengthscale} suggests that, for large enough minibatches, larger length scale leads to faster convergence for $\theta^{(k)}_2$ to $\theta^*_2=\sigma_{\epsilon}^2$. Therefore, our hypothesis that ``nearby sampling may improve the curvature'' is plausible. More numerical support will be provided in Section \ref{sec:numeric}.
\subsection{Extension to Summation of Multiple Kernels}\label{sec:sum-kernels}
So far we have assumed the kernel function $k_0(\cdot,\cdot)$ to be known. However, sometimes there might be several choices of potential kernels and it is desired to learn which kernel is the most appropriate from the data instead of manually picking one kernel. One possible approach is to let the covariance function $k(\cdot,\cdot)$ in \eqref{model} be a linear combination of all potential $M$ kernels:
\begin{equation}
    k(\cdot,\cdot)=\sum_{l=1}^M \sigma_{f,l}^2k_l(\cdot,\cdot),
\end{equation} 
and then learn the signal variances $\sigma_{f,l}^2$ associated with the $M>1$ kernels~\citep{rasmussen2003gaussian}. 
The kernel selection problem then translates to how well we can learn the signal variance parameters $\sigma_{f,l}^2$, $l=1,\dots M$. 

Under this extension of the classical GP model, let the hyperparameter $$\vect{\theta}^*=(\sigma_{f,1}^2,\dots,\sigma_{f,M}^2,\sigma_{\epsilon}^2)^\top\in\bbR^{M+1}.$$ We can still write out the log-likelihood loss $\ell(\vect{\theta})$ as in \eqref{eq:objective}, while the only difference lies in the formulation of $\vect{K}_n(\vect{\theta})$: $$\vect{K}_n(\vect{\theta})=\sum_{l=1}^M\theta_l\vect{K}_{f,n}^{(l)}+\theta_{M+1}\vect{I}_n,$$
where $\vect{K}_{f,n}^{(l)}$ is the kernel matrix of $k_l(\cdot,\cdot)$ evaluated at $\vect{X}_n$, i.e., $(\vect{K}_{f,n}^{(l)})_{ij}=k_l(\vect{x}_i,\vect{x}_j)$.
Then it is straightforward to extend Algorithm~\ref{alg:algorithm1} to this setting for learning $\vect{\theta}^*$. 

Similar to Theorem~\ref{thm:large_m_param_converg} to Theorem~\ref{thm:large_m_grad_converg_poly}, under some additional conditions on the kernel matrices, we also have convergence guarantees for $\sigma_{f,1}^2$, $\sigma_{\epsilon}^2$, and the full gradient when $k_l(\cdot,\cdot), 1\leq l\leq M$ have exponential eigendecay; for $\sigma_{\epsilon}^2$ and the full gradient when $k_l(\cdot,\cdot), 1\leq l\leq M$ have polynomial eigendecay. The detailed theoretical results for this setting are included in Appendix~\ref{append:sum-kernels}.
\section{Proof Sketch}\label{sec:proofoverview}
In this section, we present the proof sketch for the first part of Theorem \ref{thm:large_m_param_converg} and Theorem \ref{thm:large_m_grad_converg} (kernels with exponential eigendecay). The proof of the second part in Theorem \ref{thm:large_m_param_converg}, Theorem \ref{thm:large_m_param_converg_poly} and Theorem \ref{thm:large_m_grad_converg_poly} follows similar ideas although requiring more careful analysis. With a bit abuse of notation, we will omit $g(\vect{\theta}^{(k)};\vect{X}_{\xi_{k+1}},\vect{y}_{\xi_{k+1}})$ to $g(\vect{\theta}^{(k)})$ and denote its conditional expectation $\bbE(g(\vect{\theta}^{(k)})|\vect{X}_{\xi_{k+1}})$ by $g^*(\vect{\theta}^{(k)})$. Similarly we define $\nabla \ell^*(\vect{\theta}^{(k)})=\bbE(\nabla \ell(\vect{\theta}^{(k)})|\vect{X}_n)$.

Due to the bias in the stochastic gradient, we take the followings steps instead of directly drawing the connection between $g(\vect{\theta}^{(k)})$ and $\nabla \ell(\vect{\theta}^{(k)})$:
\begin{itemize}[leftmargin=*]\vspace{-2mm}
    \item For proving the first part of Theorem \ref{thm:large_m_param_converg}:
    \begin{itemize}
        \item We first show that the conditional expectation $g^*(\vect{\theta}^{(k)})$ of the stochastic gradient has a property similar to strong convexity, see Lemma \ref{lem:g_k_conv_property_abbrv}.
        \item We then prove that $g(\vect{\theta})$ is close to its conditional expectation $g^*(\vect{\theta})$ uniformly over all possible $\vect{\theta}$, and thus $g(\vect{\theta}^{(k)})$ is close to $g^*(\vect{\theta}^{(k)})$. Applying Lemma \ref{lem:statistical_err_full_grad_abbrv} to each minibatch leads to the desired result.
    \end{itemize}
    These two steps lead to the $O(\frac{1}{K})$ optimization error rate for $(\theta^{(k)}_2-\theta^*_2)^2$, and a statistical error rate depending on $m$, as shown in Theorem \ref{thm:large_m_param_converg}.
    \item For proving Theorem \ref{thm:large_m_grad_converg}:
    \begin{itemize}
        \item Lemma \ref{lem:statistical_err_full_grad_abbrv} suggests that $\nabla \ell(\vect{\theta}^{(k)})$ is close to $\nabla \ell^*(\vect{\theta}^{(k)})$\\
        \item The eigendecay of kernel matrices (see Lemma \ref{lem:eigen_ratio_bnds_abbrv}) ensures that $\|\nabla \ell^*(\vect{\theta}^{(k)})\|_2$ is controlled by  $(\theta^{(k)}_2-\theta^*_2)^2$, which is upper bounded in Theorem \ref{thm:large_m_param_converg}.
    \end{itemize}
    These steps above provide us with the same error bound of $\|\nabla \ell^*(\vect{\theta}^{(k)})\|_2$ from that of $(\theta^{(k)}_2-\theta^*_2)^2$ in Theorem \ref{thm:large_m_param_converg}.
\end{itemize}
\subsection{Key Lemmas}
The following two lemmas are the key building blocks of the proof: one shows that $g^*(\vect{\theta}^{(k)})$ enjoys a property similar to strong convexity, the other establishes a uniform bound for the statistical error $\nabla \ell(\vect{\theta})-\nabla \ell^*(\vect{\theta})$ over $\vect{\theta}\in [\theta_{\min},\theta_{\max}]^{2}$, and thus also bounds for $g(\vect{\theta}^{(k)})-g^*(\vect{\theta}^{(k)})$;
\begin{lem}[Strongly convex-like property of $g^*(\vect{\theta}^{(k)})$]\label{lem:g_k_conv_property_abbrv}
Under Assumptions \ref{assump:param_bnd} to \ref{assump:eigen_decay_exp}, if $s_2(m)=m$, $m>C$, then with probability at least $1-3Km^{-c}$, the following claim holds true for $0\leq k< K$:
       \begin{equation}\label{eq:strong_conv_abbrv}
           (\theta^{(k)}_2-\theta^*_2)(g^*(\vect{\theta}^{(k)}))_2
           \geq \frac{1}{8\theta_{\max}^2}(\theta^{(k)}_2-\theta^*_2)^2-\frac{C\log m}{m},
           \end{equation}
   Here $C>0$ depends only on $\theta_{\min},\theta_{\max},b$.
\end{lem}
Lemma \ref{lem:g_k_conv_property_abbrv} is a relaxation of strong convexity, but leads to similar convergence guarantees from running SGD on strongly convex objectives. 
The approximate ``curvature" parameter, $\frac{1}{8\theta_{\max}^2}$ on the R.H.S of \eqref{eq:strong_conv_abbrv}, remains a constant regardless of how large $m$ is. To guarantee the constant ``curvature'', we establish novel upper and lower bounds on $\sum_{j=1}^m\lambda_{j}^l(\theta^{(k)}_1\lambda_{j}+\theta^{(k)}_2)^{-2}, l=0,1,2$ with high probability when $m$ is large, where $\lambda_{j}$ is the $j$th largest eigenvalue of $\vect{K}_{f,\xi}$ (see Lemma \ref{lem:eigen_ratio_bnds_abbrv}). The proof of Lemma \ref{lem:eigen_ratio_bnds_abbrv} is based on established error bounds for the empirical eigenvalues in \cite{braun2006accurate} and the eigendecay of the kernel $k_0(\cdot,\cdot)$.

\begin{lem}[Uniform statistical error]\label{lem:statistical_err_full_grad_abbrv} Under Assumptions \ref{assump:param_bnd} to \ref{assump:sg_bnd}, Assumption \ref{assump:eigen_decay_exp} or \ref{assump:eigen_decay_poly}, for any $0<\varepsilon<\frac{1}{2}$, $1\leq i\leq 2$, we have
\begin{equation}\label{eq:full_grad_exp_bnd_abbrv}
    \bbP\left(\sup_{\vect{\theta}\in [\theta_{\min},\theta_{\max}]^{2}}\left|(\nabla \ell(\vect{\theta}))_i-(\nabla \ell^*(\vect{\theta}))_i\right|> Cn^{-\frac{1}{2}+\varepsilon}\right)\leq C\exp\{-cn^{2\varepsilon}\}.
\end{equation}
 Here $c,C>0$ only depend on $\theta_{\min}$, $\theta_{\max},b$. 
\end{lem}
The major difficulty in the proof of Lemma \ref{lem:statistical_err_full_grad_abbrv} is to control the error term uniformly over $\vect{\theta}\in [\theta_{\min},\theta_{\max}]^{2}$.  We need a uniform error bound, since $g^*(\vect{\theta}^{(k)})$ is no longer the conditional expectation of $g(\vect{\theta}^{(k)})$ if conditioning on the past iterate $\vect{\theta}^{(k)}$. Although the set $[\theta_{\min},\theta_{\max}]^{2}$ has constant dimension, the kernel matrix $\vect{K}_n(\vect{\theta})\in \bbR^{n\times n}$ is of high dimension and is determined by $\vect{\theta}$ in a non-linear way. Our solution is to explore the Taylor's expansion of $\nabla \ell(\vect{\theta})-\nabla\ell^*(\vect{\theta})$, then use truncation and covering arguments. 
\subsection{Proof of the First Part of Theorem \ref{thm:large_m_param_converg}}
Let $\widehat{e}_k= (g(\vect{\theta}^{(k)}))_{2}-(g^*(\vect{\theta}^{(k)}))_{2}$.
Due to Lemma \ref{lem:g_k_conv_property_abbrv} and Assumption \ref{assump:sg_bnd}, we have
 \begin{equation}\label{eq:SGD_err_bnd_recursive}
     \begin{split}
         \left(\theta^{(k)}_{2}-\theta^*_{2}\right)^2=&\left(\theta^{(k-1)}_{2}-\theta^*_{2}\right)^2-2\alpha_k (\theta^{(k-1)}_{2}-\theta^*_{2})(g(\vect{\theta}^{(k-1)}))_{2}\\
         &+\alpha_k^2(g(\vect{\theta}^{(k-1)}))_{2}^2\\
         \leq &\left(\theta^{(k-1)}_{2}-\theta^*_{2}\right)^2(1-\alpha_k\gamma)+\alpha_k^2G^2\\
         &+2\alpha_k\left(\frac{C\log m}{m}-(\theta^{(k-1)}_{2}-\theta^*_{2})\widehat{e}_{k-1}\right),
     \end{split}
 \end{equation}
 where $\gamma=\frac{1}{4\theta_{\max}^2}$.
  Recall that $\frac{3}{2\gamma}\leq\alpha_1\leq \frac{2}{\gamma}$, and $\alpha_k=\frac{\alpha_1}{k}$ for all $k\geq 1$. Now we prove the following statement for $k\geq 1$ by induction:
  \begin{equation}\label{eq:SGD_err_bnd1}
  \left(\theta^{(k)}_{2}-\theta^*_{2}\right)^2\leq \frac{2\alpha_1^2G^2}{k+1}+\sum_{i=0}^{k-1}\eta_{k,i}\left(\frac{C\log m}{m}-(\theta^{(k-1)}_{2}-\theta^*_{2})\widehat{e}_{k-1}\right),
  \end{equation}
  where $\eta_{k,i}=2\alpha_{i+1}\prod_{j=i+2}^k(1-\alpha_j\gamma)$. When $k=1$, by \eqref{eq:SGD_err_bnd_recursive} and the fact that $1-\alpha_1\gamma<0$,
  \begin{equation}
      \left(\theta^{(1)}_{2}-\theta^*_{2}\right)^2\leq \alpha_1^2G^2+\eta_{1,0}\left(\frac{C\log m}{m}-(\theta^{(0)}_{2}-\theta^*_{2})\widehat{e}_{0}\right).
  \end{equation}
  Assuming \eqref{eq:SGD_err_bnd1} holds for $k=l\geq 1$, then due to \eqref{eq:SGD_err_bnd_recursive} and the fact that $1-\alpha_{l+1}\gamma\geq0$ for $l\geq 1$, we have
  \begin{equation}
      \begin{split}
          &\left(\theta^{(l+1)}_{2}-\theta^*_{2}\right)^2\\
          \leq &\left(\frac{2\alpha_1^2G^2}{l+1}+\sum_{i=0}^{l-1}\eta_{l,i}\left(\frac{C\log m}{m}-(\theta^{(i)}_{2}-\theta^*_{2})\widehat{e}_{i}\right)\right)(1-\alpha_{l+1}\gamma)+\alpha_{l+1}^2G^2\\
          &+2\alpha_{l+1}\left(\frac{C\log m}{m}-(\theta^{(l)}_{2}-\theta^*_{2})\widehat{e}_{l}\right)\\
          \leq &\frac{2\alpha_1^2G^2(l+1-\alpha_1\gamma)}{(l+1)^2}+\frac{\alpha_1^2G^2}{(l+1)^2}+\sum_{i=0}^{l}\eta_{l+1,i}\left(\frac{C\log m}{m}-(\theta^{(i)}_{2}-\theta^*_{2})\widehat{e}_{i}\right)\\
          \leq &\frac{2\alpha_1^2G^2}{l+2}+\sum_{i=0}^{l}\eta_{l+1,i}\left(\frac{C\log m}{m}-(\theta^{(i)}_{2}-\theta^*_{2})\widehat{e}_{i}\right).
      \end{split}
  \end{equation}
  Here the last two lines are due to range of $\alpha_1$ and the definitions of $\eta_{l,i}$. The next step is to bound $\sum_{i=0}^{K-1}\eta_{K,i}\left(\frac{C\log m}{m}-(\theta^{(i)}_{2}-\theta^*_{2})\widehat{e}_{i}\right)$. First we have
  \begin{equation}
  \begin{split}
      &\left|\sum_{i=0}^{K-1}\eta_{K,i}\left(\frac{C\log m}{m}-(\theta^{(i)}_{2}-\theta^*_{2})\widehat{e}_{i}\right)\right|\\
      \leq &\frac{2\alpha_1}{K}\sum_{i=0}^{K-1}\left|\theta^{(i)}_{2}-\theta^*_{2}\right|\widehat{e}_i|+\frac{C\alpha_1\log m}{m}\\
      \leq &C\left(\max_{0\leq i\leq K-1}|\widehat{e}_i|+\frac{\log m}{m}\right).
  \end{split}
  \end{equation}
  Note that the distribution of each minibatch $\{\vect{X}_{\xi_{k+1}},\vect{y}_{\xi_{k+1}}\}_{i=1}^m$ is the same as sampling $m$ independent $\vect{x}_i$ from $\bbP$, and then sampling $\vect{y}_{\xi_{k+1}}\sim\mathcal{N}(0,\vect{K}^*_{\xi_{k+1}})$, thus we can apply the Lemma \ref{lem:statistical_err_full_grad_abbrv} on each $\widehat{e}_i=g(\vect{\theta}^{(i)})_{2}-g^*(\vect{\theta}^{(i)})_{2}$ and take a union bound over $0\leq i\leq K-1$:
  \begin{equation}
      \bbP\left(\max_{0\leq i\leq K-1} \left|\widehat{e}_i\right|>Cm^{-\frac{1}{2}+\varepsilon}\right)\leq CK\exp\{-cm^{2\varepsilon}\},
  \end{equation}
  for any $\varepsilon>0$.
  Therefore, 
  \begin{equation}
      \left(\theta^{(k)}_{2}-\theta^*_{2}\right)^2\leq \frac{2\alpha_1^2G^2}{k+1}+Cm^{-\frac{1}{2}+\varepsilon},
  \end{equation}
  with probability at least 
  \begin{equation}
      1-CK\exp\{-c\min\{\log m, m^{2\varepsilon}\}\geq 1-CK\exp\{-cm^{2\varepsilon}\},
  \end{equation}
for any $0<\varepsilon<C\frac{\log \log m}{\log m}$, when $m>C$ for some $C>0$ depending on $\theta_{\min},\theta_{\max},b$.

   
  
  


\subsection{Proof of Theorem \ref{thm:large_m_grad_converg}}
We start from bounding $\nabla \ell^*(\vect{\theta}^{(k)})$, the conditional expectation of $\nabla \ell(\vect{\theta}^{(k)})$ given $\vect{x}_1,\dots,\vect{x}_n$, then control the statistical error $\nabla \ell(\vect{\theta}^{(k)})-\nabla \ell^*(\vect{\theta}^{(k)})$. By the definition of $\nabla \ell^*(\vect{\theta}^{(k)})$, for $1\leq i\leq 2$,
\begin{equation}\label{eq:full_gradient_exp}
    \begin{split}
        \left(\nabla \ell^*(\vect{\theta}^{(k)})\right)_i=&\frac{1}{2n}\tr\left[\vect{K}_n(\vect{\theta}^{(k)})^{-1}(\vect{I}_n-\vect{K}_n^*\vect{K}_n(\vect{\theta}^{(k)})^{-1})\frac{\partial \vect{K}_n(\vect{\theta}^{(k)})}{\partial \theta^{(k)}_i}\right]\\
        =&\frac{1}{2n}\sum_{j=1}^n \frac{(\theta^{(k)}_1-\theta^{*}_1)\lambda_{j}^{1+\ind{i=1}}+(\theta^{(k)}_2-\theta^{*}_2)\lambda_{j}^{\ind{i=1}}}{\left(\theta^{(k)}_1\lambda_{j}+\theta^{(k)}_2\right)^2},
    \end{split}
\end{equation}
where $\lambda_{j}$ is the $j$th largest eigenvalue of $\vect{K}_{f,n}$.
The following lemma provides bounds for $\sum_{j=1}^n \frac{\lambda_{j}^l}{\left(\theta^{(k)}_1\lambda_{j}+\theta^{(k)}_2\right)^2}$ for all $l=0,1,2$.
\begin{lemma}\label{lem:eigen_ratio_bnds_abbrv}
Under Assumption \ref{assump:eigen_decay_exp}, for any $\alpha>0$, if $n>C$ for $C>0$ depending on $b$, then with probability at least $1-3n^{-\alpha}$,
\begin{equation}\label{eq:eigen_ratio_bnds_abbrv}
\begin{split}
    \text{ if }l = 1\text{ or }2, \sum_{j=1}^n \frac{\lambda_{j}^l}{\left(\theta_1\lambda_{j}+\theta_2\right)^2}\leq &\frac{2(2+\alpha)}{b\theta_{\min}^2}\log n,\\
    \sum_{j=1}^n \frac{1}{\left(\theta_1\lambda_{j}+\theta_2\right)^2}\leq &\frac{n}{\theta_{\min}^2},
\end{split}
\end{equation}
holds for any $\vect{\theta}\in[\theta_{\min},\theta_{\max}]^{2}$.
\end{lemma}
We prove Lemma \ref{lem:eigen_ratio_bnds_abbrv} by exploiting the error bounds for eigenvalues of empirical kernel matrices from the population eigenvalues of the kernel operator. A detailed version of Lemma \ref{lem:eigen_ratio_bnds_abbrv} including also lower bounds for $\frac{\lambda_j^l}{\left( \theta_1\lambda_{j}+\theta_2\right)^2}$ is presented in the Appendix, which is a key result for proving Lemma \ref{lem:g_k_conv_property_abbrv}. 

For any constant $c>0$, apply Lemma \ref{lem:eigen_ratio_bnds_abbrv} with $\alpha=c$, then \eqref{eq:eigen_ratio_bnds_abbrv} holds with probability at least $1-3n^{-c}$, if $n>C$ for $C$ depending on $b$. Combining this result and \eqref{eq:full_gradient_exp} together implies
that
\begin{equation}
    \begin{split}
        \left|\left(\nabla \ell^*(\vect{\theta}^{(k)})\right)_1\right|\leq &\frac{C\log n}{n}
    \end{split}
\end{equation}
where $C>0$ depends on $\theta_{\min},\theta_{\max},b$. Meanwhile,
\begin{equation}
    \begin{split}
        \left|\left(\nabla \ell^*(\vect{\theta}^{(k)})\right)_2\right|\leq C\left(|\theta^{(k)}_2-\theta^*_2|+\frac{\log n}{n}\right).
    \end{split}
\end{equation}
Thus we have 
\begin{equation}
    \begin{split}
        \|\nabla \ell^*(\vect{\theta}^{(k)})\|_2^2\leq& C\left[\left(\frac{\log n}{n}\right)^2+(\theta^{(k)}_2-\theta^*_2)^2\right].
    \end{split}
\end{equation}
 For bounding $\nabla \ell(\vect{\theta}^{(k)})-\nabla \ell^*(\vect{\theta}^{(k)})$, we can apply Lemma \ref{lem:statistical_err_full_grad_abbrv}.
By \eqref{eq:full_grad_exp_bnd_abbrv} and Theorem \ref{thm:large_m_param_converg}, for any $0<\varepsilon<C\frac{\log \log m}{\log m}$, if $m>C$, 
then with probability at least $1-CK\exp\{-cm^{2\varepsilon}\}$, we have
\begin{equation}
    \|\nabla \ell(\vect{\theta}^{(K)})\|_2^2\leq C\left[\frac{G^2}{K+1}+m^{-\frac{1}{2}+\varepsilon}\right],
\end{equation}
where $c,C>0$ depend only on $\theta_{\min},\theta_{\max},b$.

\section{Practical Considerations for Applying SGD on GP}\label{sec:practical}

\subsection{Sampling Scheme}
As discussed in Section \ref{sec:problem_setup}, one may consider both uniformly and nearby sampling. In Section \ref{sec:nearby_sampling}, we provide both numerical and theoretical evidence that hint to the fact that sampling nearby points for each minibatch leads to faster convergence specifically for the noise variance parameter $\theta_2^{(k)}$ towards $\theta_2^*=\sigma_{\epsilon}^2$. Besides that, our case studies in numerical experiments in Section~\ref{sec:numeric} also support this claim.  

Below we highlight the Algorithm for nearby sampling which is a simple extension of Algorithm \ref{alg:algorithm1}

\begin{algorithm}[H]
\SetAlgoLined
 Input: $\vect{\theta}^{(0)}\in \bbR^{2}$, initial step size $\alpha_1>0$.\\
 \For{$k=1,2,\dotsc,K$}{
  Sample a data point uniformly from the data pool, and then select its $m-1$ nearest neighbors, which forms $(\mathbf{X}_{\xi_k},\mathbf{y}_{\xi_k})$ of size $m$\; 
  Compute the stochastic gradient $g(\vect{\theta}^{(k-1)};\mathbf{X}_{\xi^{(2)}_k}, \mathbf{y}_{\xi^{(2)}_k})$\;
  $\alpha_k\leftarrow \frac{\alpha_1}{k}$\;
  $\vect{\theta}^{(k)}\leftarrow \vect{\theta}^{(k-1)} - \alpha_kg(\vect{\theta}^{(k-1)};\mathbf{X}_{\xi_k}, \mathbf{y}_{\xi_k})$\;
 }
 \caption{Minibatch SGD with nearby sampling}
 \label{alg:mixed_sampling}
\end{algorithm}

\subsection{Optimizing other hyperparameters}
In practice, we may also need to determine other hyperparameters of the kernel function besides signal variance and noise variance. For example, when considering the RBF kernel $k(\mathbf{x},\mathbf{x}')=\exp\{-\sum_{j=1}^d\frac{(x_j-x'_j)^2}{2l_j^2}\}$, we need to estimate the lengthscale parameters $l_j, 1\leq j\leq d$; when considering the Mat\`{e}rn kernel~\eqref{eq:matern}, the hyperparameters $\alpha$ and $h$ are also unknown and require estimation. Similar to Algorithm \ref{alg:algorithm1} and Algorithm \ref{alg:mixed_sampling}, we can update these parameters alongside the variance parameters using minibatch SGD. Our numerical experiments in Section~\ref{sec:numeric} suggest that nearby sampling method may also be a good option for the lengthscale parameters.

\subsection{Prediction}
Although our main focus in this paper is estimating the hyperparameters, or model selection, the last step when applying GP in real applications is always prediction. Following the model selection process from which we obtain optimal hyperparameters, various strategies can be applied to calculate the predictive mean for $\mathbf{x}_*$ and the predictive covariance between $\mathbf{x}_{*}$ and $\mathbf{x}_{*}'$ using the well known predictive equation below

\begin{equation}\label{eq:pred}
    \mu_{\text{pred}}(\mathbf{x}_*)=\mathbf{k}_{\mathbf{X}_n\mathbf{x}_*}^\top\mathbf{K}_n^{-1}\mathbf{y}_n,\quad\quad k_{\text{pred}}(\mathbf{x}_*,\mathbf{x}_{*}')=k(\mathbf{x}_*, \mathbf{x}_{*}') -\mathbf{k}_{\mathbf{X}_n\mathbf{x}_*}^\top\mathbf{K}_n^{-1}\mathbf{k}_{\mathbf{X}_n\mathbf{x}_{*}'} \, ,
\end{equation}
where $\mathbf{k}_{\mathbf{X}_n\mathbf{x}_*}=(k(\mathbf{x}_1,\mathbf{x}_*)\dotsc,k(\mathbf{x}_n,\mathbf{x}_*))^\top$. The main computational cost of the predictive mean and the predictive covariance come from $\mathbf{K}_n^{-1}$. In general, for $n<10^4$, they can be computed via Cholesky decomposition; for $n<10^5$, preconditioned conjugate gradient (PCG) \citep{gardner2018gpytorch} can be applied for acceleration; for $n<10^6$, PCG with partitioned kernel \citep{wang2019exact} could provide further speed up, if distributed computational resources are available. Another practical but less ideal strategy when predicting with extremely large $n$ is to follow the same approach as nearby sampling and utilize only $\tilde{n}$ nearest neighbors of $\mathbf{x}_*$ within the observed data to solve \eqref{eq:pred}, where $\tilde{n}<n$ is determined by the available computational resource. Fortunately, prediction is a one-shot process compared to the iterative training process. 

\section{Numerical Results}\label{sec:numeric}
\subsection{Numerical Illustration of Theory}\label{sec:numeric_theory}

In this section, we conduct simulation studies to verify our theoretical results. 

We consider $n=1,024$, $\mathbf{x}_i\overset{i.i.d.}{\sim}\mathcal{N}(0,5^2)$ and $\mathbf{y}_n\sim \mathcal{N}(\mathbf{0},\sigma_f^2\mathbf{K}_{f,n} + \sigma_{\epsilon}^2\mathbf{I}_n)$, where $\mathbf{K}_{f,n}$ is an RBF kernel matrix with known lengthscale $l=0.5$. The underlying true parameters are outputscale $\sigma_f^2=4$ and noise variance $\sigma_{\epsilon}^2=1$. In each experiment, we perform 25 epochs of minibatch SGD with diminishing step sizes $\alpha_k=\alpha_1/k$. We set scaling factors to $s_1(m)=3\log m$ for $\sigma_f^2$ and $s_2(m)=m$ for $\sigma_n^2$. Each experiment is repeated 10 times with independent data pools for different repetitions.

\begin{figure}[htbp!]
\centering
  \includegraphics[width=\linewidth]{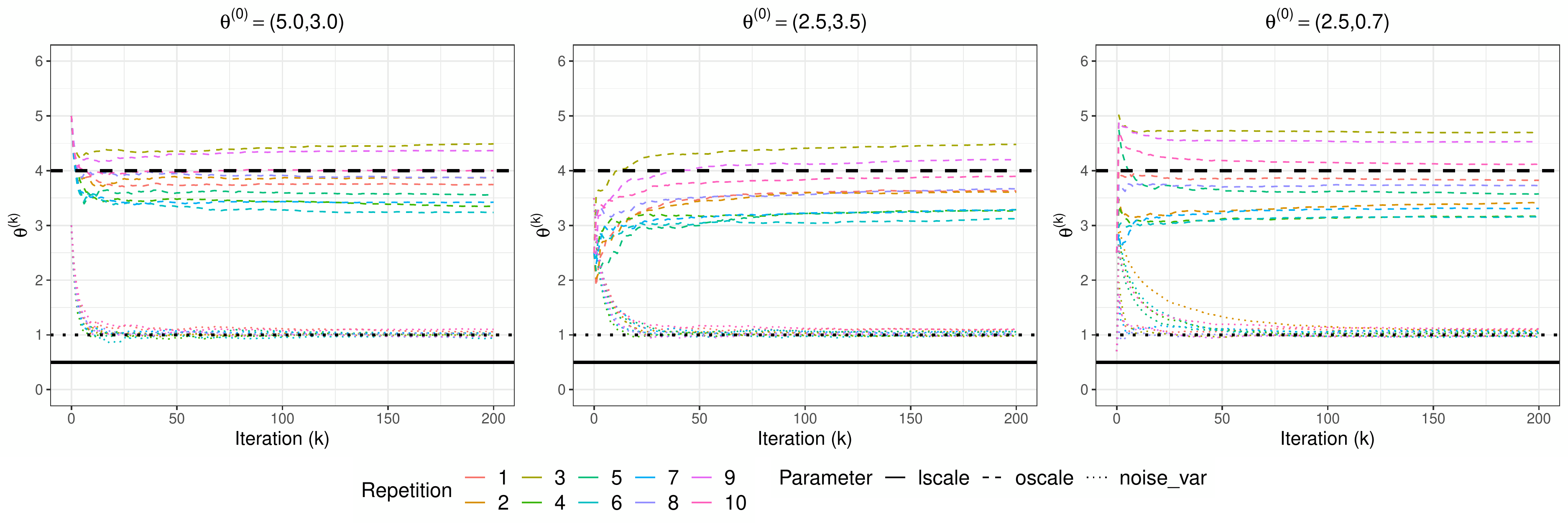}
  \caption{Illustration of the convergence of parameters under uniform sampling. We consider $m=128$, and demonstrate three cases with varying initial points where initial stepsizes are $\alpha_1=$ $9, 9$ and $6$, respectively. Lines in black denote the true parameters.}
  \label{fig:param_unif}
\end{figure}

Fig. \ref{fig:param_unif} shows the convergence of parameters under uniform sampling, varying initializations and step sizes. All, the curves display $O(\frac{1}{K})$ convergence rates which are consistent with our results in Theorem \ref{thm:large_m_param_converg}. Moreover, the locations where the updates of $\sigma_{\epsilon}^2$ converges to, are significantly more concentrated around the truth compared to that of $\sigma_f^2$, which is consistent with the $O((\log m)^{-\frac{1}{2}})$ statistical error for $\sigma_f^2$ and $O(m^{-\frac{1}{2}})$ statistical error for $\sigma_{\epsilon}^2$, also stated in Theorem \ref{thm:large_m_param_converg}. 


\begin{figure}[htbp!]
\centering
  \includegraphics[width=\linewidth]{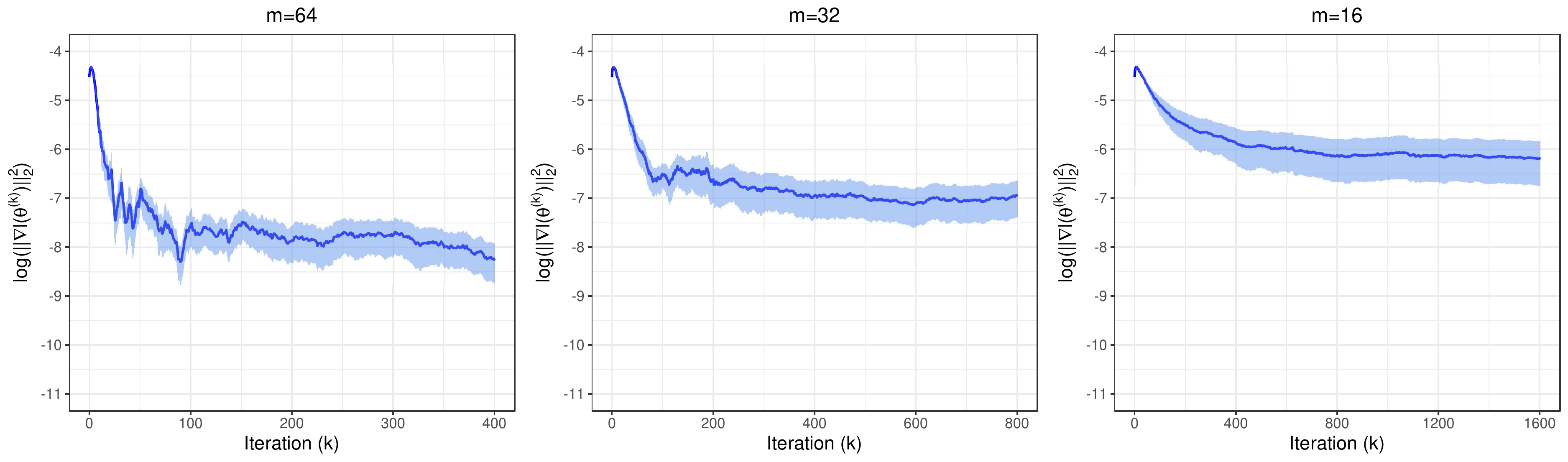}
  \caption{Comparison of the convergence of the full gradient under uniform sampling with varying minibatch sizes. The mean of $\|\nabla \ell(\vect{\theta}^{(k)})\|_2^2$ is shown in blue and the region within its one standard error over 10 repetitions is shown in light blue, both under log scale. The three experiments share initial point $\vect{\theta}^{(0)}=(5.0,3.0)$ and inital step size $\alpha_1=9$.} 
  \label{fig:grad_unif}
  \vspace{-0.2cm}
\end{figure}

Fig. \ref{fig:grad_unif} displays the effect of minibatch size on the convergence of the full gradient. To start with, the curves flatten slower and converge to larger values as minibatch size decreases, suggesting that a larger minibatch leads to faster convergence of the full gradient, as well as a full gradient with smaller statistical error. In addition, the convergence points of $\log(\|\nabla \ell(\vect{\theta}^{(k)})\|_2^2)$ scale linearly with minibatch size $m$, indicating a $O(m^{-\frac{1}{2}})$ statistical error for $\|\nabla \ell(\vect{\theta}^{(k)})\|_2^2$. The above observations confirm our statements in Theorem \ref{thm:large_m_grad_converg}. 

\begin{figure}[htbp!]
\centering
  \includegraphics[width=\linewidth]{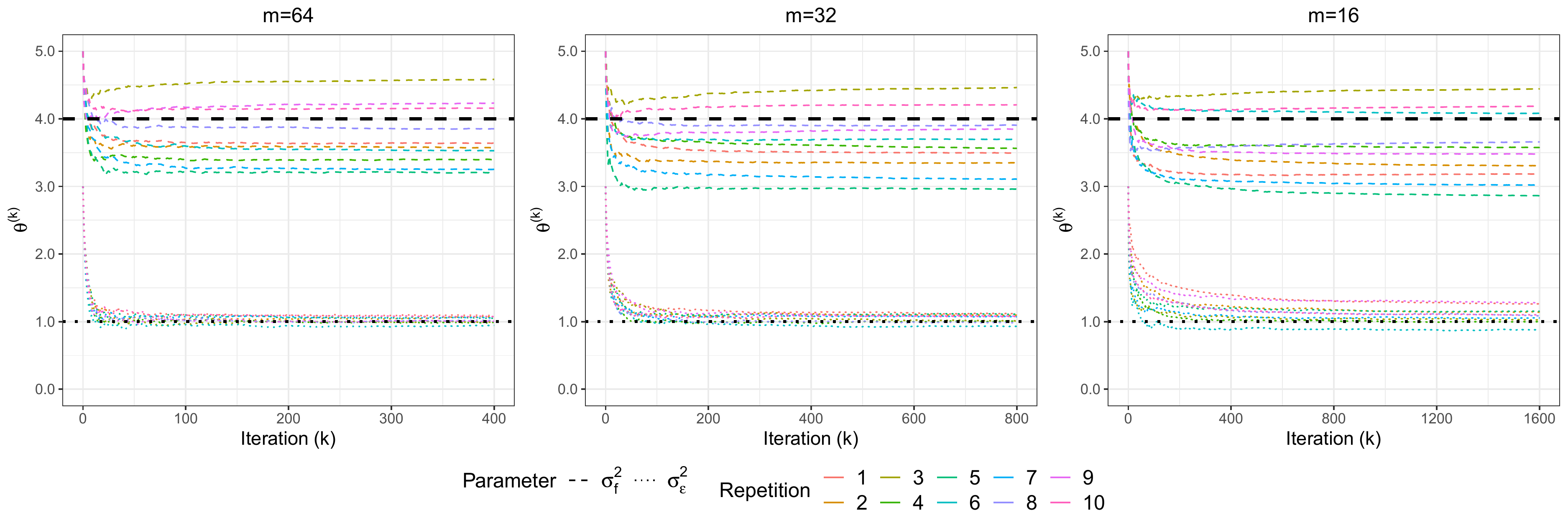}
  \caption{Comparison of the convergence of parameters with varying minibatch sizes. Lines in black denote the true parameters. The three experiments share initial point $\vect{\theta}^{(0)}=(5.0,3.0)$ and inital step size $\alpha_1=9$.} 
  \label{fig:vary_m}
\end{figure}

We also investigate how minibatch size $m$ influences the convergence of parameters, which is illustrated in Fig. \ref{fig:vary_m}. As highlighted in our theory, we find that a larger mini-batch size results in faster convergence and smaller statistical error (more concentrated curves) for the parameters. Here we note that similar results hold for the Mat\'{e}rn kernel and hence we omit the figures.

\subsection{Case Studies}\label{sec:case_studies} 

In this section, we test our model's performance on publicly available real and simulated datasets. Our benchmarked models are: (i) Exact inference using  matrix vector multiplication denoted as EGP \citep{gardner2018gpytorch, wang2019exact},  (ii)  Vecchia's GP approximation denoted as Vecchia \citep{guinness2018permutation, katzfuss2020vecchia}, (iii) sparse GP regression denoted as SGPR \citep{titsias2009variational} and (iv) stochastic variational GP denoted as SVGP \citep{hensman2013gaussian}. Our stochastic gradient-based GP approach is denoted as sgGP.

All models are tested on real datasets from the UCI repository \citep{Dua:2019} and simulated datasets from the Virtual Library of Simulation Experiments \citep{simulationlib}.The real datasets are: Bike, Energy, PM2.5, Protein and Query. The simulated dataset are Levy, Greiwank and Borehole. We also use two other simulated datasets from the Virtual Library Simulation Experiments that represent real-life systems. The OTL circuit models an output transformerless push-pull circuit while the Wing Weight models a light aircraft wing. 


Throughout all experiments, we consider constant zero prior mean function and the scaled RBF covariance function with a separate lengthscale for each input dimension. We run Adam to learn the signal variance, noise variance, and also lengthscales, as an extension from our problem set-up stated in Section \ref{sec:problem_setup}.
We conduct 10 independent trials on each dataset. In each trial, we randomly split the dataset into 60\% training set and 40\% test set. In addition, the training set is normalized to 0 mean and 1 standard deviation, and the test set is scaled accordingly. 

During model selection, the hyperparameters and variational parameters are learned through minimizing the negative log marginal likelihood or its surrogate. (i) For sgGP, we consider both uniform and nearest neighbor sampling schemes, where we perform 100 epochs of Adam with minibatch size $m=16$ and a learning rate of 0.01. (ii) For Vecchia, we order the observed data following the maximum minimum distance (MMD) ordering \citep{guinness2018permutation}. MMD ordering works by first selecting a center point, and then sequentially selecting the next point to have maximum minimum distance to all previously selected points. We let each observed response condition on its $m=16$ nearest neighbors within its predecessors from the ordered set. We also carry out 100 iterations of the Fisher scoring \citep{guinness2021gaussian} algorithm. (iii) For other methods, we follow the theoretical recommendations in \cite{burt2019rates} and the practical recommendations in \cite{wang2019exact}. Further, for EGP, we perform 100 iterations of Adam with a learning rate of 0.1.
For SGPR, we use $m=512$ inducing points and carry out 100 iterations of Adam with a learning rate of 0.1. For SVGP, we use $m=1,024$ inducing points and perform 100 epochs of Adam with a minibatch size of $1,024$ and a learning rate of 0.01. To ensure fairness of comparison, we do not perform any pretraining or fine-tuning, and we let different methods share a common but randomly selected starting point in each trial.

Regarding the prediction of sgGP, we adopt the PCG algorithm in EGP to approximate \eqref{eq:pred}. While for prediction in Vecchia, we order the inputs to be predicted using MMD ordering and append them to the ordered observed inputs. We set the conditioning-set size to $m=64$.

Model selection of sgGP is coded with R, where RANN package \citep{RANN} is used for finding nearest neighbors. Vecchia is coded using R, where we utilize GpGp package \citep{guinness2018permutation} to find ordered nearest neighbors and implement Fisher's scoring algorithm. The prediction of sgGP, together with EGP, SGPR and SVGP are implemented through GPyTorch \citep{gardner2018gpytorch}. Each experiment is performed on a single core of Intel Xeon E5-2680 v3 @ 2.50GHz CPU. For simulated datasets, we manually inject noise to the response. For query dataset, we constrain the learned noise to be at least 0.1 to regularize the ill-conditioned kernel matrix. Due to memory limit, for Borehole, OTL Circuit and Wing Weight datasets, we use PCG algorithm for prediction in sgGP yet using only 60,000 nearest neighbors of each test point. 


The results of our experiments are shown in Tables \ref{table:rmse} - \ref{table:toy}. We start first by analyzing Tables \ref{table:rmse} and \ref{table:learned_noise_var}. Table \ref{table:rmse} summarizes the prediction accuracy of all benchmarked methods while Table \ref{table:learned_noise_var} highlights the accuracy of the learned noise variance $\hat{\sigma}_n^2$ on simulated datasets where we know the underlying truth. Based on the results, one can derive many insights.

First, we find that sgGP equipped with nearest-neighbor sampling (sgGP (nn)) exhibits the best predictive performance among the various methods on datasets with varying sizes, input dimensions, and noise levels. In addition, its learned noise variance is significantly more accurate than all benchmarks. Second, while sgGP (uni) can sometimes achieve good performance, it performs poorly in comparison to sgGP (nn), Vecchia, and EGP. This supports our numerical and theoretical evidence of the advantages of nearby sampling in Section \ref{sec:nearby_sampling}. Third, Vecchia does perform well overall in terms of prediction performance, which is contrary to the finding in \cite{jankowiak2021scalable}. Most likely, the heuristic MMD ordering we adopted offers significant improvement in model approximation over the default coordinate-based ordering \citep{guinness2018permutation}. However, Vecchia significantly underestimates the noise level and subsequently yields lower prediction accuracy than sgGP (nn). Here, it should be noted that the ordering of observations is crucial for the quality of Vecchia's approximation, and therefore, extensive effort towards dataset-specific tuning may be required, yet there lacks heuristic guidance and theoretical support for datasets of higher dimensions. Fourth, EGP exhibits inferior prediction accuracy compared to sgGP (nn). This highlights the ability of sgGP (nn)  to learn parameters that generalize better as both sgGP and EGP aim at exact inference. Yet, it should be noted that while EGP tackles exact inference, it features many approximations within. Finally, we find that SGPR and SVGP both do poorly overall and yield twice the prediction errors of sgGP (nn) on datasets like Levy, PM 2.5 and Query. Also, SGPR and SVGP (especially) tend to exaggerate the noise level \citep{bauer2016understanding, jankowiak2020parametric}, as seen in Table \ref{table:learned_noise_var}. Similar to \cite{wang2019exact}, this finding sheds light on the ability of exact GPs to significantly benefit from the increase in the number of training points.

Table \ref{table:train_time} summarizes the training time of all competing methods. The results exhibit the overwhelming time advantage of sgGP in training, which significantly scales with dataset size. Not only does sgGP achieve better generalization, but it also does that in a fraction of the training time needed for competing methods. This result is again confirmed by our test of the application-driven simulated datasets of size $2\times 10^6$ in Table \ref{table:toy}. Remarkably, it takes around 30 minutes to perform model selection for OTL Circuit dataset using a single core with R functions that are not designed for fast execution. In addition, sgGP enjoys superior memory efficiency due to the use of minibatches. These experiments justify that SGD does open up a new data size regime for exploring GPs.  Here we note that we are aware that EGP is designed to leverage multiple GPU parallelization; however, much like regular SGD, sgGP can be readily extended to a batch version where the gradient estimate in each update is the average of $M$ gradient estimates from $M$ mini-batches. This allows sgGP to take advantage of parallel computing when the hardware is available.

\renewcommand{\arraystretch}{1.0}

\begin{table}[htbp!]
\centering

\caption{We summarize the RMSE of sgGP and other GPs on benchmark datasets. Here and elsewhere, we report the averages $\pm$ standard errors over $10$ dataset splits. Best results are in bold (lower is better). sgGP (unif) utilizes uniform minibatches for training and sgGP (nn) utilizes nearest neighbor minibatches for training. For query and borehole datasets, we are unable to train with EGP due to memory limit.}
\label{table:rmse}
\begin{adjustbox}{max width=\textwidth}
\begin{tabular}[c]{c c c c c c c c c} 
\toprule
 & & & \multicolumn{6}{c}{\textbf{RMSE}} \\
\cmidrule(lr){4-9}
\textbf{Dataset} & Size & $D$ & 
\textbf{sgGP (uni)} & \textbf{sgGP (nn)} & \textbf{Vecchia} & \textbf{EGP} & \textbf{SGPR} & \textbf{SVGP} \\ 
\midrule
Levy       & 10,000     & 4  & 
$0.593\pm 0.002$ & $\mathbf{0.264\pm 0.002}$ & $0.309\pm 0.003$ & $0.316\pm 0.004$ & $0.564\pm 0.010$ & $0.582\pm 0.013$ \\
Griewank   & 10,000     & 6  & 
$0.149\pm 0.003$ & $0.070\pm 0.000$ & $0.081\pm 0.007$ & $\mathbf{0.064\pm 0.000}$ & $0.132\pm 0.003$ & $0.093\pm 0.005$ \\
Bike       & 17,379     & 17 & 
$0.227\pm 0.002$  & $\mathbf{0.220\pm 0.002}$ & $0.223\pm 0.002$  & $0.228\pm 0.002$ & $0.276\pm 0.004$ & $0.250\pm 0.010$ \\
Energy     & 19,735     & 27 & 
$\mathbf{0.712\pm 0.005}$ & $0.786\pm 0.008$ & $0.738\pm 0.006$ & $0.802\pm 0.007$ & $0.843\pm 0.006$ & $0.795\pm 0.005$ \\
PM2.5      & 41,757     & 15 & 
$0.573\pm 0.003$ & $\mathbf{0.286\pm 0.002}$ & $0.385\pm 0.004$ & $0.287\pm 0.003$ & $0.638\pm 0.005$ & $0.540\pm 0.010$ \\ 
Protein    & 45,730     & 9  & 
$0.829\pm 0.002$ & $0.659\pm 0.004$ & $\mathbf{0.597\pm 0.001}$ & $0.696\pm 0.004$ & $0.715\pm 0.003$ & $0.676\pm 0.004$ \\ 
Query      & 100,000    & 4  & 
$0.128\pm 0.000$ & $0.027\pm 0.000$ & $\mathbf{0.024\pm 0.000}$ & \textendash\textendash & $0.058\pm 0.002$ & $0.061\pm 0.000$ \\
Borehole   & 1,000,000  & 8  & 
$0.173\pm 0.000$ & $\mathbf{0.172\pm 0.000}$ & $0.174\pm 0.000$ & \textendash\textendash & $0.176\pm 0.000$ & $0.173\pm 0.000$ \\    
\bottomrule
\end{tabular}
\end{adjustbox}

\bigskip\bigskip

\caption{We summarize the accuracy of learned noise variance of sgGP and other GPs on simulated datasets. Results follow the experiments in Table \ref{table:rmse}.}
\label{table:learned_noise_var}
\begin{adjustbox}{max width=\textwidth}
\begin{tabular}[c]{c c c c c c c c c} 
\toprule
 & & & \multicolumn{6}{c}{\textbf{Learned Noise Variance/True Noise Variance} ($\bm{\hat{\sigma}_n^2}/\bm{\sigma_n^2}$)} \\
\cmidrule(lr){4-9}
\textbf{Dataset} & Size & $D$ & 
\textbf{sgGP (uni)} & \textbf{sgGP (nn)} & \textbf{Vecchia} & \textbf{EGP} & \textbf{SGPR} & \textbf{SVGP} \\ 
\midrule
Levy       & 10,000     & 4  &
$9.95\pm 0.16$ & $1.32\pm 0.05$ & $0.12\pm 0.01$ & $0.14\pm 0.01$ & $2.32\pm 0.12$ & $11.05\pm 0.45$ \\
Griewank   & 10,000     & 6  &
$0.03\pm 0.00$ & $1.38\pm 0.03$ & $0.02\pm 0.01$ & $0.48\pm 0.02$ & $9.18\pm 0.49$ & $22.77\pm 6.74$ \\
Borehole   & 1,000,000  & 8  &
$0.96\pm 0.02$ & $0.99\pm 0.02$ & $0.00\pm 0.00$ & \textendash\textendash & $0.86\pm 0.08$ & $1.97\pm 0.21$ \\    
\bottomrule
\end{tabular}
\end{adjustbox}

\bigskip\bigskip

\caption{We summarize the training time of sgGP and other GPs on benchmark dataset. Results follow the experiments in Table \ref{table:rmse}.}
\label{table:train_time}
\begin{adjustbox}{max width=\textwidth}
\begin{tabular}[c]{c c c c c c c c c} 
\toprule
 & & & \multicolumn{6}{c}{\textbf{Training Time (min)}} \\
\cmidrule(lr){4-9}
\textbf{Dataset} & Size & $D$ & 
\textbf{sgGP (uni)} & \textbf{sgGP (nn)} & \textbf{Vecchia} & \textbf{EGP} & \textbf{SGPR} & \textbf{SVGP} \\ 
\midrule
Levy       & 10,000     & 4  & 
$0.41\pm 0.02$ & $0.35\pm 0.03$ & $3.15\pm 0.12$ & $11.12\pm 0.71$ & $3.55\pm 0.22$ & $14.74\pm 0.69$ \\
Griewank   & 10,000     & 6  & 
$0.54\pm 0.02 $ & $0.48\pm 0.03$ & $4.24\pm 0.09$ & $13.37\pm 1.18$ & $1.76\pm 0.12$ & $14.60\pm 0.65$ \\
Bike       & 17,379     & 17 & 
$1.50\pm 0.12$  & $1.55\pm 0.10$ & $6.70\pm 0.47$ & $29.48\pm 3.96$ & $5.31\pm 2.05$ & $25.26\pm 3.97$ \\
Energy     & 19,735     & 27 & 
$3.03\pm 0.02$ & $2.58\pm 0.17$ & $10.44\pm 1.29$ & $53.25\pm 2.47$ & $5.41\pm 0.73$ & $25.09\pm 5.50$ \\
PM2.5      & 41,757     & 15 & 
$4.90\pm 0.33$ & $4.24\pm 0.29$ & $13.69\pm 0.70$ & $372.88\pm 16.78$ & $13.59\pm 2.30$ & $52.46\pm 10.08$ \\ 
Protein    & 45,730     & 9  & 
$3.12\pm 0.01$ & $2.63\pm 0.17$ & $13.06\pm 0.12$ & $453.40\pm 21.31$ & $19.55\pm 1.66$ & $55.27\pm 13.09$ \\ 
Query      & 100,000    & 4  & 
$4.73\pm 0.36$ & $5.03\pm 0.36$ & $30.86\pm 1.69$ & \textendash\textendash & $20.73\pm 1.63$ & $124.73\pm 22.25$  \\
Borehole   & 1,000,000  & 8  & 
$54.82\pm 2.69$ & $65.19\pm 2.98$ & $235.74\pm 18.00$ & \textendash\textendash & $857.60\pm 76.02$ & $1380.86\pm 11.32$ \\    
\bottomrule
\end{tabular}
\end{adjustbox}

\end{table}

\begin{table}[htbp!]
\caption{We summarize the results of sgGP (nn) on simulated application-driven datasets. We follow similar setups of the experiments in Table \ref{table:rmse}.}
\centering
\begin{adjustbox}{max width=\textwidth}
\begin{tabular}{cccccc} 
 \hline\noalign{\smallskip}
 \textbf{Dataset} & Size & $D$ & \textbf{RMSE} & \textbf{Training Time} (min) & \textbf{Memory Usage} (GB)\\ 
 \hline\noalign{\smallskip}
 OTL Circuit & 2,000,000 & 6  & $0.401\pm 0.000$ & $33.43\pm 4.40$ & $0.99\pm 0.00$\\
 Wing Weight & 2,000,000 & 10 & $0.072\pm 0.004$ & $78.78\pm 9.26$ & $1.22\pm 0.00$\\
 \hline
\end{tabular}
\end{adjustbox}
\label{table:toy}
\end{table}

\newpage

\section{Open Problems}\label{sec:open_problem}
There still exist some open problems that are worth future investigations.
\begin{enumerate}
    \item The extension to convergence guarantees for learning the lengthscale parameter in RBF kernel is an interesting but extremely challenging problem: our case studies suggest that SGD may still be used for estimating the lengthscale in practice, but the proof for both Lemma \ref{lem:g_k_conv_property_abbrv} and Lemma \ref{lem:statistical_err_full_grad_abbrv} presents additional challenges if looking at the lengthscale. This is due to it being wrapped within the exponential term as a denominator, which translates to different eigenvectors for $\frac{\partial \vect{K}_{\xi}}{\partial l}$, $\vect{K}_{\xi}$ and $\vect{K}_{\xi}^*$. To see the difficulty for proving Lemma \ref{lem:g_k_conv_property_abbrv}, note that the curvature term for estimating $l$ involves 
    \begin{equation}\label{eq:curv_lengthscale}
        \tr\left(\vect{K}_{\xi}(\vect{\theta}^*)\vect{K}_{\xi}(\vect{\theta})^{-1}\frac{\partial \vect{K}_{\xi}(\vect{\theta})}{\partial l}\vect{K}_{\xi}(\vect{\theta})^{-1}\frac{\partial \vect{K}_{\xi}(\vect{\theta})}{\partial l}\vect{K}_{\xi}(\vect{\theta})^{-1}\right)
    \end{equation} and cannot be expressed as a function of the eigenvalues of kernel matrices due to their different eigenvectors. It is also very hard to upper bound the statistical error in Lemma \ref{lem:statistical_err_full_grad_abbrv} due to similar reasons.
    \item It would also be interesting to extend theoretical guarantees from $M=1$ to $M>1$ without Assumption \ref{assump:same_eigenvec}. The technical challenge for this part is similar to the previous point: without Assumption \ref{assump:same_eigenvec}, $K_n$ and $K_n^*$ can not be simultaneously diagonalized, which hinders the proofs of Lemma \ref{lem:g_k_conv_property_abbrv} and Lemma \ref{lem:statistical_err_full_grad_abbrv}.
    \item Another open problem is to establish convergence guarantees for running SGD with nearby sampling and to explore different techniques for nearby sampling upon the choice of the kernel. 
\end{enumerate}

\section{Conclusion}\label{sec:conclusion}
In this paper, we provide theoretical guarantees for the minibatch SGD for the model selection of Gaussian process (GP).
In particular, we prove that the iterates of SGD converge to the true hyperparameters and the critical point of the full loss function, with rate $O(\frac{1}{K})$ up to a statistical error term depending on minibatch size.
Given the correlation structure of GPs, the challenge lies in the bias of stochastic gradient when taking expectation w.r.t. random sampling. Numerical studies support our theoretical results and show that minibatch SGD has better performance than state-of-the-art methods on various datasets while enjoying huge computational benefits. 

\newpage
\appendix
\section{Table of Notations}\label{append:notations}
\begin{table}[htb!]
    \centering
    \begin{tabular}{c|c}
    \hline
        Notations & Description\\
        \hline
        $n$ & number of data points in the full data set\\
        \hline
        $m$ & number of data points in a minibatch\\
        \hline
        $K$ & number of iterations of minibatch SGD\\
        \hline
        $G$ & An upper bound for $\|g(\vect{\theta}^{(k)})\|_2$, specified in Assumption~\ref{assump:sg_bnd}\\
        \hline
        $\sigma_f^2=\theta^*_1$ & true signal variance parameter\\
        \hline
        $\sigma_{\epsilon}^2=\theta^*_2$ & true noise variance parameter\\
        \hline
        $\vect{\theta}^{(k)}$ & output of minibatch SGD at the $k$th iteration, as an estimate of $\vect{\theta}^*$\\
        \hline
        $\vect{K}_n(\vect{\theta})=\theta_1\vect{K}_{f,n}+\theta_2\vect{I}_n$ & covariance of $\vect{y}_n$ given $\vect{X}_n$, if the hyperparameter is $\vect{\theta}$\\
        \hline
        $\vect{K}_{\xi}(\vect{\theta})$ & submatrix of $\vect{K}_n(\vect{\theta})$ with rows and columns both indexed by $\xi$\\
        \hline
        $\vect{K}_{f,n}$ & kernel matrix evaluated at $\vect{X}_n$\\
        \hline
        $\nabla \ell(\vect{\theta};\vect{X}_n,\vect{y}_n)$ or $\nabla \ell(\vect{\theta})$ & full gradient evaluated at $\vect{\theta}$ and full data $\vect{X}_n$, $\vect{y}_n$\\
        \hline
        $g(\vect{\theta};\vect{X}_{\xi},\vect{y}_{\xi})$ & stochastic gradient evaluated at $\vect{\theta}$ and minibatch $\vect{X}_{\xi}$, $\vect{y}_{\xi}$\\
        \hline
        $\alpha_k=\frac{\alpha_1}{k}$ & step size at the $k$th iteration\\
        \hline
        \multirow{2}{*}{$g^*(\vect{\theta}^{(k)};\vect{X}_{\xi_{k+1}})$ or $g^*(\vect{\theta}^{(k)})$}& conditional expectation of $g(\vect{\theta}^{(k)};\vect{X}_{\xi_{k+1}},\vect{y}_{\xi_{k+1}})$\\ &at the $k$th iteration given $\vect{X}_{\xi_{k+1}}$\\ 
        \hline
        $\lambda_j^{(k)}$ & the $j$th largest eigenvalue of $\vect{K}_{f,\xi_{k+1}}$\\
        \hline
        $\lambda_j^*$ & the $j$th largest eigenvalue of $\vect{K}_{f,n}$\\
        \hline
    \end{tabular}
    \caption{Important notations used throughout the paper}
    \label{tab:notations}
\end{table}
\section{Theoretical Guarantees for Section~\ref{sec:sum-kernels}}\label{append:sum-kernels}
Before presenting the theoretical guarantees under this setting, we first provide a formal definition for the considered minibatch SGD algorithm. With sampled indices $\xi\subset[n]$, let the stochastic gradient $g(\vect{\theta};\mathbf{X}_{\xi}, \mathbf{y}_{\xi})\in \bbR^{M+1}$ be defined as follows: 
\begin{equation}\label{eq:stochastic-gradient-sum-of-kernels}
    \left(g(\vect{\theta};\mathbf{X}_{\xi}, \mathbf{y}_{\xi})\right)_l =  \frac{1}{2s_l(m)}\text{tr}\left[(\mathbf{K}_{\xi}^{-1}(\mathbf{I}_m-\mathbf{y}_{\xi}\mathbf{y}_{\xi}^\top\mathbf{K}_{\xi}^{-1})\frac{\partial \mathbf{K}_{\xi}}{\partial \theta_l}\right],\quad 1\leq l\leq M+1,
\end{equation}
where $\mathbf{K}_{\xi}$ is the principle submatrix formed by the rows and columns of $\mathbf{K}_n$ indexed by $\xi$. In the following we will also let $\vect{K}_{f,\xi}^{(l)}$ denote the $m\times m$ block of $\vect{K}_{f,n}^{(l)}$ indexed by $\xi$. 
Algorithm \ref{alg:minibatchSGD_sum-of-kernels} summarizes the steps of minibatch SGD. 
\begin{algorithm}[t]
\SetAlgoLined
 Input: $\vect{\theta}^{(0)}\in \bbR^{M+1}$, initial step size $\alpha_1>0$.\\
 \For{$k=1,2,\dotsc,K$}{
  Randomly sample a subset of indices $\xi_k$ of size $m$\; 
  Compute the stochastic gradient $g(\vect{\theta}^{(k-1)};\mathbf{X}_{\xi_k}, \mathbf{y}_{\xi_k})\in \bbR^{M+1}$\;
  $\alpha_k\leftarrow \frac{\alpha_1}{k}$\;
  $\vect{\theta}^{(k)}\leftarrow \vect{\theta}^{(k-1)} - \alpha_kg(\vect{\theta}^{(k-1)};\mathbf{X}_{\xi_k}, \mathbf{y}_{\xi_k})$\;
 }
 \caption{Minibatch SGD with uniform sampling when the covariance function is the sum of multiple kernels}
 \label{alg:minibatchSGD_sum-of-kernels}
\end{algorithm}

In the following, we present convergence guarantees for Algorithm~\ref{alg:minibatchSGD_sum-of-kernels} when kernels exhibit exponential or polynomial eigendecay. The assumptions are similar to the ones presented in Section~\ref{sec:thm}.
\begin{assumption}[Bounded iterates]\label{assump:param_bnd_sum_of_kernels}
Both $\vect{\theta}^*$ and $\vect{\theta}^{(k)}$ for $0\leq k\leq K$ lie in $[\theta_{\min},\theta_{\max}]^{M+1}$, where $0<\theta_{\min}<\theta_{\max}$. 
\end{assumption}
\begin{assumption}\label{assump:same_eigenvec}
For any $n>0$ and sample $\{\vect{x}_i\}_{i=1}^n$, the kernel matrices $\vect{K}_{f,n}^{(1)},\dots, \vect{K}_{f,n}^{(M)}$ share the same eigenvectors.
\end{assumption}
\begin{remark}[Explanation for Assumption \ref{assump:same_eigenvec}]
When extending the theoretical guarantees from $M=1$ to $M>1$, we find it extremely challenging without Assumption \ref{assump:same_eigenvec}, which ensures that matrix $\vect{K}_n(\vect{\theta})$ and $\vect{K}_n^*$ are simultaneously diagonalizable, and thus facilitates the analysis for the gradient. 
It remains an open question to establish theoretical results without this assumption. We believe that if the eigenvectors of the kernel matrices are ``close'' our results should still hold.
\end{remark}
\begin{assumption}[Bounded stochastic gradient]\label{assump:sg_bnd_sum_of_kernels}
For all $0\leq k<K$, \begin{equation*}
    \|g(\vect{\theta}^{(k)};\vect{X}_{\xi_{k+1}},\vect{y}_{\xi_{k+1}})\|_2\leq G
\end{equation*} for some $G>0$.
\end{assumption}
\subsection{Kernels with Exponential Eigendecay}\label{sec:theory_exponential_sum_of_kernels}
\begin{assumption}[Exponential eigendecay]\label{assump:eigen_decay_exp_sum_of_kernels}
For $1\leq i\leq M$, the eigenvalues of kernel function $k_i$ w.r.t. probability measure $\bbP$ are $\{C_ie^{-b_ij}\}_{j=0}^{\infty}$, where $0<b_1<b_2<\cdots<b_{M}$, and $C_i\leq 1$ are regarded as constants.
\end{assumption}
\begin{thm}[Convergence of parameter iterates, exponential eigendecay]\label{thm:large_m_param_converg_sum_of_kernels}
 Under Assumptions \ref{assump:param_bnd_sum_of_kernels} to \ref{assump:eigen_decay_exp_sum_of_kernels}, when $m>C$ for some constant $C>0$, we have the following results under two corresponding conditions on $s_l(m)$:
 \begin{enumerate}[leftmargin=*]\vspace{-2mm}
 \item If $s_{M+1}(m)=m$, initial step size $\alpha_1$ satisfies $\frac{3}{2\gamma}\leq \alpha_1\leq \frac{2}{\gamma}$ where $\gamma=\frac{1}{4\theta_{\max}^2}$, then for any $0<\varepsilon<C\frac{\log\log m}{\log m}$, with probability at least $1-CK\exp\{-cm^{2\varepsilon}\}$,
\begin{equation}\label{eq:noise_var_convg_bnd_full}
    (\theta^{(K)}_{M+1}-\theta^*_{M+1})^2\leq \frac{8G^2}{\gamma^2(K+1)}+Cm^{-\frac{1}{2}+\varepsilon}.
\end{equation}
\item If in addition to $s_{M+1}(m)=m$, $s_l(m)$ is set as $\tau\log m$ for $1\leq l\leq M$ where $\tau>C$, the eigendecay rates $b_2>2b_1$ when $M\geq 2$, $\frac{3}{2\gamma}\leq \alpha_1\leq \frac{2}{\gamma}$ where $\gamma$ depends on $\tau$, 
    then for any $0<\varepsilon<\frac{1}{2}$, with probability at least $1-CK\exp\{-c(\log m)^{2\varepsilon}\}$,
\begin{equation}
    (\theta^{(K)}_1-\theta^*_1)^2+(\theta^{(K)}_{M+1}-\theta^*_{M+1})^2\leq \frac{8G^2}{\gamma^2(K+1)}+C(\log m)^{-\frac{1}{2}+\varepsilon}.
\end{equation}
 \end{enumerate}
 Here $c,C>0$ depend only on $M,\theta_{\min},\theta_{\max},b_1,\dots,b_M$.
\end{thm}
\begin{rmk}
Here we do not provide estimation error bounds for $\theta^*_l$ or $\sigma_{f,l}^2$, for $2\leq l\leq M$, since they are associated with kernels with faster eigendecay than $k_1$ and thus are not identifiable. The technical condition on decay rates $b_2>2b_1$ is to ensure the convergence of $\theta^{(K)}_1$ although $|\theta^{(K)}_l-\theta^*_l|$ does not converge to $0$ for $2\leq l\leq M$.
\end{rmk}
\begin{rmk}
For the second case where $s_{l}(m)=\tau\log m$, $\tau$ and $\gamma$ need to satisfy
\begin{equation}\label{eq:tau_lowerbnd}
    \tau>\frac{8b_2(b_2-b_1)(M+1)^2\theta_{\max}^4}{3b_1^2(b_2-2b_1)\theta_{\min}^{4}},
\end{equation}
\begin{equation}\label{eq:gamma_log_m_scaling}
        \gamma=\min\left\{\frac{3(b_2-2b_1)}{8 \tau b_2(b_2-b_1)(M+1)^2\theta_{\max}^2},\frac{1}{4\theta_{\max}^2}-\frac{2b_2(b_2-b_1)(M+1)^2\theta_{\max}^2}{3\tau b_1^2(b_2-2b_1)\theta_{\min}^{4}}\right\}.
    \end{equation}
\end{rmk}

\begin{thm}[Convergence of full gradient, exponential eigendecay]\label{thm:large_m_grad_converg_sum_of_kernels}
Under Assumptions \ref{assump:param_bnd_sum_of_kernels} to \ref{assump:eigen_decay_exp_sum_of_kernels}, if $\frac{3}{2\gamma}\leq \alpha_1\leq \frac{2}{\gamma}$ for $\gamma=\frac{1}{4\theta_{\max}^2}$,  $m>C$, $s_{M+1}(m)=m$, then for any $0<\varepsilon<C\frac{\log\log m}{\log m}$, with probability at least $1-CK\exp\{-cm^{2\varepsilon}\}$,
\begin{equation}
    \|\nabla \ell(\vect{\theta}^{(K)})\|_2^2\leq C\left[\frac{G^2}{K+1}+m^{-\frac{1}{2}+\varepsilon}\right],
\end{equation}
holds, where $c,C>0$ depend only on $M,\theta_{\min},\theta_{\max},b_1,\dots,b_M$.
\end{thm}

\subsection{Kernels with Polynomial Eigendecay}\label{sec:theory_polynomial_sum_of_kernels}
\begin{assumption}[Polynomial eigendecay]\label{assump:eigen_decay_poly_sum_of_kernels}
For $1\leq l\leq M$, the eigenvalues of kernel function $k_l$ w.r.t. probability measure $\bbP$ are $\{C_lj^{-2b_l}\}_{j=0}^{\infty}$, where $\frac{\sqrt{21}+3}{4}<b_1<b_2<\cdots<b_{M}$, and $C_l\leq 1$ are regarded as constants.
\end{assumption}
   \begin{thm}[Convergence of parameter iterates, polynomial eigendecay ]\label{thm:large_m_param_converg_poly_sum_of_kernels}
 Under Assumptions \ref{assump:param_bnd_sum_of_kernels} to \ref{assump:sg_bnd_sum_of_kernels} and Assumption \ref{assump:eigen_decay_poly_sum_of_kernels}, when $m>C$ for some constant $C>0$, $s_{M+1}(m)=m$, $\frac{3}{2\gamma}\leq \alpha_1\leq \frac{2}{\gamma}$ where $\gamma=\frac{1}{8\theta_{\max}^2}$, then for any $\varepsilon\in(\max\{0,f_1(b_1)\},\frac{1}{2})$, with probability at least $1-CKm^{-f_2(b_1)\left[\varepsilon-f_1(b_1)\right]}-CK\exp\{-cm^{2\varepsilon}\}$,
      \begin{equation}\label{eq:noise_var_converg_bnd_poly_full}
          (\theta^{(K)}_{M+1}-\theta^*_{M+1})^2\leq \frac{8G^2}{\gamma^2(K+1)}+Cm^{-\frac{1}{2}+\varepsilon}.
      \end{equation}
  Here $c,C>0$ depend only on $M,\theta_{\min},\theta_{\max},b_1,\dots,b_{M}$, and $f_1(\cdot)$, $f_2(\cdot)$ are defined as in Theorem~\ref{thm:large_m_param_converg_poly}.
\end{thm}

\begin{thm}[Convergence of full gradient, polynomial eigendecay]\label{thm:large_m_grad_converg_poly_sum_of_kernels}
Under the same conditions as Theorem \ref{thm:large_m_grad_converg_poly_sum_of_kernels}, for any $\varepsilon\in(\max\{0,f_1(b_1)\},\frac{1}{2})$, with probability at least $1-CK\left(m^{-f_2(b_1)\left[\varepsilon-f_1(b_1)\right]}+\exp\{-cm^{2\varepsilon}\}\right)$,
\begin{equation}\label{eq:full_grad_converg_bnd_poly_sum_of_kernels}
    \|\nabla \ell(\vect{\theta}^{(K)})\|_2^2\leq C\left[\frac{G^2}{K+1}+m^{-\frac{1}{2}+\varepsilon}\right],
\end{equation}
holds, where $c,C>0$ depend only on $M,\theta_{\min},\theta_{\max},b_1,\dots,b_M$, $f_1(\cdot)$ and $f_2(\cdot)$ are defined as in Theorem \ref{thm:large_m_param_converg_poly}.
\end{thm}

\section{Detailed Versions of Key Lemmas}
First we present the detailed versions of the key lemmas (Lemma \ref{lem:g_k_conv_property_abbrv}, \ref{lem:statistical_err_full_grad_abbrv} and \ref{lem:eigen_ratio_bnds_abbrv}) discussed in Section \ref{sec:proofoverview}, which are useful for proving the second part of Theorem \ref{thm:large_m_param_converg}, Theorem \ref{thm:large_m_param_converg_poly} to Theorem \ref{thm:large_m_grad_converg_poly_sum_of_kernels}, and other supporting lemmas. We will focus on the general case where the covariance function is a linear combination of multiple kernels, introduced in Section~\ref{sec:sum-kernels}. The model introduced in Section~\ref{sec:problem_setup} can be viewed a special case of this general model, with $M=1$. Since we are considering the general setting in our proofs, the notations are consistent with the ones introduced in Section~\ref{sec:sum-kernels}: when $M=1$, the only kernel function is referred to as $k_1(\cdot,\cdot)$ instead of $k_0(\cdot,\cdot)$, we use $b_1$ to denote its eigendecay rate instead of $b$. 
\begin{lemma}[Strongly convex-like property of $g^*(\vect{\theta}^{(k)})$, exponential eigendecay]\label{lem:g_k_conv_property}
  Under Assumptions \ref{assump:param_bnd} to \ref{assump:eigen_decay_exp}, 
  \begin{enumerate}
      \item if $s_{M+1}(m)=m$, $m>C$ for some $C>0$, then with probability at least $1-3MKm^{-c}$, the following claim holds true for $0\leq k< K$:
\begin{equation}\label{eq:strong_conv}
           \langle \widetilde{\vect{\theta}}^{(k)}-\widetilde{\vect{\theta}}^*,\widetilde{g}_k^*\rangle
           \geq \frac{\gamma}{2}\|\widetilde{\vect{\theta}}^{(k)}-\widetilde{\vect{\theta}}^*\|_2^2-\varepsilon,
           \end{equation}
           where $\widetilde{\vect{\theta}}^{(k)}=\theta^{(k)}_{M+1}$, $\widetilde{\vect{\theta}}^*=\theta^*_{M+1}$, $\widetilde{g}_k^*=(g^*(\vect{\theta}^{(k)}))_{M+1}$, $\gamma=\frac{1}{4\theta_{\max}^2}$, $\varepsilon=\frac{C\log m}{m}$;
       \item if $M\geq 2$, in addition to $s_{M+1}(m)=m$, we also have $s_i(m)=\tau\log m$ for $1\leq i\leq M$, and $\tau$ satisfies \eqref{eq:tau_lowerbnd}, $b_2>2b_1$, then for any $0<\alpha<\min\{\frac{2b_1+b_2}{2b_1},\frac{2b_2-4b_1}{14b_1+b_2}\}$, with probability at least $1-3MKm^{-\alpha}$, \eqref{eq:strong_conv} holds for $\widetilde{\vect{\theta}}^{(k)}=(\theta^{(k)}_1,\theta^{(k)}_{M+1})$, $\widetilde{\vect{\theta}}^*=(\theta^*_1,\theta^*_{M+1})$, $\widetilde{g}_k^*=((g^*(\vect{\theta}^{(k)}))_1,(g^*(\vect{\theta}^{(k)}))_{M+1})^\top$,
    \begin{equation}
               \gamma=\min\left\{\frac{3(b_2-2b_1)}{8 \tau b_2(b_2-b_1)(M+1)^2\theta_{\max}^2},\frac{1}{4\theta_{\max}^2}-\frac{2b_2(b_2-b_1)(M+1)^2\theta_{\max}^2}{3\tau b_1^2(b_2-2b_1)\theta_{\min}^{4}}\right\};
    \end{equation}
    and $\varepsilon=C(\alpha+(\log m)^{-1})$;
    \item if $M=1$, in addition to $s_{M+1}(m)=m$, we also have $s_1(m)=\tau\log m$ where $\tau>\frac{64\theta_{\max}^4}{b_1\theta_{\min}^4}$, then with probability at least $1-2Km^{-c}$, \eqref{eq:strong_conv} holds for $\widetilde{\vect{\theta}}^{(k)}=\vect{\theta}^{(k)}$, $\widetilde{\vect{\theta}}^*=\vect{\theta}^*$, $\widetilde{g}_k^*=g^*(\vect{\theta})$,
    \begin{equation}\label{eq:gamma_log_m_scaling}
               \gamma=\min\left\{\frac{1}{32\tau b_1\theta_{\max}^2},\frac{1}{4\theta_{\max}^2}-\frac{2\theta_{\max}^2}{\tau b_1\theta_{\min}^{4}}\right\}.
    \end{equation}
    and $\varepsilon=C\frac{\log m}{m}$.
  \end{enumerate}
   Here $C>0$ depends only on $M,\theta_{\min},\theta_{\max},b_1,\dots,b_{M}$.
\end{lemma}

\begin{lemma}[Strongly convex-like property of $g^*(\vect{\theta}^{(k)})$, polynomial eigendecay]\label{lem:g_k_conv_property_poly}
   If\\
   $s_{M+1}(m)=m$, $m>C$ for some $C>0$, then for any $0<\alpha<\frac{8b_1^2-12b_1-6}{4b_1+3}$, with probability at least $1-MKm^{-\alpha}$, the following claim holds true for $0\leq k< K$:
\begin{equation}\label{eq:strong_conv_poly}
           (g^*(\vect{\theta}^{(k)}))_{M+1}(\theta^{(k)}_{M+1}-\theta^*_{M+1})\geq\frac{\gamma}{2} (\theta^{(k)}_{M+1}-\theta_{M+1}^*)^2-\epsilon,
           \end{equation}
           where $\gamma=\frac{1}{8\theta_{\max}^2}$, $\epsilon=Cm^{-\frac{8b_1^2-12b_1-6-\alpha(4b_1+3)}{4b_1(2b_1-1)}}$.
   Here $C>0$ depends only on $M,\theta_{\min},\theta_{\max},b_1,\dots,b_{M}$.
   \end{lemma}
  
  Lemma \ref{lem:g_k_conv_property} and Lemma \ref{lem:g_k_conv_property_poly} are detailed versions of Lemma \ref{lem:g_k_conv_property_abbrv}.
  
\begin{lemma}[Uniform statistical error]\label{lem:statistical_err_full_grad}For any $x>0$, $1\leq i\leq M+1$, we have
\begin{equation}\label{eq:full_grad_exp_bnd}
    \bbP\left(\sup_{\vect{\theta}\in [\theta_{\min},\theta_{\max}]^{M+1}}\frac{n}{s_i(n)}\left|(\nabla \ell(\vect{\theta}))_i-(\nabla \ell^*(\vect{\theta}))_i\right|> Cx\right)\leq \delta(x).
\end{equation}
If Assumption \ref{assump:eigen_decay_exp} holds, $s_i(n)=\tau\log n$ for $\tau$ satisfying \eqref{eq:tau_lowerbnd}, $n>C$ for some $C>0$, then
\begin{equation*}
    \delta(x)\leq Cn^{-c}+C(\log x)^{2M+2}\exp\{-c\log n\min\{x^2,x\}\}.
\end{equation*}
If Assumption \ref{assump:eigen_decay_exp} or \ref{assump:eigen_decay_poly} hold, $s_i(n)=n$, \begin{equation*}\delta(x)\leq C(\log x)^{2M+2}\exp\{-cn\min\{x^2,x\}\}.\end{equation*}
 Here $c,C>0$ only depend on $M,\theta_{\min},\theta_{\max},b_1,\dots,b_{M}$. 
\end{lemma}

Lemma \ref{lem:statistical_err_full_grad} is a detailed version of Lemma \ref{lem:statistical_err_full_grad_abbrv}, including results for kernels with exponential or polynomial eigendecay.

\begin{lemma}[Detailed version of Lemma \ref{lem:eigen_ratio_bnds_abbrv}, exponential eigendecay]\label{lem:eigen_ratio_bnds}
Under Assumption \ref{assump:eigen_decay_exp}, for any $\alpha>0$, if $n>C$ for $C>0$ depending on $M,b_1,\dots,b_{M}$, then with probability at least $1-3Mn^{-\alpha}$,
\begin{equation}\label{eq:eigen_ratio_bnds1}
\begin{split}
    \text{ if }l\wedge l'\leq M, \sum_{j=1}^n \frac{\lambda_{lj}\lambda_{l'j}}{\left(\sum_{h=1}^{M+1} \theta_h\lambda_{hj}\right)^2}\leq &\frac{2(2+\alpha)}{b_1\theta_{\min}^2}\log n,\\
    \frac{n-C(\alpha)\log n}{4\theta_{\max}^2}\leq \sum_{j=1}^n \frac{\lambda_{M+1,j}^2}{\left(\sum_{h=1}^{M+1} \theta_h\lambda_{hj}\right)^2}\leq &\frac{n}{\theta_{\min}^2},\\
    \sum_{j=1}^n \frac{\lambda_{1j}\lambda_{M+1,j}}{\left(\sum_{h=1}^{M+1} \theta_h\lambda_{hj}\right)^2}\leq &\frac{5+2\alpha}{7b_1\theta_{\min}^2}\log n,
\end{split}
\end{equation}
holds for any $\vect{\theta}\in[\theta_{\min},\theta_{\max}]^{M+1}$, where $C(\alpha)>0$ depends only on $\alpha,b_1$.
Furthermore,
\begin{itemize}
    \item if $M=1$, then for any $0<\alpha,\epsilon<1$, with probability at least $1-2n^{-\alpha}$, in addition to \eqref{eq:eigen_ratio_bnds1} we have
\begin{equation}\label{eq:eigen_ratio_bnds3}
     \sum_{j=1}^n \frac{\lambda_{1j}^2}{\left(\sum_{h=1}^{2} \theta_h\lambda_{hj}\right)^2}\geq \frac{\epsilon \log n}{8b\theta_{\max}^2};
     \end{equation}
    \item if $b_2>2b_1$ holds, then for any $\frac{2b_1}{b_2}<\epsilon<1$, $0<\alpha< \min\left\{\epsilon,\frac{2\epsilon b_2-4b_1}{6b_1+\epsilon b_2}\right\}$, with probability at least $1-3Mn^{-\alpha}$, in addition to \eqref{eq:eigen_ratio_bnds1} we have
\begin{equation}\label{eq:eigen_ratio_bnds2}
\begin{split}
    \sum_{j=1}^n \frac{\lambda_{1j}^2}{\left(\sum_{h=1}^{M+1} \theta_h\lambda_{hj}\right)^2}\geq& \frac{\epsilon(b_2-2b_1)}{2b_1(b_2-b_1)(M+1)^2\theta_{\max}^2}\log n,\\
    \text{for $1<i\leq M$, $l\in S_i$, }\sum_{j=1}^n \frac{\lambda_{lj}\lambda_{ij}}{\left(\sum_{h=1}^{M+1} \theta_h\lambda_{hj}\right)^2}\leq& \frac{(6b_1+b_2)\alpha\log n}{2b_1(4b_l-b_2-6b_1)\theta_{\min}^2}+\frac{C(\epsilon)}{\theta_{\min}^2};
\end{split}
\end{equation}
\end{itemize}
as long as $n>C(\epsilon)$, where $C(\epsilon)>0$ depends on $M, b_1,\dots,b_{M}$ and $\epsilon$.
Here $S_i=\{1,i,i+1,\dots,M+1\}$.
\end{lemma}

\begin{lemma}[Detailed version of Lemma \ref{lem:eigen_ratio_bnds_abbrv}, polynomial eigendecay]\label{lem:eigen_ratio_bnds_poly}
Under Assumption \ref{assump:eigen_decay_poly}, for any $0<\alpha<\frac{8b_1^2-12b_1-6}{4b_1+3}$, with probability at least $1-Mn^{-\alpha}$, for any $\vect{\theta}\in[\theta_{\min},\theta_{\max}]^M$,
\begin{equation}\label{eq:eigen_ratio_bnds_poly}
\begin{split}
    \text{ if }l\leq l'\leq M, &\sum_{j=1}^n \frac{\lambda_{lj}\lambda_{l'j}}{\left(\sum_{h=1}^{M+1} \theta_h\lambda_{hj}\right)^2}\leq n^{\frac{(2+\alpha)(4b_l+3)}{4b_l(2b_l-1)}}\left(\frac{1}{\theta_{\min}^2}+\frac{a_l(4b_l+3)}{\theta_{\min}^2(4b_l^2-6b_l-3)}\right),\\
    &\frac{n-M\max_l a_l n^{\frac{(2+\alpha)(4b_1+3)}{4b_1(2b_1-1)}}}{4\theta_{\max}^2}\leq \sum_{j=1}^n \frac{\lambda_{M+1,j}^2}{\left(\sum_{h=1}^{M+1} \theta_h\lambda_{hj}\right)^2}\leq \frac{n}{\theta_{\min}^2},
\end{split}
\end{equation}
where $a_l=2\sqrt{2C_l}+\sqrt{\frac{2C_l}{2b_l-1}}+\frac{C_l}{2b_l-1}$.
\end{lemma}
\section{Proofs of Theorem \ref{thm:large_m_param_converg}, \ref{thm:large_m_param_converg_poly} and \ref{thm:large_m_grad_converg_poly}}\label{sec:proof_thm}
\begin{proof}[Proof of Theorem \ref{thm:large_m_param_converg}]
First we apply Lemma \ref{lem:g_k_conv_property} under both cases of $s_i(m)$: for the first case ($s_{M+1}(m)=m$) discussed in Lemma \ref{lem:g_k_conv_property}, define $\widetilde{g}(\vect{\theta}^{(k)})=(g(\vect{\theta}^{(k)}))_{M+1}$, and for the second case ($s_i(m)=\tau\log m$ and $s_{M+1}(m)=m$), define $\widetilde{g}(\vect{\theta}^{(k)})=((g(\vect{\theta}^{(k)}))_1,(g(\vect{\theta}^{(k)})_{M+1})^\top$. Then let $\widehat{\vect{e}}_k= \widetilde{g}(\vect{\theta}^{(k)})-\widetilde{g}_k^*$. 
Due to Lemma \ref{lem:g_k_conv_property} and Assumption \ref{assump:sg_bnd}, we have
 \begin{equation}\label{eq:SGD_err_bnd_recursive2}
     \begin{split}
         \|\widetilde{\vect{\theta}}^{(k)}-\widetilde{\vect{\theta}}^*\|_2^2=&\|\widetilde{\vect{\theta}}^{(k-1)}-\widetilde{\vect{\theta}}^*\|_2^2-2\alpha_k\langle \widetilde{\vect{\theta}}^{(k-1)}-\widetilde{\vect{\theta}}^*,\widetilde{g}(\vect{\theta}^{(k-1)})\rangle+\alpha_k^2\|\widetilde{g}(\vect{\theta}^{(k-1)})\|_2^2\\
         \leq &\|\widetilde{\vect{\theta}}^{(k-1)}-\widetilde{\vect{\theta}}^*\|_2^2(1-\alpha_k\gamma)+\alpha_k^2G^2+2\alpha_k\left(\varepsilon-\langle \widetilde{\vect{\theta}}^{(k-1)}-\widetilde{\vect{\theta}}^*,\widehat{\vect{e}}_{k-1}\rangle\right).
     \end{split}
 \end{equation}
  Recall that $\frac{3}{2\gamma}\leq\alpha_1\leq \frac{2}{\gamma}$, and $\alpha_k=\frac{\alpha_1}{k}$ for all $k\geq 1$. Now we prove the following statement for $k\geq 1$ by induction:
  \begin{equation}\label{eq:SGD_err_bnd2}
  \|\widetilde{\vect{\theta}}^{(k)}-\widetilde{\vect{\theta}}^*\|_2^2\leq \frac{2\alpha_1^2G^2}{k+1}+\sum_{i=0}^{k-1}\eta_{k,i}\left(\varepsilon-\langle\widetilde{\vect{\theta}}^{(i)}-\widetilde{\vect{\theta}}^*,\widehat{\vect{e}}_i\rangle\right),
  \end{equation}
  where $\eta_{k,i}=2\alpha_{i+1}\prod_{j=i+2}^k(1-\alpha_j\gamma)$. When $k=1$, by \eqref{eq:SGD_err_bnd_recursive2} and the fact that $1-\alpha_1\gamma<0$,
  \begin{equation}
      \|\widetilde{\vect{\theta}}^{(1)}-\widetilde{\vect{\theta}}^*\|_2^2\leq \alpha_1^2G^2+\eta_{1,0}\left(\varepsilon-\langle \widetilde{\vect{\theta}}^{(0)}-\widetilde{\vect{\theta}}^*,\widehat{\vect{e}}_0\rangle\right).
  \end{equation}
  Assuming \eqref{eq:SGD_err_bnd2} holds for $k=l\geq 1$, then due to \eqref{eq:SGD_err_bnd_recursive2} and the fact that $1-\alpha_{l+1}\gamma\geq0$ for $l\geq 1$, we have
  \begin{equation}
      \begin{split}
          &\|\widetilde{\vect{\theta}}^{(l+1)}-\widetilde{\vect{\theta}}^*\|_2^2\\
          \leq &\left(\frac{2\alpha_1^2G^2}{l+1}+\sum_{i=0}^{l-1}\eta_{l,i}\left(\varepsilon-\langle\widetilde{\vect{\theta}}^{(i)}-\widetilde{\vect{\theta}}^*,\widehat{\vect{e}}_i\rangle\right)\right)(1-\alpha_{l+1}\gamma)+\alpha_{l+1}^2G^2\\
          &+2\alpha_{l+1}\left(\varepsilon-\langle \widetilde{\vect{\theta}}^{(l)}-\widetilde{\vect{\theta}}^*,\widehat{\vect{e}}_l\rangle\right)\\
          \leq &\frac{2\alpha_1^2G^2(l+1-\alpha_1\gamma)}{(l+1)^2}+\frac{\alpha_1^2G^2}{(l+1)^2}+\sum_{i=0}^{l}\eta_{l+1,i}\left(\varepsilon-\langle\widetilde{\vect{\theta}}^{(i)}-\widetilde{\vect{\theta}}^*,\widehat{\vect{e}}_i\rangle\right)\\
          \leq &\frac{2\alpha_1^2G^2}{l+2}+\sum_{i=0}^{l}\eta_{l+1,i}\left(\varepsilon-\langle\widetilde{\vect{\theta}}^{(i)}-\widetilde{\vect{\theta}}^*,\widehat{\vect{e}}_i\rangle\right).
      \end{split}
  \end{equation}
  Here the last two lines are due to range of $\alpha_1$ and the definitions of $\eta_{l,i}$. The next step is to bound $\sum_{i=0}^{K-1}\eta_{K,i}\left(\varepsilon-\langle\widetilde{\vect{\theta}}^{(i)}-\widetilde{\vect{\theta}}^*,\widehat{\vect{e}}_i\rangle\right)$. First we have
  \begin{equation}
  \begin{split}
      &\left|\sum_{i=0}^{K-1}\eta_{K,i}\left(\varepsilon-\langle\widetilde{\vect{\theta}}^{(i)}-\widetilde{\vect{\theta}}^*,\widehat{\vect{e}}_i\rangle\right)\right|\\
      \leq &\frac{2\alpha_1}{K}\sum_{i=0}^{K-1}\|\widetilde{\vect{\theta}}^{(i)}-\widetilde{\vect{\theta}}^*\|_2\|\widehat{\vect{e}}_i\|_2+2\alpha_1\varepsilon\\
      \leq &C\left(\max_{0\leq i\leq K-1}\|\widehat{\vect{e}}_i\|_2+\varepsilon\right).
  \end{split}
  \end{equation}
  Note that the distribution of each minibatch $\{\vect{X}_{\xi_{k+1}},\vect{y}_{\xi_{k+1}}\}_{i=1}^m$ is the same as sampling $m$ independent $\vect{x}_i$ from $\bbP$, and then sampling $\vect{y}_{\xi_{k+1}}\sim\mathcal{N}(0,\vect{K}^*_{\xi_{k+1}})$, thus we can apply Lemma \ref{lem:statistical_err_full_grad} on $\widetilde{g}(\vect{\theta}^{(k)})$ and $\widetilde{g}_k^*$.
  Combining Lemma \ref{lem:g_k_conv_property}, Lemma \ref{lem:statistical_err_full_grad} and \eqref{eq:SGD_err_bnd2} leads to the following conclusion.
  \begin{enumerate}
  \item If $s_{M+1}(m)=m$, $m>C$, then for any $0<\varepsilon<\frac{1}{2}$, with probability at least $1-CKm^{-c}-CK\exp\{-cm^{2\varepsilon}\}$,
      \begin{equation}
          (\theta^{(K)}_{M+1}-\theta^*_{M+1})^2\leq \frac{8G^2}{\gamma^2(K+1)}+Cm^{-\frac{1}{2}+\varepsilon},
      \end{equation}
      where $\gamma=\frac{1}{4\theta_{\max}^2}$. Let $\varepsilon< C\frac{\log\log m}{\log m}$, then $K\exp\{-cm^{2\varepsilon}\}\geq CKm^{-c}$, thus the probability term is $1-CK\exp\{-cm^{2\varepsilon}\}$.
      \item If $M=1$, $s_1(m)=\tau\log m$, $s_{2}(m)=m$, $m>C$, then for any $0<\varepsilon<\frac{1}{2}$, with probability at least \begin{equation*}
          1-CK\exp\{-c(\log m)^{2\varepsilon}\},
      \end{equation*}
      we have
      \begin{equation}
          (\theta^{(K)}_1-\theta^*_1)^2+(\theta^{(K)}_{2}-\theta^*_{2})^2\leq \frac{8G^2}{\gamma^2(K+1)}+C(\log m)^{-\frac{1}{2}+\varepsilon},
      \end{equation}
      where $\gamma=\min\left\{\frac{1}{32\tau b\theta_{\max}^2},\frac{1}{4\theta_{\max}^2}-\frac{2\theta_{\max}^2}{\tau b\theta_{\min}^4}\right\}$.
      \item If $s_i(m)=\tau\log m$ for $1\leq i\leq M$, $s_{M+1}(m)=m$, $m>C$, $b_2>2b_1$, then for any $0<\alpha<\min\{\frac{2b_1+b_2}{2b_1},\frac{2b_2-4b_1}{14b_1+b_2}\}$,$0<\varepsilon<\frac{1}{2}$, with probability at least \begin{equation*}
          1-CKm^{-\alpha}-CK\exp\{-c(\log m)^{2\varepsilon}\},
      \end{equation*}
      we have
      \begin{equation}
          (\theta^{(K)}_1-\theta^*_1)^2+(\theta^{(K)}_{M+1}-\theta^*_{M+1})^2\leq \frac{8G^2}{\gamma^2(K+1)}+C(\log m)^{-\frac{1}{2}+\varepsilon}+C\alpha,
      \end{equation}
      where $\gamma$ is defined in \eqref{eq:gamma_log_m_scaling}. Let $c(\log m)^{-1+2\varepsilon)}\alpha<C(\log m)^{-\frac{1}{2}+\varepsilon}$, then $Km^{-\alpha}\geq \exp\{-c(\log m)^{2\varepsilon}\}$ and thus the probability term can be written as $1-CK\exp\{-c(\log m)^{2\varepsilon}\}$ and the error bound is $\frac{8G^2}{\gamma^2(K+1)}+C(\log m)^{-\frac{1}{2}+\varepsilon}$.
  \end{enumerate}
  Here $c,C>0$ depend only on $M,\theta_{\min},\theta_{\max},b_1,\dots,b_{M}$.
 \end{proof}
\begin{proof}[Proof of Theorem \ref{thm:large_m_param_converg_poly}]
Define $\widehat{e}_k= (g(\vect{\theta}^{(k)}))_{M+1}-(g^*(\vect{\theta}^{(k)}))_{M+1}$. Following similar arguments from the proof of Theorem \ref{thm:large_m_param_converg} and applying Lemma \ref{lem:g_k_conv_property_poly}, one can show that
  \begin{equation}\label{eq:param_converg_iterbnd}
      \begin{split}
          (\theta^{(k)}_{M+1}-\theta^*_{M+1})^2\leq \frac{2\alpha_1^2G^2}{k+1}+\sum_{i=0}^{k-1}\eta_{k,i}\left(\epsilon-(\theta^{(i)}_{M+1}-\theta^*_{M+1})\widehat{e}_i\right),
      \end{split}
  \end{equation}
  where $\eta_{k,i}=2\alpha_{i+1}\prod_{j=i+2}^k(1-\alpha_j\gamma)$, $\gamma=\frac{1}{8\theta_{\max}^2}$ and $\varepsilon=Cm^{-\frac{8b_1^2-12b_1-6-\alpha(4b_1+3)}{4b_1(2b_1-1)}}$.
  Also note that
  \begin{equation}
  \begin{split}
      &\left|\sum_{i=0}^{K-1}\eta_{K,i}\left(\epsilon-(\theta^{(i)}_{M+1}-\theta^*_{M+1})\widehat{e}_i\right)\right|\\
      \leq &2\alpha_1(\theta_{\max}-\theta_{\min})\max_{0\leq i<K}\widehat{e}_i|+2\alpha_1\epsilon.
  \end{split}
  \end{equation}
  Similarly from the proof of Theorem \ref{thm:large_m_grad_converg}, we can apply Lemma \ref{lem:statistical_err_full_grad} on $(g(\vect{\theta}^{(k)}))_{M+1}$ and $(g^*(\vect{\theta}^{(k)}))_{M+1}$.
Therefore, combining \eqref{eq:param_converg_iterbnd} and Lemma \ref{lem:statistical_err_full_grad} leads to the following result:
 If $s_{M+1}(m)=m$, $m>C$, then for any $0<\alpha<\frac{8b_1^2-12b_1-6}{4b_1+3}$, $0<\varepsilon<\frac{1}{2}$, with probability at least $1-MKm^{-\alpha}-CK\exp\{-cm^{2\varepsilon}\}$,
      \begin{equation}
          (\theta^{(K)}_{M+1}-\theta^*_{M+1})^2\leq \frac{8G^2}{\gamma^2(K+1)}+Cm^{-\frac{8b_1^2-12b_1-6-\alpha(4b_1+3)}{4b_1(2b_1-1)}}+Cm^{-\frac{1}{2}+\varepsilon},
      \end{equation}
      where $\gamma=\frac{1}{8\theta_{\max}^2}$. 
  Here $c,C>0$ depend only on $M,\theta_{\min},\theta_{\max},b_1,\dots,b_{M}$. Let $$\frac{8b_1^2-12b_1-6-\alpha(4b_1+3)}{4b_1(2b_1-1)}=\frac{1}{2}-\varepsilon,$$ we arrive at the final conclusion.
 \end{proof}
\begin{proof}[Proof of Theorem \ref{thm:large_m_grad_converg_poly}]
Similarly from the proof of Theorem \ref{thm:large_m_grad_converg}, we utilize \eqref{eq:full_gradient_exp} and let $\lambda_{lj}$ be the $j$th largest eigenvalue of $\vect{K}_{f,n}^{(l)}$, $\lambda_{M+1,j}=1$.
By \eqref{eq:full_gradient_exp} and Lemma \ref{lem:eigen_ratio_bnds_poly}, for any $0<\alpha<\frac{8b_1^2-12b_1-6}{4b_1+3}$, with probability at least $1-Mn^{-\alpha}$,
\begin{equation}
    \begin{split}
        \left|\left(\nabla \ell^*(\vect{\theta}^{(k)})\right)_i\right|\leq &Cn^{-\frac{8b_1^2-12b_1-6-\alpha(4b_1+3)}{4b_1(2b_1-1)}},
    \end{split}
\end{equation}
for $1\leq i\leq M$, where $C>0$ depends on $M,\theta_{\min},\theta_{\max},b_1,\dots,b_{M}$. Meanwhile, 
\begin{equation}
    \begin{split}
        \left|\left(\nabla \ell^*(\vect{\theta}^{(k)})\right)_{M+1}\right|\leq C\left(|\theta^{(k)}_{M+1}-\theta^*_{M+1}|+n^{-\frac{8b_1^2-12b_1-6-\alpha(4b_1+3)}{4b_1(2b_1-1)}}\right).
    \end{split}
\end{equation}
Thus we have 
\begin{equation}
    \begin{split}
        \|\nabla \ell^*(\vect{\theta}^{(k)})\|_2^2\leq& C\left[n^{-\frac{8b_1^2-12b_1-6-\alpha(4b_1+3)}{2b_1(2b_1-1)}}+(\theta^{(k)}_{M+1}-\theta^*_{M+1})^2\right].
    \end{split}
\end{equation}
By \eqref{eq:full_grad_exp_bnd}, Theorem \ref{thm:large_m_param_converg_poly} and Lemma \ref{lem:statistical_err_full_grad}, for any $\varepsilon\in(\max\{0,f_1(b_1)\},\frac{1}{2})$, 
if $m>C$, 
then with probability at least $1-CK\left(m^{-f_2(b_1)\left[\varepsilon-f_1(b_1)\right]}+\exp\{-cm^{2\varepsilon}\}\right)$, we have
\begin{equation}
    \|\nabla \ell(\vect{\theta}^{(K)})\|_2^2\leq C\left[\frac{G^2}{K+1}
    +m^{-\frac{1}{2}+\varepsilon}\right],
\end{equation}
where $c,C>0$ depend only on $M,\theta_{\min},\theta_{\max},b_1,\dots,b_{M}$.
\end{proof}
\section{Proofs of Supporting Lemmas}\label{sec:proof_lem}
\begin{proof}[proof of Lemma \ref{lem:g_k_conv_property}]
  Let $\lambda^{(k)}_{lj}$ be the $j$th eigenvalue of $\vect{K}_{f,\xi_{k+1}}^{(l)}$ for $1\leq l\leq M$, and $\lambda^{(k)}_{M+1,j}=1$, then by the definition of $g^*(\vect{\theta}^{(k)})$, we have
   \begin{equation}
       \begin{split}
           (g^*(\vect{\theta}^{(k)}))_1=&\frac{1}{2s_1(m)}\tr\Bigg[\vect{K}_{\xi_{k+1}}(\vect{\theta}^{(k)})^{-1}\left(\vect{I}_m-\vect{K}_{\xi_{k+1}}(\vect{\theta}^*) \vect{K}_{\xi_{k+1}}(\vect{\theta}^{(k)})^{-1}\right)\vect{K}^{(1)}_{f,\xi_{k+1}}\Bigg]\\
           =&\frac{1}{2s_1(m)}\tr\Bigg[\vect{K}_{\xi_{k+1}}(\vect{\theta}^{(k)})^{-1}\left(\sum_{l=1}^{M} (\theta^{(k)}_l-\theta_l^{*})\vect{K}^{(l)}_{f,\xi_{k+1}}+(\theta^{(k)}_{M+1}-\theta^*_{M+1})\vect{I}_m\right)\\
           &\vect{K}_{\xi_{k+1}}(\vect{\theta}^{(k)})^{-1}\vect{K}^{(1)}_{f,\xi_{k+1}}\Bigg]\\
           =&\frac{1}{2s_1(m)}\sum_{l=1}^{M+1}(\theta^{(k)}_l-\theta_l^{*})\sum_{j=1}^m\frac{\lambda_{lj}^{(k)}\lambda_{1j}^{(k)}}{\left(\sum_{l=1}^{M+1} \theta^{(k)}_l\lambda_{lj}^{(k)}\right)^2},
       \end{split}
   \end{equation} 
   and
   \begin{equation}
       \begin{split}
           (g^*(\vect{\theta}^{(k)}))_{M+1}=&\frac{1}{2m}\tr\Bigg[\vect{K}_{\xi_{k+1}}(\vect{\theta}^{(k)})^{-1}\left(\vect{I}_m-\vect{K}_{\xi_{k+1}}(\vect{\theta}^*) \vect{K}_{\xi_{k+1}}(\vect{\theta}^{(k)})^{-1}\right)\Bigg]\\
           =&\frac{1}{2m}\sum_{l=1}^{M+1}(\theta^{(k)}_l-\theta_l^{*})\sum_{j=1}^m\frac{\lambda_{lj}^{(k)}}{\left(\sum_{l=1}^{M+1} \theta^{(k)}_l\lambda_{lj}^{(k)}\right)^2}.
       \end{split}
   \end{equation}
   We prove Lemma \ref{lem:g_k_conv_property} under two cases separately.
   \begin{enumerate}
       \item $s_i(m)=\tau\log m$ for $1\leq i\leq M$, $s_{M+1}(m)=m$, $\widetilde{\vect{\theta}}^{(k)}=(\theta^{(k)}_1,\theta^{(k)}_{M+1})^\top$, $\widetilde{\vect{\theta}}^*=(\theta^*_1,\theta^*_{M+1})^\top$ and $\widetilde{g}_k^*=((g^*(\vect{\theta}^{(k)}))_1,(g^*(\vect{\theta}^{(k)})_{M+1})^\top$.\\
       Under this case, 
       we can write $\langle \widetilde{\vect{\theta}}^{(k)}-\widetilde{\vect{\theta}}^*,\widetilde{g}_k^*\rangle$ as 
   \begin{equation*}
   \begin{split}
       \langle \widetilde{\vect{\theta}}^{(k)}-\widetilde{\vect{\theta}}^*,\widetilde{g}_k^*\rangle=(\widetilde{\vect{\theta}}^{(k)}-\widetilde{\vect{\theta}}^*)^\top \vect{A}(\widetilde{\vect{\theta}}^{(k)}-\widetilde{\vect{\theta}}^*)+\varepsilon,
   \end{split}
   \end{equation*}
   where each entry $A_{ij}$ of $\vect{A}\in \bbR^{2\times 2}$ is defined as follows:
   \begin{equation*}
   \begin{split}
       A_{11}=&\frac{1}{2\tau\log m}\sum_{j=1}^m\frac{\lambda_{1j}^{(k)2}}{\left(\sum_{l=1}^{M+1} \theta^{(k)}_l\lambda_{lj}^{(k)}\right)^2},\\
       A_{12}=&\frac{1}{2\tau\log m}\sum_{j=1}^m\frac{\lambda_{1j}^{(k)}}{\left(\sum_{l=1}^{M+1} \theta^{(k)}_l\lambda_{lj}^{(k)}\right)^2},\\
    A_{21}=&\frac{1}{2m}\sum_{j=1}^m\frac{\lambda_{1j}^{(k)}}{\left(\sum_{l=1}^{M+1} \theta^{(k)}_l\lambda_{lj}^{(k)}\right)^2},\\
    A_{22}=&\frac{1}{2m}\sum_{j=1}^m\frac{1}{\left(\sum_{l=1}^{M+1} \theta^{(k)}_l\lambda_{lj}^{(k)}\right)^2},
   \end{split}
   \end{equation*}
   and 
   \begin{equation}
   \begin{split}
       \varepsilon=&\frac{1}{2\tau\log m}\sum_{l=2}^{M}(\theta^{(k)}_l-\theta_l^{*})(\theta_1^{(k)}-\theta_1^*)\sum_{j=1}^m\frac{\lambda_{lj}^{(k)}\lambda_{1j}^{(k)}}{\left(\sum_{l=1}^{M+1} \theta^{(k)}_l\lambda_{lj}^{(k)}\right)^2}\\
       &+\frac{1}{2m}\sum_{l=2}^{M}(\theta^{(k)}_l-\theta_l^{*})(\theta_{M+1}^{(k)}-\theta_{M+1}^*)\sum_{j=1}^m\frac{\lambda_{lj}^{(k)}}{\left(\sum_{l=1}^{M+1} \theta^{(k)}_l\lambda_{lj}^{(k)}\right)^2},
       \end{split}
   \end{equation}
   for $M\geq 2$ and $\varepsilon=0$ for $M=1$.
   Note that the distribution of each minibatch $\vect{X}_{\xi_{k+1}}$ can be seen as $m$ independent samples from $\bbP$,
   thus we can still apply Lemma \ref{lem:eigen_ratio_bnds}, but substituting $n$ by $m$. 
   \begin{itemize}
   \item When $M=1$, apply \eqref{eq:eigen_ratio_bnds3} in Lemma \ref{lem:eigen_ratio_bnds} with $\epsilon=\frac{1}{2}$, then for any $0<\alpha<1$, with probability at least $1-2m^{-\alpha}$, we have 
   \begin{equation}\label{eq:A_bnd1}
       \begin{split}
           A_{11}\geq \frac{1}{32\tau b_1\theta_{\max}^2},&\quad A_{22}\geq \frac{1}{8\theta_{\max}^2}\left(1-\frac{C\log m}{m}\right),\\
           A_{12}\leq \frac{1}{2\tau b_1\theta_{\min}^2},&\quad A_{21}\leq \frac{\log m}{2b_1\theta_{\min}^2m}.
       \end{split}
   \end{equation}
   Also note that for any $\omega>0$,
   \begin{equation}\label{eq:grad_inner_prod_lwrbnd1_1}
       \begin{split}
           &(\widetilde{\vect{\theta}}^{(k)}-\widetilde{\vect{\theta}}^*)^\top \vect{A}(\widetilde{\vect{\theta}}^{(k)}-\widetilde{\vect{\theta}}^*)\\
           \geq &\left(A_{11}-\frac{(A_{12}+A_{21})\omega}{2}\right)(\theta^{(k)}_1-\theta_1^{*})^2
           +\left(A_{22}-\frac{(A_{12}+A_{21})}{2\omega}\right)(\theta^{(k)}_{2}-\theta_{2}^*)^2
       \end{split}
   \end{equation}
   Let $\omega=\frac{\theta_{\min}^2}{16\theta_{\max}^2}$,
   then by \eqref{eq:A_bnd1} and \eqref{eq:grad_inner_prod_lwrbnd1_1}, one can show that
   \begin{equation}
       \begin{split}
           &\langle \widetilde{\vect{\theta}}^{(k)}-\widetilde{\vect{\theta}}^*,\widetilde{g}_k^*\rangle\\
           \geq &\frac{1}{64 \tau b\theta_{\max}^2}(\theta^{(k)}_1-\theta_1^*)^2+\left(\frac{1}{8\theta_{\max}^2}-\frac{4\theta_{\max}^2}{\tau b\theta_{\min}^{4}}\right)(\theta^{(k)}_{2}-\theta_{2}^*)^2-C\frac{\log m}{m}\\
           \geq &\frac{\gamma}{2}\|\widetilde{\vect{\theta}}^{(k)}-\widetilde{\vect{\theta}}^*\|_2^2-C\frac{\log m}{m},
       \end{split}
   \end{equation}
   where 
   \begin{equation}
       \gamma=\min\left\{\frac{1}{32 \tau b\theta_{\max}^2},\frac{1}{4\theta_{\max}^2}-\frac{2\theta_{\max}^2}{\tau b\theta_{\min}^{4}}\right\},
   \end{equation}
   and $C>0$ depends on $\theta_{\min},\theta_{\max},b_1$. It is guaranteed that $\gamma>0$ Since we have assumed 
   \begin{equation*}
   \tau>\frac{64\theta_{\max}^4}{b\theta_{\min}^{4}}.
   \end{equation*}
   Therefore, if $m>C$, for any $0<\alpha<1$, with probability $1-2m^{-\alpha}$, the following claims holds true:
       \begin{equation}\label{eq:g_k_conv_property}
           \langle \widetilde{\vect{\theta}}^{(k)}-\widetilde{\vect{\theta}}^*,\widetilde{g}_k^*\rangle
           \geq \frac{\gamma}{2}\|\widetilde{\vect{\theta}}^{(k)}-\widetilde{\vect{\theta}}^*\|_2^2-\varepsilon,
           \end{equation}
       where $\varepsilon=C\frac{\log m}{m}$ for some constant $C>0$ depending on $\theta_{\min},\theta_{\max},b_1$.
   \item When $M\geq 2$ and $b_2>2b_1$, apply \eqref{eq:eigen_ratio_bnds2} in Lemma \ref{lem:eigen_ratio_bnds} with $\epsilon=\frac{2b_1+b_2}{2b_2}$, then for any $0<\alpha<\min\{\frac{2b_1+b_2}{2b_2},\frac{2b_2-4b_1}{14b_1+b_2}\}$, with probability at least $1-3Mm^{-\alpha}$, we have 
   \begin{equation}\label{eq:quadratic_remainder_bnd}
       \begin{split}
           |\varepsilon|\leq CM(\theta_{\max}-\theta_{\min})^2\left(\frac{\alpha+1/\log m}{\tau\theta_{\min}^2}+\frac{\log m}{m}\right),
       \end{split}
   \end{equation}
   \begin{equation}\label{eq:A_bnd2}
       \begin{split}
           A_{11}\geq \frac{(2b_1+b_2)(b_2-2b_1)}{8\tau b_1b_2(b_2-b_1)(M+1)^2\theta_{\max}^2},&\quad A_{22}\geq \frac{1}{8\theta_{\max}^2}\left(1-\frac{C\log m}{m}\right),\\
           A_{12}\leq \frac{1}{2\tau b_1\theta_{\min}^2},&\quad A_{21}\leq \frac{\log m}{2b_1\theta_{\min}^2m}.
       \end{split}
   \end{equation}
   Also note that for any $\omega>0$,
   \begin{equation}\label{eq:grad_inner_prod_lwrbnd2_1}
       \begin{split}
           &(\widetilde{\vect{\theta}}^{(k)}-\widetilde{\vect{\theta}}^*)^\top \vect{A}(\widetilde{\vect{\theta}}^{(k)}-\widetilde{\vect{\theta}}^*)\\
           \geq &\left(A_{11}-\frac{(A_{12}+A_{21})\omega}{2}\right)(\theta^{(k)}_1-\theta_1^{*})^2
           +\left(A_{22}-\frac{(A_{12}+A_{21})}{2\omega}\right)(\theta^{(k)}_{M+1}-\theta_{M+1}^*)^2
       \end{split}
   \end{equation}
   Let 
   \begin{equation*}
   \omega=\frac{(2b_1+b_2)(b_2-2b_1)\theta_{\min}^2}{4b_2(b_2-b_1)(M+1)^2\theta_{\max}^2},
   \end{equation*}
   then by \eqref{eq:quadratic_remainder_bnd}, \eqref{eq:A_bnd2} and \eqref{eq:grad_inner_prod_lwrbnd2_1}, one can show that
   \begin{equation}
       \begin{split}
           &\langle \widetilde{\vect{\theta}}^{(k)}-\widetilde{\vect{\theta}}^*,\widetilde{g}_k^*\rangle\\
           \geq &\frac{(2b_1+b_2)(b_2-2b_1)}{16 \tau b_1b_2(b_2-b_1)(M+1)^2\theta_{\max}^2}(\theta^{(k)}_1-\theta_1^*)^2\\
           &+\left(\frac{1}{8\theta_{\max}^2}-\frac{b_2(b_2-b_1)(M+1)^2\theta_{\max}^2}{\tau b_1(2b_1+b_2)(b_2-2b_1)\theta_{\min}^{4}}\right)(\theta^{(k)}_{M+1}-\theta_{M+1}^*)^2\\
           &-\frac{CM(\theta_{\max}-\theta_{\min})^2}{\tau\theta_{\min}^2}(\alpha+(\log m)^{-1})\\
           &-\frac{C(M+1)^2\theta_{\max}^4(\theta_{\max}-\theta_{\min})^2}{\theta_{\min}^4}\frac{\log m}{m}\\
           \geq &\frac{\gamma}{2}\|\widetilde{\vect{\theta}}^{(k)}-\widetilde{\vect{\theta}}^*\|_2^2-C(\alpha+(\log m)^{-1}),
       \end{split}
   \end{equation}
   where 
   \begin{equation}
       \gamma=\min\left\{\frac{3(b_2-2b_1)}{8 \tau b_2(b_2-b_1)(M+1)^2\theta_{\max}^2},\frac{1}{4\theta_{\max}^2}-\frac{2b_2(b_2-b_1)(M+1)^2\theta_{\max}^2}{3\tau b_1^2(b_2-2b_1)\theta_{\min}^{4}}\right\},
   \end{equation}
   and $C>0$ depends on $M, \theta_{\min},\theta_{\max},b_1,\dots,b_{M}$. $C$ does not depend on $\tau$ since we have assumed 
   \begin{equation*}
   \tau>\frac{8b_2(b_2-b_1)(M+1)^2\theta_{\max}^4}{3b_1^2(b_2-2b_1)\theta_{\min}^{4}},
   \end{equation*}
  which implies $\gamma>0$. 
   Therefore, if $m>C$, for any $0<\alpha<\min\{\frac{2b_1+b_2}{2b_2},\frac{2b_2-4b_1}{14b_1+b_2}\}$, with probability $1-3Mm^{-\alpha}$, the following claims holds true:
       \begin{equation}\label{eq:g_k_conv_property}
           \langle \widetilde{\vect{\theta}}^{(k)}-\widetilde{\vect{\theta}}^*,\widetilde{g}_k^*\rangle
           \geq \frac{\gamma}{2}\|\widetilde{\vect{\theta}}^{(k)}-\widetilde{\vect{\theta}}^*\|_2^2-\varepsilon,
           \end{equation}
       where $\varepsilon=C(\alpha+(\log m)^{-1})$ for some constant $C>0$ depending on $M, \theta_{\min},\theta_{\max},b_1,\dots,b_{M}$.
   \end{itemize}
   
   \item $s_{M+1}(m)=m$, $\widetilde{\vect{\theta}}^{(k)}=\theta^{(k)}_{M+1}$, $\widetilde{\vect{\theta}}^*=\theta^*_{M+1}$ and $\widetilde{g}_k^*=(g^*(\vect{\theta}^{(k)}))_{M+1}$\\
   Under this case, we can apply \eqref{eq:eigen_ratio_bnds1} in Lemma \ref{lem:eigen_ratio_bnds}. Following similar arguments from the first case, one can show that with probability at least $1-3Mm^{-c}$, 
   \begin{equation}
       \langle \widetilde{\vect{\theta}}^{(k)}-\widetilde{\vect{\theta}}^*,\widetilde{g}_k^*\rangle\geq\frac{\gamma}{2} (\theta^{(k)}_{M+1}-\theta_{M+1}^*)^2-\varepsilon,
   \end{equation}
   where $\gamma=\frac{1}{4\theta_{\max}^2}$, $\varepsilon=\frac{C\log m}{m}$, if $m>C$. Here $C>0$ depends only on $\theta_{\min},\theta_{\max},M,b_1,\dots,b_{M}$.
   \end{enumerate}
   \end{proof}
 \begin{proof}[proof of Lemma \ref{lem:statistical_err_full_grad}]
 Without loss of generality, we start from bounding $\left(\nabla \ell(\vect{\theta})\right)_i-\left(\nabla \ell^*(\vect{\theta})\right)_i$ for an arbitrary $1\leq i\leq M+1$. By the definition of $\nabla \ell(\vect{\theta})$ and $\nabla \ell^*(\vect{\theta})$, we have
 \begin{equation}
 \begin{split}
     &\left(\nabla \ell(\vect{\theta})\right)_i-\left(\nabla \ell^*(\vect{\theta})\right)_i\\
     =&-\frac{1}{2n}\left[\vect{y}_n^\top \vect{K}_n^{-1}(\vect{\theta})\vect{K}_{f,n}^{(i)}\vect{K}_n^{-1}(\vect{\theta})\vect{y}_n-\tr\left(\vect{K}_n(\vect{\theta})^{-1}\vect{K}_{f,n}^{(i)}\vect{K}_n(\vect{\theta})^{-1}\vect{K}_n^*\right)\right]\\
     =&-\left((\vect{K}_n^{*})^{-\frac{1}{2}}\vect{y}_n\right)^\top \vect{A}(\vect{\theta})\left((\vect{K}_n^*)^{-\frac{1}{2}}\vect{y}_n\right)+\tr(\vect{A}(\vect{\theta})),
 \end{split}
 \end{equation}
 where $\vect{A}(\vect{\theta})=\frac{1}{2n}\vect{K}_n^{*\frac{1}{2}}\vect{K}_n^{-1}(\vect{\theta})\vect{K}_{f,n}^{(i)}\vect{K}_n^{-1}(\vect{\theta})\vect{K}_n^{*\frac{1}{2}}$. Since we have assumed that the eigenvectors of $\vect{K}_{f,n}^{(1)},\dots,\vect{K}_{f,n}^{(M)}$ are all the same in Assumption \ref{assump:same_eigenvec}, we can write $\vect{K}_{f,n}^{(j)}=\vect{P}^\top \vect{\Lambda}_j\vect{P}$ for all $j$, where $\vect{P}$ is an orthogonal matrix and $\vect{\Lambda}_j$ is a diagonal matrix consisting of the eigenvalues of $\vect{K}_{f,n}^{(j)}$. Also let $\vect{\Lambda}_{M+1}=\vect{I}_n$, then we have 
 \begin{equation}
 \vect{A}(\vect{\theta}^{(k)})=\vect{P}^\top \frac{1}{2n}\left(\sum_{l=1}^{M+1}\theta_l^{*}\vect{\Lambda}_l\right) \left(\sum_{l=1}^{M+1} \theta_l\vect{\Lambda}_l\right)^{-2}\vect{\Lambda}_i\vect{P}.   
 \end{equation}
 Let $\vect{z}_n=\vect{P}(\vect{K}_n^*)^{-\frac{1}{2}}\vect{y}_n$, and $\vect{\Lambda}(\vect{\theta})=\left(\sum_{l=1}^{M+1}\theta_l^{*}\vect{\Lambda}_l\right) \left(\sum_{l=1}^{M+1} \theta_l\vect{\Lambda}_l\right)^{-2}\vect{\Lambda}_i$, where $\theta_l$ is the $l$th entry of $\vect{\theta}$, then our goal is to derive a bound for
 \begin{equation*}
 \begin{split}
     &\sup_{\vect{\theta}\in [\theta_{\min},\theta_{\max}]^{M+1}}\frac{1}{2n}\left|\vect{z}_n^\top \vect{\Lambda}(\vect{\theta})\vect{z}_n-\tr(\vect{\Lambda}(\vect{\theta}))\right|.
 \end{split}
 \end{equation*}
 We claim that there exists an $\varepsilon$-net $\{\vect{\theta}_{\varepsilon}^{(1)},\dots,\vect{\theta}_{\varepsilon}^{(N)}\}$ of $[\theta_{\min},\theta_{\max}]^{M+1}$ under $\|\cdot\|_{\infty}$, with size $N=(1+\frac{(\theta_{\max}-\theta_{\min})}{\varepsilon})^{M+1}$. That is to say, for any $\vect{\theta}\in [\theta_{\min},\theta_{\max}]^{M+1}$, $\exists \widetilde{\vect{\theta}}\in \{\vect{\theta}_{\varepsilon}^{(1)},\dots,\vect{\theta}_{\varepsilon}^{(N)}\}$ such that $\vect{\Delta}=\vect{\theta}-\widetilde{\vect{\theta}}$ satisfies$\|\vect{\Delta}\|_{\infty}\leq \varepsilon$. The following proof of this claim is very similar to the proof of Lemma 5.2 in \cite{vershynin2010introduction}.
 
 Define $\vect{\theta}_c=(\frac{\theta_{\min}+\theta_{\max}}{2},\dots,\frac{\theta_{\min}+\theta_{\max}}{2})\in \bbR^{M+1}$, then an alternative way to represent $[\theta_{\min},\theta_{\max}]^{M+1}$ is $\vect{\theta}_c+B_{\infty}(\frac{\theta_{\max}-\theta_{\min}}{2})$. Let $\{\vect{\theta}^{(1)}_{\varepsilon},\dots,\vect{\theta}^{(N)}_{\varepsilon}\}$ be a maximal $\varepsilon$-separated subset of $\vect{\theta}_c+B_{\infty}(\frac{\theta_{\max}-\theta_{\min}}{2})$ (not the iterates of the SGD algorithm), which means that it is an $\varepsilon$-net of $\vect{\theta}_c+B_{\infty}(\frac{\theta_{\max}-\theta_{\min}}{2})$, and $\forall 1\leq i\neq j\leq N$, $\|\vect{\theta}_{\varepsilon}^{(i)}-\vect{\theta}_{\varepsilon}^{(j)}\|_{\infty}\geq \varepsilon$. Consider the $\ell_{\infty}$ balls with centers $\{\vect{\theta}_{\varepsilon}^{(i)}\}_{i=1}^N$ and radius $\frac{\varepsilon}{2}$, then these balls are disjoint and are subsets of $\vect{\theta}_c+B_{\infty}(\frac{\theta_{\max}-\theta_{\min}+\varepsilon}{2})$. Thus the sum of volumes of these balls is bounded by that of $\vect{\theta}_c+B_{\infty}(\frac{\theta_{\max}-\theta_{\min}+\varepsilon}{2})$, which finishes the proof of
 \begin{equation}
     N\leq\left(1+\frac{\theta_{\max}-\theta_{\min}}{\varepsilon}\right)^{M+1},
 \end{equation}
    In the following we linearize $\vect{\Lambda}(\vect{\theta})=\vect{\Lambda}(\widetilde{\vect{\theta}}+\vect{\Delta})$ based on the Taylor series expression of each of its diagonal entries, so that the upper bound for $\left|\vect{z}_n^\top \vect{\Lambda}(\vect{\theta})\vect{z}_n-\tr(\vect{\Lambda}(\vect{\theta}))\right|$ can be implied by some bounds related to $\widetilde{\vect{\theta}}$. 
    For any $1\leq j\leq m$, denote the $j$th diagonal entry of $\vect{\Lambda}_l$ by $\lambda_{lj}$ which is independent of $\vect{\theta}$, then the $j$th diagonal entry of $\vect{\Lambda}(\vect{\theta})$ can be written as follows:
     \begin{equation}\label{eq:quadratic_diagonal_entries}
    \begin{split}
        \vect{\Lambda}_{jj}(\vect{\theta})=\frac{\sum_{l=1}^{M+1}\lambda_{lj}\lambda_{ij}\theta_l^{*}}{\left(\sum_{l=1}^{M+1}\lambda_{lj}\theta_l\right)^2}.
        \end{split}
    \end{equation}
    Meanwhile, let $\Delta_l$ and $\widetilde{\theta}_l$ be the $l$th entry of $\vect{\Delta}$ and $\vect{\theta}$, then one can show that
    \begin{equation}\label{eq:quadratic_diagonal_taylor1}
    \begin{split}
        \frac{1}{\left(\sum_{l=1}^{M+1}\lambda_{lj}\theta_l\right)^2}=&\frac{1}{\left(\sum_{l=1}^{M+1} \lambda_{lj}\widetilde{\theta}_l+\sum_{l=1}^{M+1} \lambda_{lj}\Delta_l\right)^2}\\
        =&\left(\sum_{l=1}^{M+1} \lambda_{lj}\widetilde{\theta}_l\right)^{-2}\left(1+\frac{\sum_{l=1}^{M+1} \lambda_{lj}\Delta_l}{\sum_{l=1}^{M+1}\lambda_{lj}\widetilde{\theta}_l}\right)^{-2}\\
        =&\left(\sum_{l=1}^{M+1} \lambda_{lj}\widetilde{\theta}_l\right)^{-2}\sum_{h=0}^{H-1} \frac{h+1}{\left(-\sum_{l=1}^{M+1}\lambda_{lj}\widetilde{\theta}_l\right)^h}\left(\sum_{l=1}^{M+1} \lambda_{lj}\Delta_l\right)^{h}\\
        &+\frac{H+1}{(1+\xi)^{H+2}\left(-\sum_{l=1}^{M+1}\lambda_{lj}\widetilde{\theta}_l\right)^{H+2}}\left(\sum_{l=1}^{M+1} \lambda_{lj}\Delta_l\right)^{H}\\
        =&\left(\sum_{l=1}^{M+1} \lambda_{lj}\widetilde{\theta}_l\right)^{-2}\left(\sum_{h_1+\dots+h_{M+1}\leq H-1}\alpha^{(j)}_{h_1,\dots,h_{M+1}}\prod_{l=1}^{M+1}\Delta_l^{h_l}+\text{RES}_H^{(j)}(\vect{\theta})\right),
        \end{split}
    \end{equation}
 where the third equality holds if $\left|\sum_{l=1}^{M+1} \lambda_{lj}\Delta_l\right|<\sum_{l=1}^{M+1}\lambda_{lj}\widetilde{\theta}_l$, which is implied by $\|\vect{\Delta}\|_{\infty}\leq \theta_{\min}$, and we will choose $\varepsilon$ small enough to satisfy this. Here $\xi$ lies between $0$ and $\frac{\sum_{l=1}^{M+1} \lambda_{lj}\Delta_l}{\sum_{l=1}^{M+1}\lambda_{lj}\widetilde{\theta}_l}$, 
 \begin{equation}
 \begin{split}
      \alpha_{h_1,\dots,h_{M+1}}^{(j)}=&\frac{(\sum_{l=1}^{M+1}h_l+1)!\prod_{l=1}^{M+1} \lambda_{lj}^{h_l}}{h_1!\cdots h_{M+1}!(-\sum_{l=1}^{M+1}\lambda_{lj}\widetilde{\theta}_l)^{\sum_{l=1}^{M+1} h_l}},\\
      \text{RES}_H^{(j)}(\vect{\theta})=&\sum_{h_1+\cdots +h_{M+1}=H}\frac{(H+1)!\prod_{l=1}^{M+1}\lambda_{lj}^{h_l}\Delta_l^{h_l}}{h_1!\cdots h_{M+1}!(1+\xi)^{H+2}(-\sum_{l=1}^{M+1}\lambda_{lj}\widetilde{\theta}_l)^H}.
 \end{split}
 \end{equation}
 The quantities above satisfy 
 \begin{equation}
 \begin{split}
     |\alpha_{h_1,\dots,h_{M+1}}^{(j)}|\leq \left(\sum_{l=1}^{M+1} h_l+1\right)\left(\frac{M+1}{\theta_{\min}}\right)^{\sum_{l=1}^{M+1} h_l},
     |\text{RES}_H^{(j)}(\vect{\theta})|\leq &(H+1)\left(\frac{\varepsilon (M+1)}{\theta_{\min}}\right)^H,
 \end{split}
 \end{equation}
 since $\sum_{h_1+\cdots+h_{M+1}=h}\frac{h!}{h_1!\cdots h_{M+1}!}=(M+1)^h$.
 Define the following diagonal matrices:\\ $\vect{\Lambda}^{(h_1,\dots,h_{M+1})}(\widetilde{\vect{\theta}}), \vect{\Lambda}^{(H)}(\vect{\theta})\in \bbR^{n\times n}$ are with diagonal entries
 \begin{equation}
     \vect{\Lambda}^{(h_1,\dots,h_{M+1})}_{jj}(\widetilde{\vect{\theta}})=\alpha_{h_1,\dots,h_{M+1}}^{(j)}\vect{\Lambda}_{jj}(\widetilde{\vect{\theta}}),\vect{\Lambda}^{(H)}_{jj}(\vect{\theta})=\text{RES}_H^{(j)}(\vect{\theta})\vect{\Lambda}_{jj}(\widetilde{\vect{\theta}}).
 \end{equation}
 Then we can write
 \begin{equation*}
     \vect{\Lambda}(\vect{\theta})=\sum_{h_1+\cdots+h_{M+1}\leq H-1}\vect{\Lambda}^{(h_1,\dots,h_{M+1})}(\widetilde{\vect{\theta}})\prod_{l=1}^{M+1}\Delta_l^{h_l}+\vect{\Lambda}^{(H)}(\vect{\theta}),
 \end{equation*}
 and thus 
 \begin{equation}\label{eq:quadratic_sums}
     \begin{split}
         &\left|\vect{z}_n^\top \vect{\Lambda}(\vect{\theta})\vect{z}_n-\tr(\vect{\Lambda}(\vect{\theta}))\right|\\
         \leq&\max_{1\leq k\leq N}\sum_{h_1+\cdots+h_{M+1}\leq H-1}\varepsilon^{\sum_{l=1}^{M+1}h_l}\left|\vect{z}_n^\top \vect{\Lambda}^{(h_1,\dots,h_{M+1})}(\vect{\theta}_{\varepsilon}^{(k)})\vect{z}_n-\tr(\vect{\Lambda}^{(h_1,\dots,h_{M+1})}(\vect{\theta}_{\varepsilon}^{(k)}))\right|\\
         &+\left|\vect{z}_n^\top \vect{\Lambda}^{(H)}(\vect{\theta})\vect{z}_n-\tr(\vect{\Lambda}^{(H)}(\vect{\theta}))\right|.
     \end{split}
 \end{equation}
 In order to provide an upper bound for the first term above, we first bound  \begin{equation*}\left|\vect{z}_n^\top \vect{\Lambda}^{(h_1,\dots,h_{M+1})}(\vect{\theta}_{\varepsilon}^{(k)})\vect{z}_n-\tr(\vect{\Lambda}^{(h_1,\dots,h_{M+1})}(\vect{\theta}_{\varepsilon}^{(k)}))\right|\end{equation*} for an arbitrary $k$. First note that for any $1\leq k\leq N$,
 \begin{equation}\label{eq:quadratic_l2_norm}
     \|\vect{\Lambda}(\vect{\theta}_{\varepsilon}^{(k)})\|_2=\max_j \frac{\sum_{l=1}^{M+1}\lambda_{lj}\lambda_{ij}\theta_l^*}{\left(\sum_{l=1}^{M+1}\lambda_{lj}\theta^{(k)}_l\right)^2}\leq\frac{\theta_{\max}}{\theta_{\min}^2}.
 \end{equation}
 While for $\|\vect{\Lambda}(\vect{\theta}_{\varepsilon}^{(k)})\|_F^2$, one can show that
 \begin{equation}\label{eq:quadratic_frob_norm}
     \begin{split}
         \|\vect{\Lambda}(\vect{\theta}_{\varepsilon}^{(k)})\|_F^2\leq &\theta_{\max}^2\sum_{j=1}^n\frac{(\sum_{l=1}^{M+1}\lambda_{lj}\lambda_{ij})^2}{(\sum_{l=1}^{M+1}\lambda_{lj}\theta^{(k)}_l)^4}\\
         \leq& C\sum_{j=1}^n\frac{\sum_{l=1}^{M+1}\lambda_{lj}\lambda_{ij}}{(\sum_{l=1}^{M+1}\lambda_{lj}\theta_l)^2}.
     \end{split}
 \end{equation}
 Let \begin{equation*}t_i(n)=\sum_{j=1}^n\frac{\sum_{l=1}^{M+1}\lambda_{lj}\lambda_{ij}}{(\sum_{l=1}^{M+1}\lambda_{lj}\theta_l)^2},\end{equation*} then a deterministic bound for $t_i(n)$ is 
 \begin{equation}\label{eq:t_n_determ_bnd}
     t_i(n)\leq Cn.
 \end{equation} 
 If Assumption \ref{assump:eigen_decay_exp} holds, applying Lemma \ref{lem:eigen_ratio_bnds} without the condition $b_2>2b_1$ leads to 
 \begin{equation}\label{eq:t_n_prob_bnd}
     t_i(n)\leq C\log n\text{ for }1\leq i\leq M, \text{ and }t_{M+1}(n)\leq Cn,
 \end{equation}
 with probability at least $1-3Mn^{-c}$ for any constant $c>0$, if $n>C$. Here $C>0$ depends only on $M,b_1,\dots,b_{M},\theta_{\min},\theta_{\max}$. Therefore, by the definition of $\vect{\Lambda}^{(h_1,\dots,h_{M+1})}(\vect{\theta}_{\varepsilon}^{(k)})$, for any $\vect{\theta}_{\varepsilon}^{(k)}$, 
 \begin{equation}
 \begin{split}
     \|\vect{\Lambda}^{(h_1,\dots,h_{M+1})}(\vect{\theta}_{\varepsilon}^{(k)})\|_2\leq&C\left(\sum_{l=1}^{M+1} h_l+1\right)\left(\frac{M+1}{\theta_{\min}}\right)^{\sum_{l=1}^{M+1} h_l},\\
     \|\vect{\Lambda}^{(h_1,\dots,h_{M+1})}(\vect{\theta}_{\varepsilon}^{(k)})\|_F^2\leq& C\left(\sum_{l=1}^{M+1} h_l+1\right)^2\left(\frac{M+1}{\theta_{\min}}\right)^{2\sum_{l=1}^{M+1} h_l}t_i(n).
 \end{split}
 \end{equation}
 Let $\varepsilon=\frac{\theta_{\min}}{e(M+1)H}$, then by applying Hanson-wright's inequality, one can show that with probability at least
 $1-2\exp\{-c\min\{t,\frac{t^2}{t_i(n)}\}\}$,
 \begin{equation}\label{eq:quadratic_bnds}
     \left|\vect{z}_n^\top \vect{\Lambda}^{(h_1,\dots,h_{M+1})}(\vect{\theta}_{\varepsilon}^{(k)})\vect{z}_n-\tr(\vect{\Lambda}^{(h_1,\dots,h_{M+1})}(\vect{\theta}_{\varepsilon}^{(k)}))\right|\leq (e\varepsilon)^{-\sum_{l=1}^{M+1}h_l}t,
 \end{equation}
 where $c>0$ depends on $M, \theta_{\min},\theta_{\max}, b_1,\dots,b_{M}$.
 Meanwhile, the following lemma provides an upper bound for the residual term:
 \begin{lemma}\label{lem:quadratic_remainder}
   $\exists$ $c, C>0$ depending only on $M,\theta_{\min},\theta_{\max},b_1,\dots,b_{M}$ such that,
    \begin{equation}
    \begin{split}
        &\bbP\left[\sup_{\vect{\theta}\in [\theta_{\min},\theta_{\max}]^{M+1}}|\vect{z}_n^\top \vect{\Lambda}^{(H)}(\vect{\theta})\vect{z}_n-\tr(\vect{\Lambda}^{(H)}(\vect{\theta}))|>e^{-H}(t+Ct_i(n))\right]\\
        \leq& 2\exp\left\{-c\min\left\{\frac{t^2}{t_i(n)},t\right\}\right\}.
    \end{split}
    \end{equation}
 \end{lemma}
 Now we take a union bound for each term in \eqref{eq:quadratic_sums}, 
 then with probability at least 
 \begin{equation}
     \begin{split}
       &1-2\left({H+M+1\choose M+1}N+1\right)\exp\left\{-c\min\left\{\frac{t^2}{t_i(n)},t\right\}\right\}\\
 \geq &1-CH^{2M+2}\exp\left\{-c\min\left\{\frac{t^2}{t_i(n)},t\right\}\right\},
 \end{split}
 \end{equation}
 we have
 \begin{equation}
     \begin{split}
       &\sup_{\vect{\theta}\in [\theta_{\min},\theta_{\max}]^{M+1}}\left|\vect{z}_n^\top \vect{\Lambda}(\vect{\theta})\vect{z}_n-\tr(\vect{\Lambda}(\vect{\theta}))\right|\\
       \leq &t\left[\sum_{h=0}^{H-1}{h+M+1\choose M+1}e^{-h}+e^{-H}\right]+Ce^{-H}t_i(n)\\
       \leq&C(t+e^{-H}t_i(n)).
     \end{split}
 \end{equation}
 If Assumption \ref{assump:eigen_decay_exp} holds, $s_i(n)=\tau\log n$ for some $\tau$ satisfying \eqref{eq:tau_lowerbnd}, we apply the probabilistic bound \eqref{eq:t_n_prob_bnd} on $t_i(n)$. For any $x>0$, let $H=\log \frac{1}{x}$ and $t=x\log n$, then with probability at least 
  \begin{equation*}1-Cn^{-c}-C(\log x)^{2(M+1)}\exp\left\{-c\log n\min\left\{x^2,x\right\}\right\}, \end{equation*} we have
 \begin{equation}
     \sup_{\vect{\theta}\in [\theta_{\min},\theta_{\max}]^{M+1}}\left|\vect{z}_n^\top \vect{\Lambda}(\vect{\theta})\vect{z}_n-\tr(\vect{\Lambda}(\vect{\theta}))\right|\leq Cx\log n,
 \end{equation}
 which implies
 \begin{equation}
     \sup_{\vect{\theta}\in [\theta_{\min},\theta_{\max}]^{M+1}}\frac{n}{s_i(n)}\left|\left(\nabla \ell(\vect{\theta})\right)_i-\left(\nabla \ell^*(\vect{\theta})\right)_i\right|\leq Cx.
 \end{equation}
 Otherwise, if $s_i(n)=n$, we apply the deterministic bound \eqref{eq:t_n_determ_bnd} on $t_i(n)$. For any $x>0$, let $H=\log \frac{1}{x}$ and $t=xn$, then with probability at least 
  \begin{equation*}1-C(\log x)^{2(M+1)}\exp\left\{-cn\min\left\{x^2,x\right\}\right\}, \end{equation*} we have
 \begin{equation}
     \sup_{\vect{\theta}\in [\theta_{\min},\theta_{\max}]^{M+1}}\left|\vect{z}_n^\top \vect{\Lambda}(\vect{\theta})\vect{z}_n-\tr(\vect{\Lambda}(\vect{\theta}))\right|\leq Cxn,
 \end{equation}
 which implies
 \begin{equation}
     \frac{n}{s_i(n)}\left|\left(\nabla \ell(\vect{\theta})\right)_i-\left(\nabla \ell^*(\vect{\theta})\right)_i\right|\leq Cx.
 \end{equation}
 \end{proof}
 \begin{proof}[proof of Lemma \ref{lem:eigen_ratio_bnds}]
In order to prove Lemma \ref{lem:eigen_ratio_bnds}, we need to derive upper and lower bounds for $\lambda_{lj}, 1\leq l\leq M$ w.h.p. First we restate Theorem 1 and Theorem 4 in \cite{braun2006accurate} on the bounds for $\lambda_{lj}$ in the following:
    \begin{lemma}\label{lem:eigenval_err_bnds}
    Let $k$ be a Mercer kernel on a probability space $\mathcal{X}$ with probability measure $\bbP$, satisfying $k(x,x)\leq 1$ for all $x\in \mathcal{X}$, with eigenvalues $\{\lambda_i^*\}_{i=1}^{\infty}$. Let $\vect{K}_{f,n}\in \bbR^{n\times n}$ be the empirical kernel matrix evaluated on data $\{\vect{x}_1,\dots,\vect{x}_n\}$ i.i.d. sampled from $\bbP$, then the eigenvalues $\lambda_i(\vect{K}_{f,n})$ satisfies the following bound for $1\leq j, r\leq n$:
    \begin{equation*}
        \left|\frac{\lambda_j(\vect{K}_{f,n})}{n}-\lambda_j^*\right|\leq \lambda_j^* C(r,n)+E(r,n),
    \end{equation*}
    and for any $1\leq r\leq n$, there are two bounds for $C(r,n), E(r,n)$:
    \begin{enumerate}[label=(\roman*)]
        \item With probability at least $1-\delta$,
    \begin{equation}\label{eq:eigenval_err_bnd1}
        \begin{split}
            C(r,n)<&r\sqrt{\frac{2}{n\lambda_r^*}\log \frac{2r(r+1)}{\delta}}+\frac{4r}{3n\lambda^*_r}\log \frac{2r(r+1)}{\delta},\\
            E(r,n)<&\lambda_r^*+\sum_{i=r+1}^{\infty}\lambda_i^*+\sqrt{\frac{2\sum_{i=r+1}^{\infty}\lambda_i^*}{n}\log \frac{2}{\delta}}+\frac{2}{3n}\log\frac{2}{\delta};
        \end{split}
    \end{equation}
    \item With probability at least $1-\delta$,
    \begin{equation}\label{eq:eigenval_err_bnd2}
        \begin{split}
            C(r,n)<r\sqrt{\frac{r(r+1)}{n\delta\lambda_r^*}},\quad E(r,n)<\lambda_r^*+\sum_{i=r+1}^{\infty}\lambda_i^*+\sqrt{\frac{2\sum_{i=r+1}^{\infty}\lambda_i^*}{n\delta}}.
        \end{split}
    \end{equation} 
    \end{enumerate} 
    \end{lemma}
We consider two different upper bounds for $\lambda_{lj}$ that could be useful in later arguments. First we apply Lemma \ref{lem:eigenval_err_bnds} on $\vect{K}_{f,n}^{(l)}$. In particular, plug $r=j$ for each $1\leq j\leq n$ into \eqref{eq:eigenval_err_bnd2} and let $\delta=n^{-(1+\alpha)}$ for some $\alpha>0$. Then with probability at least $1-n^{-\alpha}$, for all $1\leq j\leq n$,
    \begin{equation*}
    	\begin{split}
    	    1+C(j,n)<&C_l^{-\frac{1}{2}}j^2n^{\frac{\alpha}{2}}e^{\frac{b_lj}{2}}\sqrt{\frac{j+1}{j}}+1<Cj^2n^{\frac{\alpha}{2}}e^{\frac{b_lj}{2}},\\
    	    E(r,n)<&\frac{C_le^{-b_lj}}{1-e^{-b_l}}+\sqrt{\frac{2C_le^{-b_l}}{1-e^{-b_l}}}e^{-\frac{b_lj}{2}}n^{\frac{\alpha}{2}}.
    	\end{split}
    \end{equation*}
    Thus we have
    \begin{equation}\label{eq:eigen_upp_bnd2_sigma_f}
    	\begin{split}
    	    \lambda_{lj}\leq &\left(Cj^2+\frac{C_ln^{-\frac{\alpha}{2}}e^{-\frac{b_l}{2}j}}{1-e^{-b_l}}+\sqrt{\frac{2C_le^{-b_l}}{1-e^{-b_l}}}\right)n^{1+\frac{\alpha}{2}}e^{-\frac{b_lj}{2}}\\
    	    \leq &C(\eta)n^{1+\frac{\alpha}{2}}e^{-\frac{b_lj}{2\eta}},
    	\end{split}
    \end{equation}
    where the last line holds for any $\eta>1$, and $C(\eta)>0$ depends on $b_l,\eta$. We will specify $\eta$ later to suit our needs.
    	    
    The second upper bound for $\lambda_{lj}$ requires applying \eqref{eq:eigenval_err_bnd2} with $r=\frac{1+\alpha}{b_l}\log n$, and $\delta=n^{-\alpha}$. Then with probability at least $1-n^{-\alpha}$,
    \begin{equation*}
    	1+C(r,n)<C_l^{-\frac{1}{2}}\sqrt{r^3(r+1)n^{2\alpha}}+1\leq C(\log n)^2n^{\alpha},
    \end{equation*}
    \begin{equation*}
    	\begin{split}
    	    E(r,n)<\frac{n^{-(1+\alpha)}}{1-e^{-b_l}}+\sqrt{\frac{2e^{-b_l}n^{-(1+\alpha)}}{(1-e^{-b_l})n^{(1-\alpha)}}}\leq \frac{C}{n},
        \end{split}
    \end{equation*}
   where $C$ depends on $b_l$. Thus $\lambda_{lj}\leq C(\log n)^2n^{1+\alpha}e^{-b_lj}+C$. Thus for any $\alpha>0$, with probability at least $1-2Mn^{-\alpha}$,
    \begin{equation}\label{eq:eigenval_upp_bnds}
    	\lambda_{lj}\leq  \min\left\{C(\eta)n^{1+\frac{\alpha}{2}}e^{-\frac{b_lj}{2\eta}},C(\log n)^2n^{1+\alpha}e^{-b_lj}+C\right\},
    \end{equation}
    holds for $\eta>1$, $1\leq l\leq M$, $1\leq j\leq n$, where $C>0$ depends on $b_1,\dots,b_{M}$, $C(\eta)$ depends on $b_1,\dots,b_M, \eta$.
    
    While for lower bounding $\lambda_{lj}$, we apply \eqref{eq:eigenval_err_bnd1} in Lemma \ref{lem:eigenval_err_bnds} with $r=\frac{\epsilon}{b_l}\log n$ for some $0<\epsilon<1$, and $\delta=n^{-\alpha}$ for some $0<\alpha<1$. Then when $n>C(\epsilon)$ for some constant $C(\epsilon)>0$ depending on $b_l, \epsilon$, with probability at least $1-n^{-\alpha}$,
    \begin{equation*}
    	\begin{split}
    	    C(r,n)<&r\sqrt{\frac{2\log\left[2r(r+1)n^{\alpha}\right]}{C_ln^{1-\epsilon}}}+\frac{4r\log\left[2r(r+1)n^{\alpha}\right]}{3C_ln^{1-\epsilon}}<\frac{1}{2},\\
    	    E(r,n)<&\frac{C_l}{(1-e^{-b_l})n^{\epsilon}}+\sqrt{\frac{2C_le^{-b_l}}{1-e^{-b_l}}}\sqrt{\frac{\log 2n^{\alpha}}{n^{1+\epsilon}}}+\frac{2}{3n}\log 2n^{\alpha}<Cn^{-\epsilon},
       \end{split}
    \end{equation*}
    thus $\lambda_{lj}\geq \frac{C_l}{2}ne^{-b_l j}-Cn^{1-\epsilon}$ for $C>0$ depending on $b_l$.
    	    
    Therefore, for any $0<\epsilon,\alpha<1$, if $n>C(\epsilon)$ for $C(\epsilon)>0$ depending on $b_1,\dots,b_{M},\epsilon$, then with probability at least $1-Mn^{-\alpha}$,
    \begin{equation}\label{eq:eigenval_lwr_bnds}
    \lambda_{lj}\geq \frac{C_l}{2}ne^{-b_lj}-Cn^{1-\epsilon},
    \end{equation}
    holds for $1\leq l\leq M$, $1\leq j\leq n$, where $C>0$ depends on $b_1,\dots,b_{M}$.
    Now we are ready to prove the bounds for $\sum_{j=1}^n \frac{\lambda_{lj}\lambda_{l'j}}{\left(\sum_{h=1}^{M+1} \theta_h\lambda_{hj}\right)^2}$ for $1\leq l,l'\leq M+1$.
    \begin{enumerate}
    	\item $l=l'=1$\\
    	First we derive an upper bound. Let $\eta=\frac{3}{2}$ in \eqref{eq:eigenval_upp_bnds}, then we have
    	\begin{equation}
    	    \begin{split}
    	        \sum_{j=1}^n \frac{\lambda_{1j}^2}{\left(\sum_{h=1}^{M+1} \theta_h\lambda_{hj}\right)^2}\leq&\frac{|\{n^{1+\frac{\alpha}{2}}e^{-\frac{b_1j}{3}}> 1\}|}{\theta_{\min}^2}+\sum_{n^{1+\frac{\alpha}{2}}e^{-\frac{b_1j}{3}}\leq1}\frac{\lambda_{1j}^2}{\theta_{\min}^2}\\
    	        \leq&\frac{|\{n^{1+\frac{\alpha}{2}}e^{-\frac{b_1j}{3}}> 1\}|}{\theta_{\min}^2}+\frac{Cn^{2+\alpha}}{\theta_{\min}^2}\sum_{e^{-\frac{2b_1j}{3}}\leq n^{-2-\alpha}}e^{-\frac{2b_1}{3}j}.
    	     \end{split}
    	\end{equation}
    	Since 
    	\begin{equation}
    	    \begin{split}
    	        n^{1+\frac{\alpha}{2}}e^{-\frac{b_1j}{3}}>1\Rightarrow j< \frac{6+3\alpha}{2b_1}\log n,
    	    \end{split}
    	\end{equation}
    	one can show that
    	\begin{equation}
    	    \begin{split}
    	        \sum_{j=1}^n \frac{\lambda_{1j}^2}{\left(\sum_{h=1}^{M+1} \theta_h\lambda_{hj}\right)^2}\leq&\frac{6+3\alpha}{2b_1\theta_{\min}^2}\log n+\frac{C}{\theta_{\min}^2}\leq \frac{4+2\alpha}{b_1\theta_{\min}^2}\log n,
    	     \end{split}
       \end{equation}
    	when $n>C$ for $C$ depending on $b_1$.
    	In terms of the lower bound, we discuss the proof for two cases separately:
    	\begin{itemize}
    	\item When $M=1$, first note that
    	\begin{equation}
    	    \begin{split}
    	        \sum_{j=1}^n \frac{\lambda_{1j}^2}{\left(\sum_{h=1}^{2} \theta_h\lambda_{hj}\right)^2}\geq &\sum_{j=1}^n \frac{\lambda_{1j}^2}{4\theta^2_{\max}\max_h\lambda_{hj}^2}\\
    	        \geq&\frac{\left|\{j:\lambda_{1j}=\max_h\lambda_{hj}\}\right|}{4\theta_{\max}^2}.
    	     \end{split}
    	 \end{equation}
    	Due to \eqref{eq:eigenval_lwr_bnds}, we have
    	\begin{equation}
    	    \begin{split}
    	        \lambda_{1j}=\max_h\lambda_{hj}\Leftarrow&\frac{C_1}{2}ne^{-b_1j}\geq Cn^{1-\epsilon}+C\\
    	        \Leftarrow& e^{-b_1j}\geq Cn^{-\epsilon}\\
    	        \Leftarrow& j\leq\frac{\epsilon}{b_1}\log n-C,
    	    \end{split}
    	\end{equation}
    	when $n>C$ for some $C>0$ depending on $b$, which implies
    	\begin{equation*}
    	    \left|\{j:\lambda_{1j}=\max_h\lambda_{hj}\}\right|\geq \frac{\epsilon}{b_1}\log n-C\geq \frac{\epsilon}{2b_1}\log n,
    	\end{equation*}
    	if $n>C$. Thus we have 
    	\begin{equation*}
    	    \begin{split}
    	        \sum_{j=1}^n \frac{\lambda_{1j}^2}{\left(\sum_{h=1}^{2} \theta_h\lambda_{hj}\right)^2}\geq \frac{\epsilon\log n}{8b_1\theta_{\max}^2},
    	    \end{split}
    	\end{equation*}
    	\item When $M\geq 2$ and $b_2>2b_1$, first note that
    	\begin{equation}
    	    \begin{split}
    	        \sum_{j=1}^n \frac{\lambda_{1j}^2}{\left(\sum_{h=1}^{M+1} \theta_h\lambda_{hj}\right)^2}\geq &\sum_{j=1}^n \frac{\lambda_{1j}^2}{(M+1)^2\theta^2_{\max}\max_h\lambda_{hj}^2}\\
    	        \geq&\frac{\left|\{j:\lambda_{1j}=\max_h\lambda_{hj}\}\right|}{(M+1)^2\theta_{\max}^2}.
    	     \end{split}
    	 \end{equation}
    	Due to \eqref{eq:eigenval_lwr_bnds}, we have
    	\begin{equation}
    	    \begin{split}
    	        \lambda_{1j}=\max_h\lambda_{hj}\Leftarrow&\frac{C_1}{2}ne^{-b_1j}\geq Cn^{1-\epsilon}+C(\log n)^2n^{1+\alpha}e^{-b_2j}+C\\
    	        \Leftarrow& e^{-b_1j}\geq Cn^{-\epsilon}+C(\log n)^2n^{\alpha}e^{-b_2j}\\
    	        \Leftarrow&\frac{\alpha}{b_2-b_1}\log n+\frac{3}{b_2-b_1}\log\log n\leq j\leq\frac{\epsilon}{b_1}\log n-C,
    	    \end{split}
    	\end{equation}
    	when $n>C$ for some $C>0$ depending on $b_1,\dots,b_{M}$, which implies
    	\begin{equation*}
    	    \left|\{j:\lambda_{1j}=\max_h\lambda_{hj}\}\right|\geq \frac{\epsilon b_2-(\epsilon+\alpha)b_1}{b_1(b_2-b_1)}\log n-\frac{3\log\log n}{b_2-b_1}-C.
    	\end{equation*}
    	Thus we have 
    	\begin{equation*}
    	    \begin{split}
    	        &\sum_{j=1}^n \frac{\lambda_{1j}^2}{\left(\sum_{h=1}^{M+1} \theta_h\lambda_{hj}\right)^2}\\
    	        \geq& \frac{\epsilon b_2-(\epsilon+\alpha)b_1}{b_1(b_2-b_1)(M+1)^2\theta_{\max}^2}\log n-\frac{3\log\log n}{(b_2-b_1)(M+1)^2\theta_{\max}^2}-C\\
    	        \geq&\frac{\epsilon(b_2-2b_1)\log n}{2b_1(b_2-b_1)(M+1)^2\theta_{\max}^2},
    	    \end{split}
    	\end{equation*}
    	if $n>C$ and $\alpha\leq \epsilon$.
    	\end{itemize}
      \item $l=l'=M+1$\\
    	   The upper bound for $\sum_{j=1}^n \frac{\lambda_{M+1,j}^2}{\left(\sum_{h=1}^{M+1}\theta_h\lambda_{hj}\right)^2}$ in Lemma \ref{lem:eigen_ratio_bnds} is straightforward, since $\sum_{h=1}^{M+1}\theta_h\lambda_{hj}\geq \theta_{\min}\lambda_{M+1,j}$. While for the lower bound, note that $\lambda_{M+1,j}=1$, and thus
    	   \begin{equation}
    	       \sum_{j=1}^n \frac{\lambda_{M+1,j}^2}{\left(\sum_{h=1}^{M+1}\theta_h\lambda_{hj}\right)^2}\geq \frac{|\{j:\sum_{h=1}^{M+1}\theta_h\lambda_{hj}\leq 2\theta_{\max}\}|}{4\theta_{\max}^2}.
    	   \end{equation}
    	   Meanwhile, let $\eta=\frac{3}{2}$ in \eqref{eq:eigenval_upp_bnds}, then one can show that
    	   \begin{equation}
    	       \begin{split}
    	           \sum_{h=1}^{M+1}\theta_h\lambda_{hj}\leq 2\theta_{\max}\Leftarrow & CM\theta_{\max}n^{1+\frac{\alpha}{2}}e^{-\frac{b_1j}{3}}\leq\theta_{\max}\\
    	           \Leftarrow&j\geq \frac{6+3\alpha}{2b_1}\log n+C\\
    	           \Leftarrow&j\geq \frac{6+3\alpha}{b_1}\log n\\
    	       \end{split}
    	   \end{equation}
    	   when $n>C$ for $C>0$ depending on $M,b_1,\dots,b_{M}$.
    	   Therefore, 
    	   \begin{equation}
    	       \sum_{j=1}^n \frac{\lambda_{M+1,j}^2}{\left(\sum_{h=1}^{M+1}\theta_h\lambda_{hj}\right)^2}\geq \frac{n-\frac{6+3\alpha}{b_1}\log n}{4\theta_{\max}^2}.
    	   \end{equation}
      \item $1<l\leq l'\leq M$\\
      First note that by similar arguments from the first case where $l=l'=1$, one can show that
      \begin{equation}
          \sum_{j=1}^n \frac{\lambda_{lj}\lambda_{l'j}}{\left(\sum_{h=1}^{M+1} \theta_h\lambda_{hj}\right)^2}\leq \frac{4+2\alpha}{b_l\theta_{\min}^2}\log n.
      \end{equation}
      Furthermore, if $M\geq 2$ and $b_2>2b_1$ hold, then we can utilize the following upper bound for each term $\frac{\lambda_{lj}\lambda_{l'j}}{\left(\sum_{h=1}^{M+1} \theta_h\lambda_{hj}\right)^2}$:
      \begin{equation}
          \begin{split}
              \frac{\lambda_{lj}\lambda_{l'j}}{\left(\sum_{h=1}^{M+1} \theta_h\lambda_{hj}\right)^2}\leq &\min\left\{\frac{1}{\theta_{\min}^2},\frac{C(\eta)n^{2+\alpha}e^{-\frac{b_l}{\eta}j}}{\theta_{\min}^2\left(\lambda_{1j}+1\right)^2}\right\},
          \end{split}
      \end{equation}
      where $\lambda_{1j}\geq \frac{C_1}{2}ne^{-b_1j}-Cn^{1-\epsilon}$.
      When $j<\frac{\epsilon}{b_1}\log n-C$ for some $C$ depending on $b_1,\dots,b_{M}$, we have $Cn^{1-\epsilon}<\frac{C_1}{4}ne^{-b_1j}$, and thus \begin{equation*}
          \frac{\lambda_{lj}\lambda_{l'j}}{\left(\sum_{h=1}^{M+1} \theta_h\lambda_{hj}\right)^2}\leq \frac{C(\eta)n^{\alpha}}{\theta_{\min}^2}e^{(2b_1-\frac{b_l}{\eta})j};
      \end{equation*}
      while for $j\geq \frac{\epsilon}{b_1}\log n-C$, we have the bound
      \begin{equation*}\frac{\lambda_{lj}\lambda_{l'j}}{\left(\sum_{h=1}^{M+1} \theta_h\lambda_{hj}\right)^2}\leq \frac{C(\eta)n^{2+\alpha}e^{-\frac{b_l}{\eta}j}}{\theta_{\min}^2}.\end{equation*}
      Let $\frac{2b_1}{b_2}<\epsilon<1$, $\eta=\frac{6b_1+\epsilon b_2}{8b_1}$, $\alpha\leq \frac{2\epsilon b_2-4b_1}{6b_1+\epsilon b_2}$, then one can show that
      \begin{equation}
          \begin{split}
              \sum_{j=1}^n \frac{\lambda_{lj}\lambda_{l'j}}{\left(\sum_{h=1}^{M+1} \theta_h\lambda_{hj}\right)^2}\leq&\frac{\alpha\log n}{\theta_{\min}^2(\frac{b_l}{\eta}-2b_1)}+\sum_{j=\lceil\frac{\alpha}{b_l/\eta-2b_1}\log n\rceil}^{\lfloor\frac{\epsilon}{b_1}\log n-C\rfloor}\frac{C(\eta)n^{\alpha}e^{-(\frac{b_l}{\eta}-2b_1)j}}{\theta_{\min}^2}\\
              &+\sum_{j=\lceil\frac{\epsilon}{b_1}\log n-C\rceil}^n \frac{C(\eta)n^{2+\alpha}e^{-\frac{b_lj}{\eta}}}{\theta_{\min}^2}\\
              \leq&\frac{(6b_1+b_2)\alpha}{2b_1(4b_l-b_2-6b_1)\theta_{\min}^2}\log n+\frac{C(\epsilon)}{\theta_{\min}^2},
          \end{split}
      \end{equation}
      for $C(\epsilon)>0$ depending on $\epsilon, b_1,\dots,b_{M}$. Here the last line is due to that when $n>C(\epsilon)$, we have $\frac{\epsilon b_2}{b_1}-\frac{Cb_2}{\log n}\geq \frac{2b_1+\epsilon b_2}{2b_1}$, and thus
       \begin{equation*}
      n^{2+\alpha}\exp\left\{-\frac{b_l}{\eta}\left(\frac{\epsilon}{b_1}\log n-C\right)\right\}\leq n^{\frac{4b_1-2\epsilon b_2}{6b_1+\epsilon b_2}+\alpha}\leq 1.
      \end{equation*}
      \item $1=l<l'\leq M$\\
      Similarly from the previous case, we first have the bound 
      \begin{equation}
          \sum_{j=1}^n \frac{\lambda_{lj}\lambda_{l'j}}{\left(\sum_{h=1}^{M+1} \theta_h\lambda_{hj}\right)^2}\leq \frac{4+2\alpha}{b_1\theta_{\min}^2}\log n,
      \end{equation}
      which holds with as long as $0<b_1<b_2<\dots<b_M$.
      If $M\geq 2$ and $b_2>2b_1$ hold, then we bound each term in the summation as follows
      \begin{equation}
          \begin{split}
              \frac{\lambda_{1j}\lambda_{l'j}}{\left(\sum_{h=1}^{M+1} \theta_h\lambda_{hj}\right)^2}\leq &\min\left\{\frac{1}{\theta_{\min}^2},\frac{C(\eta)n^{1+\frac{\alpha}{2}}e^{-\frac{b_{l'}}{2\eta}j}}{\theta_{\min}^2\left(\lambda_{1j}+1\right)}\right\}.
          \end{split}
      \end{equation}
     Let $\frac{2b_1}{b_2}<\epsilon<1$, $\eta=\frac{6b_1+\epsilon b_2}{8b_1}$, $\alpha\leq \frac{2\epsilon b_2-4b_1}{6b_1+\epsilon b_2}$, then one can show that
     \begin{equation}
         \begin{split}
             \sum_{j=1}^n \frac{\lambda_{1j}\lambda_{l'j}}{\left(\sum_{h=1}^{M+1} \theta_h\lambda_{hj}\right)^2}\leq&\frac{\alpha\log n}{\theta_{\min}^2(\frac{b_{l'}}{\eta}-2b_1)}+\sum_{j=\lceil\frac{\alpha}{b_{l'}/\eta-2b_1}\log n\rceil}^{\lfloor\frac{\epsilon}{b_1}\log n-C\rfloor}\frac{C(\eta)n^{\frac{\alpha}{2}}e^{-(\frac{b_{l'}}{2\eta}-b_1)j}}{\theta_{\min}^2}\\
              &+\sum_{j=\lceil\frac{\epsilon}{b_1}\log n-C\rceil}^n \frac{C(\eta)n^{1+\frac{\alpha}{2}}e^{-\frac{b_{l'}j}{2\eta}}}{\theta_{\min}^2}\\
              \leq&\frac{(6b_1+b_2)\alpha}{2b_1(4b_{l'}-b_2-6b_1)\theta_{\min}^2}\log n+\frac{C(\epsilon)}{\theta_{\min}^2},
         \end{split}
     \end{equation}
     for $C(\epsilon)>0$ depending on $\epsilon, b_1,\dots,b_{M}$. 
     \item $1<l<l'=M+1$\\
     Note that
     \begin{equation}
         \begin{split}
             \frac{\lambda_{lj}\lambda_{M+1,j}}{\left(\sum_{h=1}^{M+1} \theta_h\lambda_{hj}\right)^2}\leq &\min\left\{\frac{1}{\theta_{\min}^2},\frac{C(\eta)n^{1+\frac{\alpha}{2}}e^{-\frac{b_{l}}{2\eta}j}}{\theta_{\min}^2\left(\lambda_{1j}+1\right)}\right\},
         \end{split}
     \end{equation}
     thus based on the same argument as the previous case, we have
     \begin{equation}
          \sum_{j=1}^n \frac{\lambda_{lj}\lambda_{M+1,j}}{\left(\sum_{h=1}^{M+1} \theta_h\lambda_{hj}\right)^2}\leq \frac{4+2\alpha}{b_1\theta_{\min}^2}\log n,
      \end{equation}
      and when $M\geq 2$, $b_2>2b_1$,
     \begin{equation}
         \sum_{j=1}^n \frac{\lambda_{lj}\lambda_{M+1,j}}{\left(\sum_{h=1}^{M+1} \theta_h\lambda_{hj}\right)^2}\leq\frac{(6b_1+b_2)\alpha}{2b_1(4b_l-b_2-6b_1)\theta_{\min}^2}\log n+\frac{C(\epsilon)}{\theta_{\min}^2}.
     \end{equation}
    \item $l=1, l'=M+1$\\
    Since 
    \begin{equation}
        \frac{\lambda_{1j}\lambda_{M+1,j}}{\left(\sum_{h=1}^{M+1} \theta_h\lambda_{hj}\right)^2}\leq \min\left\{\frac{1}{4\theta_{\min}^2},\frac{C(\eta)n^{1+\frac{\alpha}{2}}e^{-\frac{b_1}{2\eta}j}}{\theta_{\min}^2}\right\},
    \end{equation}
    one can show that
    \begin{equation}
        \begin{split}
            \sum_{j=1}^n\frac{\lambda_{1j}\lambda_{M+1,j}}{\left(\sum_{h=1}^{M+1} \theta_h\lambda_{hj}\right)^2}\leq& \frac{(2+\alpha)\eta\log n}{4b_1\theta_{\min}^2}+\frac{C(\eta)n^{1+\frac{\alpha}{2}}}{\theta_{\min}^2}\sum_{j=\lceil\frac{(2+\alpha)\eta}{b_1}\log n\rceil}^ne^{-\frac{b_1}{2\eta}j}\\
            \leq&\left(\frac{(2+\alpha)\eta}{4}+\frac{C(\eta)}{\log n}\right)\frac{\log n}{b_1\theta_{\min}^2}.
        \end{split}
    \end{equation}
    Let $\eta=\frac{8}{7}$, then when $n>C$ for some $C>0$ depending on $b_1,\dots,b_{M}$,  \begin{equation*}\sum_{j=1}^n\frac{\lambda_{1j}\lambda_{Mj}}{\left(\sum_{h=1}^{M+1} \theta_h\lambda_{hj}\right)^2}\leq\frac{(5+2\alpha)\log n}{7b_1\theta_{\min}^2}.\end{equation*}
\end{enumerate}
Therefore, for any $\alpha>0$, if $n>C$ for $C>0$ depending on $M,b_1,\dots,b_{M}$, then with probability at least $1-3Mn^{-\alpha}$, \eqref{eq:eigen_ratio_bnds1} holds.

If $M=1$, then for any $0<\alpha,\epsilon<1$, with probability at least $1-2n^{-\alpha}$, \eqref{eq:eigen_ratio_bnds3} holds in addition to \eqref{eq:eigen_ratio_bnds1}; If $M\geq 2$ and $b_2>2b_1$ hold, then any $\frac{2b_1}{b_2}<\epsilon<1$, $0<\alpha< \min\left\{\frac{2\epsilon b_2-4b_1}{6b_1+\epsilon b_2},1\right\}$, if $n>C(\epsilon)$ for $C(\epsilon)>0$ depending on $M,b_1,\dots,b_{M},\epsilon$, with probability at least $1-3Mn^{-\alpha}$, \eqref{eq:eigen_ratio_bnds2} holds in addition to \eqref{eq:eigen_ratio_bnds1}.
\end{proof}
   \begin{proof}[Proof of Lemma \ref{lem:quadratic_remainder}]
    First note that
    \begin{equation}
    \begin{split}
        \left|z_n^\top \vect{\Lambda}^{(H)}(\vect{\theta})z_n-\tr(\vect{\Lambda}^{(H)}(\vect{\theta}))\right|=&\left|\sum_{j=1}^n \vect{\Lambda}_{jj}^{(H)}(\vect{\theta})(z_{nj}^2-1)\right|\\
\leq &\sum_{j=1}^n\vect{\Lambda}^{(H)}_{jj}(\vect{\theta})|z_{nj}^2-1|.
    \end{split}
    \end{equation}
     By the definition of $\vect{\Lambda}^{(H)}(\vect{\theta})$, $\varepsilon$, $t_i(n)$, \eqref{eq:quadratic_l2_norm} and \eqref{eq:quadratic_frob_norm},
     \begin{equation*}
     \begin{split}
     \|\vect{\Lambda}^{(H)}(\vect{\theta})\|_2\leq&\|\vect{\Lambda}(\widetilde{\vect{\theta}})\|_2(H+1)\left(\frac{\varepsilon (M+1)}{\theta_{\min}}\right)^H\leq Ce^{-H}, \\
         \|\vect{\Lambda}^{(H)}(\vect{\theta})\|_F^2\leq& e^{-2H}\|\vect{\Lambda}(\widetilde{\vect{\theta}})\|_F^2\leq Ce^{-2H}t_i(n).
     \end{split}
     \end{equation*}
    Also note that following similar arguments for bounding $\|\vect{\Lambda}(\vect{\theta}_{\varepsilon}^{(k)})\|_F^2$, we have
    \begin{equation}
        \sum_{j=1}^n\left|\vect{\Lambda}_{jj}(\widetilde{\vect{\theta}})\right|\leq Ct_i(n),
    \end{equation}
    and thus 
    \begin{equation}
    \begin{split}
        \sum_{j=1}^n\left|\vect{\Lambda}^{(H)}_{jj}(\vect{\theta})\right|\leq &e^{-H}\sum_{j=1}^n\left|\vect{\Lambda}_{jj}(\widetilde{\vect{\theta}})\right|\leq Ce^{-H}t_i(n).
        \end{split}
    \end{equation}
     Therefore, 
    \begin{equation}
    \begin{split}
        &\bbP\left(\left|z_n^\top \vect{\Lambda}^{(H)}(\vect{\theta})z_n-\tr(\vect{\Lambda}^{(H)}(\vect{\theta}))\right|>e^{-H}\left(t+Ct_i(n)\right)\right)\\
        \leq &\bbP\left(\sum_{i=1}^m\vect{\Lambda}_{jj}^{(H)}(\vect{\theta})(|z_{nj}^2-1|-\bbE(|z_{nj}^2-1|))>e^{-H}t\right),
    \end{split}
    \end{equation}
    where $C>0$ depends on $M, \theta_{\min},\theta_{\max},b_1,\dots,b_{M}$.
    Since $|z_{nj}^2-1|$ is sub-exponential with constant parameter, 
    \begin{equation}
        \bbP\left(\sum_{j=1}^n\vect{\Lambda}^{(H)}_{jj}(\vect{\theta})\left(|z_{nj}^2-1|-\bbE|z_{nj}^2-1|\right)>e^{-H}t\right)\leq 2\exp\left\{-c\min\left\{\frac{t^2}{t_i(n)},t\right\}\right\}.
    \end{equation}
    \end{proof}
\begin{proof}[proof of Lemma \ref{lem:g_k_conv_property_poly}]
    Following the calculations in the proof of Lemma \ref{lem:g_k_conv_property}, one can show that
    \begin{equation}
        \begin{split}
           &(g^*(\vect{\theta}^{(k)}))_{M+1}(\theta^{(k)}_{M+1}-\theta^*_{M+1})\\
           =&\frac{1}{2m}\sum_{l=1}^{M+1}(\theta^{(k)}_l-\theta^*_l)(\theta^{(k)}_{M+1}-\theta^*_{M+1})\sum_{j=1}^m\frac{\lambda^{(k)}_{lj}}{\left(\sum_{l=1}^{M+1}\theta_l\lambda^{(k)}_{lj}\right)^2}\\
           \geq &\frac{1}{2m}(\theta^{(k)}_{M+1}-\theta^*_{M+1})^2\sum_{j=1}^m\frac{1}{\left(\sum_{l=1}^{M+1}\theta_l\lambda^{(k)}_{lj}\right)^2}-\frac{1}{2m}(\theta_{\max}-\theta_{\min})^2\sum_{l=1}^{M}\sum_{j=1}^m\frac{\lambda^{(k)}_{lj}}{\left(\sum_{l=1}^{M+1}\theta_l\lambda^{(k)}_{lj}\right)^2}\\
           \geq&(\theta^{(k)}_{M+1}-\theta^*_{M+1})^2\frac{1-M\max_l a_l m^{\frac{(2+\alpha)(4b_1+3)}{4b_1(2b_1-1)}-1}}{8\theta_{\max}^2}\\
           &-(\theta_{\max}-\theta_{\min})^2M\left(\frac{1}{2\theta_{\min}^2}+\frac{\max_l a_l(4b_1+3)}{2\theta_{\min}^2(4b_1^2-6b_1-3)}\right)m^{\frac{(2+\alpha)(4b_1+3)}{4b_1(2b_1-1)}-1},
        \end{split}
    \end{equation}
    with probability at least $1-Mm^{-\alpha}$ for any $0<\alpha<\frac{8b_1^2-12b_1-6}{4b_1+3}$, where the last line is due to the following Lemma \ref{lem:eigen_ratio_bnds_poly}.
    
Therefore, 
   \begin{equation}
       (g^*(\vect{\theta}^{(k)}))_{M+1}(\theta^{(k)}_{M+1}-\theta^*_{M+1})\geq\frac{\gamma}{2} (\theta^{(k)}_{M+1}-\theta_{M+1}^*)^2-\varepsilon,
   \end{equation}
   where $\gamma=\frac{1}{8\theta_{\max}^2}$, $\varepsilon=Cm^{\frac{(2+\alpha)(4b_1+3)}{4b_1(2b_1-1)}-1}$, if $m>C$. Here $C>0$ depends only on $\theta_{\min}$,$\theta_{\max}$,$M$,$b_1,\dots,b_{M}$.
   \end{proof}
\begin{proof}[proof of Lemma \ref{lem:eigen_ratio_bnds_poly}]
Similarly from the proof of Lemma \ref{lem:eigen_ratio_bnds}, we apply Lemma \ref{lem:eigenval_err_bnds} on $\vect{K}_{f,n}^{(l)}$ to derive upper bounds for $\lambda_{lj}, 1\leq l\leq M-1$ w.h.p.
In particular, plug $r=j^{\frac{4b_l}{4b_l+3}}$ for each $1\leq j\leq n$ into \eqref{eq:eigenval_err_bnd2} and let $\delta=n^{-(\alpha+1)}$ for $0<\alpha<\frac{8b_1^2-12b_1-6}{4b_1+3}$,
then with probability at least $1-n^{-(\alpha+1)}$, 
    \begin{equation*}
    	\begin{split}
    	    C(r,n)<& \sqrt{\frac{2}{C_l}}r^{b_l+2}n^{\frac{\alpha}{2}}\leq\sqrt{\frac{2}{C_l}}j^{\frac{4b_l(b_l+2)}{4b_l+3}}n^{\frac{\alpha}{2}},\\
    	    E(r_j,n)<&\frac{C_l}{2b_l-1}r^{-(2b_l-1)}+\sqrt{\frac{2C_l}{2b_l-1}}r^{-(b_l-\frac{1}{2})}n^{\frac{\alpha}{2}}\\
    	    \leq&\left(\frac{C_l}{2b_l-1}+\sqrt{\frac{2C_l}{2b_l-1}}\right)j^{-\frac{2b_l(2b_l-1)}{4b_l+3}}n^{\frac{\alpha}{2}},
    	\end{split}
    \end{equation*}
    Thus we have
    \begin{equation}\label{eq:eigen_upp_bnd_poly}
    	\begin{split}
    	    \lambda_{lj}\leq &C_lj^{-2b_l}\left(n+\sqrt{\frac{2}{C_l}}j^{\frac{4b_l(b_l+2)}{4b_l+3}}n^{1+\frac{\alpha}{2}}\right)+\left(\frac{C_l}{2b_l-1}+\sqrt{\frac{2C_l}{2b_l-1}}\right)j^{-\frac{2b_l(2b_l-1)}{4b_l+3}}n^{1+\frac{\alpha}{2}}\\
    	    \leq&\left(2\sqrt{2C_l}+\sqrt{\frac{2C_l}{2b_l-1}}+\frac{C_l}{2b_l-1}\right)j^{-\frac{2b_l(2b_l-1)}{4b_l+3}}n^{1+\frac{\alpha}{2}}\\
    	    :=&a_lj^{-\frac{2b_l(2b_l-1)}{4b_l+3}}n^{1+\frac{\alpha}{2}},
    	\end{split}
    \end{equation}
   for $1\leq j\leq n, 1\leq l\leq M$, with probability at least $1-Mn^{-\alpha}$. 	    
    Now we are ready to prove the bounds for $\sum_{j=1}^n \frac{\lambda_{lj}\lambda_{l'j}}{\left(\sum_{h=1}^{M+1} \theta_h\lambda_{hj}\right)^2}$ for $1\leq l,l'\leq M+1$.
    \begin{enumerate}
    	\item $1\leq l\leq l'\leq M$\\
    	For any $0<L\leq n$, one can show that
    	\begin{equation}
    	    \begin{split}
    	        \sum_{j=1}^n \frac{\lambda_{lj}\lambda_{l'j}}{\left(\sum_{h=1}^{M+1}\theta_h\lambda_{hj}\right)^2}\leq&\frac{L}{\theta_{\min}^2}+\frac{a_la_{l'}}{\theta_{\min}^2}\sum_{j=L}^{\infty}j^{-\frac{4b_l(2b_l-1)}{4b_l+3}}n^{2+\alpha}\\
    	        \leq&\frac{L}{\theta_{\min}^2}+\frac{a_la_{l'}L^{1-\frac{4b_l(2b_l-1)}{4b_l+3}}}{\theta_{\min}^2(\frac{4b_l(2b_l-1)}{4b_l+3}-1)}n^{2+\alpha}
    	     \end{split}
    	\end{equation}
    	Let $L=n^{\frac{(2+\alpha)(4b_l+3)}{4b_l(2b_l-1)}}$, then we have
    	\begin{equation}
    	    \begin{split}
    	        \sum_{j=1}^n \frac{\lambda_{lj}\lambda_{l'j}}{\left(\sum_{h=1}^{M+1}\theta_h\lambda_{hj}\right)^2}\leq n^{\frac{(2+\alpha)(4b_l+3)}{4b_l(2b_l-1)}}\left(\frac{1}{\theta_{\min}^2}+\frac{a_la_{l'}(4b_l+3)}{\theta_{\min}^2(8b_l^2-8b_l-3)}\right)
    	     \end{split}
    	\end{equation}
      \item $l=l'=M+1$\\
    	   The upper bound for $\sum_{j=1}^n \frac{\lambda_{M+1,j}^2}{\left(\sum_{h=1}^{M+1}\theta_h\lambda_{hj}\right)^2}$ in Lemma \ref{lem:eigen_ratio_bnds_poly} is straightforward, since $\sum_{h=1}^{M+1}\theta_h\lambda_{hj}\geq \theta_{\min}\lambda_{M+1,j}$. While for the lower bound, note that $\lambda_{M+1,j}=1$, and
    	   \begin{equation}
    	       \sum_{j=1}^n \frac{\lambda_{M+1,j}^2}{\left(\sum_{h=1}^{M+1}\theta_h\lambda_{hj}\right)^2}\geq \frac{|\{j:\sum_{h=1}^{M+1}\theta_h\lambda_{hj}\leq 2\theta_{\max}\}|}{4\theta_{\max}^2}.
    	   \end{equation}
    	   Meanwhile, by \eqref{eq:eigen_upp_bnd_poly} one can show that
    	   \begin{equation}
    	       \begin{split}
    	           \sum_{h=1}^{M+1}\theta_h\lambda_{hj}\leq 2\theta_{\max}\Leftarrow & M\max_l a_l\theta_{\max}j^{-\frac{2b_1(2b_1-1)}{4b_1+3}}n^{1+\frac{\alpha}{2}}\leq\theta_{\max}\\
    	           \Leftarrow&j\geq M\max_l a_l n^{\frac{(2+\alpha)(4b_1+3)}{4b_1(2b_1-1)}}.
    	       \end{split}
    	   \end{equation}
    	   Therefore, 
    	   \begin{equation}
    	       \sum_{j=1}^n \frac{\lambda_{M+1,j}^2}{\left(\sum_{h=1}^{M+1}\theta_h\lambda_{hj}\right)^2}\geq \frac{n-M\max_l a_l n^{\frac{(2+\alpha)(4b_1+3)}{4b_1(2b_1-1)}}}{4\theta_{\max}^2}.
    	   \end{equation}
      \item $1\leq l\leq M$, $l'=M+1$\\
      First note that by similar arguments from the first case where $1\leq l\leq l'\leq M$, one can show that for any $L>0$,
      \begin{equation}
      \begin{split}
          \sum_{j=1}^n \frac{\lambda_{lj}\lambda_{l'j}}{\left(\sum_{h=1}^{M+1} \theta_h\lambda_{hj}\right)^2}\leq &\frac{L}{\theta_{\min}^2}+\frac{a_l}{\theta_{\min}^2}\sum_{j=L}^{\infty}j^{-\frac{2b_l(2b_l-1)}{4b_l+3}}n^{1+\frac{\alpha}{2}}\\
    	        \leq&\frac{L}{\theta_{\min}^2}+\frac{a_lL^{1-\frac{2b_l(2b_l-1)}{4b_l+3}}}{\theta_{\min}^2(\frac{2b_l(2b_l-1)}{4b_l+3}-1)}n^{1+\frac{\alpha}{2}},
    	     \end{split}
    	\end{equation}
      Let $L=n^{\frac{(2+\alpha)(4b_l+3)}{4b_l(2b_l-1)}}$, then we have
    	\begin{equation}
    	    \begin{split}
    	        \sum_{j=1}^n \frac{\lambda_{lj}\lambda_{l'j}}{\left(\sum_{h=1}^{M+1}\theta_h\lambda_{hj}\right)^2}\leq n^{\frac{(2+\alpha)(4b_l+3)}{4b_l(2b_l-1)}}\left(\frac{1}{\theta_{\min}^2}+\frac{a_l(4b_l+3)}{\theta_{\min}^2(4b_l^2-6b_l-3)}\right)
    	     \end{split}
    	\end{equation}
\end{enumerate}
Therefore, for any $0<\alpha<\frac{8b_1^2-12b_1-6}{4b_1+3}$, with probability at least $1-Mn^{-\alpha}$, \eqref{eq:eigen_ratio_bnds_poly} holds.
\end{proof}
\begin{proof}[proof for Lemma~\ref{lem:curv_lengthscale}]
	Consider the case where $x_i\sim \mathcal{N}(0,\sigma^2)$, $k(x,x')=\exp\{-(x-x')^2/2l^2\}$, then $\lambda_j$ takes the following analytical form \cite[see]{zhu1997gaussian}: 
	\begin{equation}
	\lambda_j=(1-\beta)\beta^{j-1},
	\end{equation}
	where $\beta=\frac{2\sigma^2}{2\sigma^2+l^2+l\sqrt{l^2+4\sigma^2}}$ is a decreasing function of positive $l$. We want to see if $\widetilde{\gamma}_{\epsilon}$ is a decreasing function of $\beta$. First note that
	\begin{equation}\label{eq:gamma_epsilon_deriv}
	\begin{split}
	\frac{\partial \widetilde{\gamma}^{(k)}_{\epsilon}(\beta,m)}{\partial \beta}=&-2\sum_{j=1}^m \frac{(j-1)\beta^{j-2}-j\beta^{j-1}}{\left(\theta^{(k)}_1m(1-\beta)\beta^{j-1}+\theta^{(k)}_2\right)^{3}}\\
	=&2\sum_{j=1}^{m}\frac{j\beta^{j-1}}{\left(\theta^{(k)}_1m(1-\beta)\beta^{j-1}+\theta^{(k)}_2\right)^{3}}-2\sum_{j=1}^{m-1}\frac{j\beta^{j-1}}{\left(\theta^{(k)}_1m(1-\beta)\beta^{j}+\theta^{(k)}_2\right)^{3}}\\
	=&2\sum_{j=1}^{m}j\beta^{j-1}\left[\left(\theta^{(k)}_1m(1-\beta)\beta^{j-1}+\theta^{(k)}_2\right)^{-3}-\left(\theta^{(k)}_1m(1-\beta)\beta^{j}+\theta^{(k)}_2\right)^{-3}\right]\\
	&+2\frac{m\beta^{m-1}}{\left(\theta^{(k)}_1m(1-\beta)\beta^{m}+\theta^{(k)}_2\right)^{3}}.
	\end{split}
	\end{equation}
	Let $a=\frac{\theta_1^{(k)}(1-\beta)}{\theta^{(k)}_2}$, and we provide an upper bound for $\frac{\theta^{(k)3}_2m}{2\log m}\frac{\partial \widetilde{\gamma}^{(k)}_{\epsilon}(\beta,m)}{\partial \beta}$ in the following:
	\begin{equation}\label{eq:curv_epsilon_deriv_upp}
	\begin{split}
	\frac{\theta^{(k)3}_2m}{2\log m}\frac{\partial \widetilde{\gamma}^{(k)}_{\epsilon}(\beta,m)}{\partial \beta}&=\frac{m}{\log m}\sum_{j=1}^m j\beta^{j-1}\left[\left(am\beta^{j-1}+1\right)^{-3}-\left(am\beta^j+1\right)^{-3}\right]\\
	&+\frac{m^2\beta^{m-1}}{\left(am\beta^{m}+1\right)^{3}\log m}\\
	\leq& -\frac{3a(1-\beta)m^2}{\log m}\sum_{j=1}^m \frac{j\beta^{2j-2}}{\left(
		am\beta^{j-1}+1\right)^4}+\frac{m^2\beta^{m-1}}{\left(am\beta^{m}+1\right)^{3}\log m}\\
	\leq&-\frac{3a(1-\beta)}{\log (\beta^{-1})\beta^2\left(
		a/\beta+1\right)^4}+\frac{m^2\beta^{m-1}}{\log m},
	\end{split}
	\end{equation}
	where we let $j=\frac{\log m}{\log (\beta^{-1})}$ on the last line. Since 
	\begin{equation}
	\lim_{m\rightarrow \infty}\frac{m^2\beta^{m-1}}{\log m}=0,
	\end{equation}
	\eqref{eq:curv_epsilon_deriv_upp} implies that for any $l_0>0$, there exists a $m_0>0$ depending on $\theta^{(k)}_1,\theta^{(k)}_2,\sigma,l_0$ such that as long as $m>m_0$, $\widetilde{\gamma}_{\epsilon}$ is a increasing function of $l\geq l_0$. That is to say, for large enough minibatch, larger length scale leads to faster convergence for $\sigma_{\epsilon}^2$, which suggests the potential benefit of nearby sampling.
	\end{proof}
\section{Explanation on the connection between Assumption~\ref{assump:param_bnd} and Assumption~\ref{assump:sg_bnd}}\label{append:sg_bnd}
We explain the how Assumption~\ref{assump:sg_bnd} can be proved with Assumption~\ref{assump:param_bnd} under the exponential eigendecay and polynomial eigendecay cases separately.
\begin{itemize}
    \item \textbf{Exponential eigendecay (Assumption~\ref{assump:eigen_decay_exp})}: Consider Lemma~\ref{lem:statistical_err_full_grad} with $M=1$, $s_1(m)=\tau\log m$ and $s_2(m)=m$, then we have 
$$
\|g(\vect{\theta}^{(k)};\vect{X}_{\xi_{k+1}},\vect{y}_{\xi_{k+1}})-g^*(\vect{\theta}^{(k)})\|_2\leq C(\log m)^{-\frac{1}{2}+\varepsilon},
$$
with probability at least $1-C\exp\{-c(\log m)^{2\varepsilon}\}$. Hence it suffices to show that $\|g^*(\vect{\theta}^{(k)})\|_2$ is bounded with high probability. Meanwhile, some calculation suggests 
\begin{equation}\label{eq:sg_formula}
\begin{split}
    \mathbb{E}[g_1(\vect{\theta}^{(k)};\vect{X}_{\xi_{k+1}},\vect{y}_{\xi_{k+1}})|\vect{X}_{\xi_{k+1}}]=\frac{1}{2s_1(m)}\sum_{l=1}^2(\theta_l^{(k)}-\theta_l^*)\sum_{j=1}^m\frac{(\lambda_{j}^{(k)})^{1+\ind{l=1}}}{(\theta_1^{(k)}\lambda_j^{(k)}+\theta_2^{(k)})^2},\\
    \mathbb{E}[g_2(\vect{\theta}^{(k)};\vect{X}_{\xi_{k+1}},\vect{y}_{\xi_{k+1}})|\vect{X}_{\xi_{k+1}}]=\frac{1}{2s_2(m)}\sum_{l=1}^2(\theta_l^{(k)}-\theta_l^*)\sum_{j=1}^m\frac{(\lambda_{j}^{(k)})^{\ind{l=1}}}{(\theta_1^{(k)}\lambda_j^{(k)}+\theta_2^{(k)})^2}.
    \end{split}
\end{equation}
Under the exponential eigendecay assumption, Lemma~\ref{lem:eigen_ratio_bnds} suggests that 
\begin{equation}\label{eq:bnd_sg_eigenratio_bnds_exp}
\begin{split}
    \sum_{j=1}^m\frac{(\lambda_{j}^{(k)})^{1+\ind{l=1}}}{(\theta_1^{(k)}\lambda_j^{(k)}+\theta_2^{(k)})^2}\leq& C\log m,\\
    \sum_{j=1}^m\frac{1}{(\theta_1^{(k)}\lambda_j^{(k)}+\theta_2^{(k)})^2}\leq& Cm.
\end{split}
\end{equation}
By the boundedness of $\vect{\theta}^{(k)}$ (Assumption~\ref{assump:param_bnd}), \eqref{eq:bnd_sg_eigenratio_bnds_exp} and \eqref{eq:sg_formula}, we have $$\|g(\vect{\theta^{(k)}};\vect{X}_{\xi_{k+1}},\vect{y}_{\xi_{k+1}})\|_2\leq C$$ when $s_1(m)\geq c\log m$ and $s_2(m)=m$ and $m$ is large enough.
\item \textbf{Polynomial eigendecay (Assumption~\ref{assump:eigen_decay_poly})}: Consider Lemma~\ref{lem:statistical_err_full_grad} with $M=1$ and Assumption~\ref{assump:eigen_decay_poly}, then when $s_1(m)=s_2(m)=m$,
$$
\|g(\vect{\theta}^{(k)};\vect{X}_{\xi_{k+1}},\vect{y}_{\xi_{k+1}})-g^*(\vect{\theta}^{(k)})\|_2\leq Cm^{-\frac{1}{2}+\varepsilon},
$$
with probability at least $1-C\exp\{-cm^{2\varepsilon}\}$. Meanwhile, note that Lemma~\ref{lem:eigen_ratio_bnds} suggests that for $i=0,1$,
\begin{equation}\label{eq:bnd_sg_eigenratio_bnds_poly}
    \sum_{j=1}^m\frac{(\lambda_j^{(k)})^i}{(\theta_1^{(k)}\lambda_j^{(k)}+\theta_2^{(k)})^2}\leq Cm,
\end{equation}
and since \eqref{eq:sg_formula} still holds,
we have $\|g^*(\vect{\theta}^{(k)})\|_2\leq C$ when $s_1(m)=s_2(m)=m$.
Therefore, by the boundedness of $\vect{\theta}^{(k)}$ (Assumption~\ref{assump:param_bnd}),  $\|g(\vect{\theta}^{(k)};\vect{X}_{\xi_{k+1}},\vect{y}_{\xi_{k+1}})\|_2\leq C$ when $s_1(m)=s_2(m)=m$ and $m$ is large enough.
\end{itemize}

\bibliography{reference}
\end{document}